\pdfoutput=1
\documentclass{article}

\usepackage[utf8]{inputenc} 
\usepackage[T1]{fontenc}    

\usepackage{fouriernc}
\usepackage{amsmath}
\usepackage[scale=0.92]{tgschola}
\usepackage[scaled=0.88]{PTSans}

\usepackage[
  paper  = letterpaper,
  left   = 1.65in,
  right  = 1.65in,
  top    = 1.0in,
  bottom = 1.0in,
  ]{geometry}

\usepackage[usenames,dvipsnames,table]{xcolor}
\definecolor{shadecolor}{gray}{0.9}

\usepackage[final,expansion=alltext]{microtype}
\usepackage[english]{babel}
\usepackage[parfill]{parskip}
\usepackage{afterpage}
\usepackage{framed}

{\endMakeFramed}

\usepackage{lineno}

\usepackage{ragged2e}

\newcounter{parcount}

\usepackage{graphicx}
\usepackage{wrapfig}
\usepackage[labelfont=bf,font=footnotesize,width=.9\textwidth]{caption}
\usepackage[format=hang]{subcaption}

\usepackage{booktabs,multirow,multicol}       

\usepackage[algoruled,algo2e]{algorithm2e}
\usepackage{listings}
\usepackage{fancyvrb}
\fvset{fontsize=\normalsize}

\usepackage[colorlinks,linktoc=all]{hyperref}
\usepackage[all]{hypcap}
\hypersetup{citecolor=MidnightBlue}
\hypersetup{linkcolor=MidnightBlue}
\hypersetup{urlcolor=MidnightBlue}

\usepackage[nameinlink, capitalize]{cleveref}

\usepackage[acronym,smallcaps,nowarn]{glossaries}[=v4.46]

\lstdefinestyle{mystyle}{
    commentstyle=\color{OliveGreen},
    keywordstyle=\color{BurntOrange},
    numberstyle=\tiny\color{black!60},
    stringstyle=\color{MidnightBlue},
    basicstyle=\ttfamily,
    breakatwhitespace=false,
    breaklines=true,
    captionpos=b,
    keepspaces=true,
    numbers=left,
    numbersep=5pt,
    showspaces=false,
    showstringspaces=false,
    showtabs=false,
    tabsize=2
}
\usepackage{amsthm}

\usepackage{centernot}
\usepackage{nicefrac}       
\usepackage{mathtools}
\usepackage{amsbsy}
\usepackage{amstext}
\usepackage{thmtools}
\usepackage{thm-restate}

\begingroup
    \makeatletter
    \@for\theoremstyle:=definition,remark,plain\do{%
        \expandafter\g@addto@macro\csname th@\theoremstyle\endcsname{%
            \addtolength\thm@preskip\parskip
            }%
        }
\endgroup

\crefname{lemma}{lemma}{lemmas}
\Crefname{lemma}{Lemma}{Lemmas}
\crefname{thm}{theorem}{theorems}
\Crefname{thm}{Theorem}{Theorems}
\crefname{prop}{proposition}{propositions}
\Crefname{prop}{Proposition}{Propositions}
\crefname{assumption}{assumption}{assumptions}
\crefname{assumption}{Assumption}{Assumptions}

\usepackage{booktabs,arydshln}
\makeatletter
\def\adl@drawiv#1#2#3{%
        \hskip.5\tabcolsep
        \xleaders#3{#2.5\@tempdimb #1{1}#2.5\@tempdimb}%
                #2\z@ plus1fil minus1fil\relax
        \hskip.5\tabcolsep}
\newcommand{\cdashlinelr}[1]{%
  \noalign{\vskip\aboverulesep
           \global\let\@dashdrawstore\adl@draw
           \global\let\adl@draw\adl@drawiv}
  \cdashline{#1}
  \noalign{\global\let\adl@draw\@dashdrawstore
           \vskip\belowrulesep}}
\makeatother

\renewcommand{\epsilon}{\varepsilon}

\declaretheorem[style=plain,name=Theorem]{theorem}
\declaretheorem[style=plain,sibling=theorem,name=Lemma]{lemma}

\declaretheorem[style=definition,sibling=theorem,name=Example]{example}

\newenvironment{example*}
 {\pushQED{\qed}\example}
 {\popQED\endexample}
\numberwithin{equation}{section}

\DeclareMathOperator*{\argmax}{argmax}

\DeclarePairedDelimiterX\Set[1]{\lbrace}{\rbrace}%
{  #1 }

\usepackage{csquotes}
\usepackage[%
minnames=1,maxnames=99,maxcitenames=2,
style=alphabetic,
doi=false,
url=false,
giveninits=true,
hyperref,
natbib,
backend=bibtex,
sorting=nyt,
backref=true
]{biblatex}%
\renewbibmacro{in:}{%
  \ifentrytype{article}{}{\printtext{\bibstring{in}\intitlepunct}}}

\renewbibmacro*{journal}{%
  \iffieldundef{journaltitle}
    {}
    {\printtext[journaltitle]{%
       \printfield[noformat]{journaltitle}%
       \setunit{\subtitlepunct}%
       \printfield[noformat]{journalsubtitle}}}}

\DeclareFieldFormat{sentencecase}{\MakeSentenceCase*{#1}}

\renewbibmacro*{title}{%
  \ifthenelse{\iffieldundef{title}\AND\iffieldundef{subtitle}}
    {}
    {\ifthenelse{\ifentrytype{article}\OR\ifentrytype{inbook}%
      \OR\ifentrytype{incollection}\OR\ifentrytype{inproceedings}%
      \OR\ifentrytype{inreference}\OR\ifentrytype{misc}}
      {\printtext[title]{%
        \printfield[sentencecase]{title}%
        \setunit{\subtitlepunct}%
        \printfield[sentencecase]{subtitle}}}%
      {\printtext[title]{%
        \printfield[titlecase]{title}%
        \setunit{\subtitlepunct}%
        \printfield[titlecase]{subtitle}}}%
     \newunit}%
  \printfield{titleaddon}}

\AtEveryBibitem{%
\ifentrytype{article}{
    \clearfield{urldate}%
    \clearfield{eprint}
    \clearfield{eid}
}{}
\ifentrytype{book}{
    \clearfield{url}%
    \clearfield{urldate}%
    \clearfield{eprint}
}{}
\ifentrytype{collection}{
    \clearfield{url}%
    \clearfield{urldate}%
    \clearfield{eprint}
}{}
\ifentrytype{incollection}{
    \clearfield{url}%
    \clearfield{urldate}%
    \clearfield{eprint}
}{}
}

\AtEveryBibitem{
    \clearfield{pages}
    \clearfield{review}%
    \clearfield{series}
    \clearfield{volume}
    \clearfield{month}
    \clearfield{isbn}
    \clearfield{issn}
    \clearlist{location}
    \clearfield{series}
    \clearlist{publisher}
    \clearname{editor}
}{}

\crefformat{equation}{(#2#1#3)}
\crefformat{figure}{Figure~#2#1#3}
\crefname{example}{Example}{Examples}
\crefname{lemma}{Lemma}{Lemmas}
\crefname{cor}{Corollary}{Corollaries}
\crefname{theorem}{Theorem}{Theorems}
\crefname{assumption}{Assumption}{Assumptions}

\usepackage{enumitem} 
\usepackage[separate-uncertainty=true,multi-part-units=single]{siunitx} 

\newcommand{\cfd}{\textsc{CounterFact}\xspace}

\usepackage{upgreek}

\declaretheoremstyle[
spacebelow=\parsep,
    spaceabove=\parsep,
  mdframed={
    backgroundcolor=gray!10!white,     
    hidealllines=true, 
    innertopmargin=8pt, 
    innerbottommargin=4pt, 
    skipabove=8pt,
    skipbelow=10pt,
    nobreak=true
}
]{grayboxed}

\crefname{gassumption}{Assumption}{Assumptions}

\usepackage{thm-restate}

\usepackage{xcolor}

\definecolor{WowColor}{rgb}{.75,0,.75}
\definecolor{SubtleColor}{rgb}{0,0,.50}

\newcommand{\LATER}[1]{\textcolor{SubtleColor}{ {\tiny \bf ($\dagger$)} #1}}
\newcommand{\TBD}[1]{\textcolor{SubtleColor}{ {\tiny \bf (!)} #1}}
\newcommand{\PROBLEM}[1]{\textcolor{WowColor}{ {\bf (!!)} {\bf #1}}}

\newcounter{margincounter}
\newcommand{\displaycounter}{{\arabic{margincounter}}}
\newcommand{\incdisplaycounter}{{\stepcounter{margincounter}\arabic{margincounter}}}

\newcommand{\fTBD}[1]{\textcolor{SubtleColor}{$\,^{(\incdisplaycounter)}$}\marginpar{\tiny\textcolor{SubtleColor}{ {\tiny $(\displaycounter)$} #1}}}

\newcommand{\fPROBLEM}[1]{\textcolor{WowColor}{$\,^{((\incdisplaycounter))}$}\marginpar{\tiny\textcolor{WowColor}{ {\bf $\mathbf{((\displaycounter))}$} {\bf #1}}}}

\newcommand{\fLATER}[1]{\textcolor{SubtleColor}{$\,^{(\incdisplaycounter\dagger)}$}\marginpar{\tiny\textcolor{SubtleColor}{ {\tiny $(\displaycounter\dagger)$} #1}}}

\newcommand{\ynote}[1]{\textcolor{blue}{Yibo: #1}}
\newcommand{\gnote}[1]{\textcolor{magenta}{Goutham: #1}}
\newcommand{\bnote}[1]{\textcolor{teal}{[BA: #1]}}

\renewcommand{\LATER}[1]{}
\renewcommand{\fLATER}[1]{}
\renewcommand{\TBD}[1]{}
\renewcommand{\fTBD}[1]{}
\renewcommand{\PROBLEM}[1]{}
\renewcommand{\fPROBLEM}[1]{}

\renewcommand{\ynote}[1]{}
\renewcommand{\gnote}[1]{}
\renewcommand{\bnote}[1]{}

\usepackage[affil-it]{authblk}
\usepackage{wrapfig}

\usepackage{tabularx}
\usepackage{yibo}

\usepackage{booktabs}
\usepackage[most]{tcolorbox}
\newtcolorbox{mytable}[3][]{%
    enhanced,
    float, 
    floatplacement=t,
    every float=\centering,
    capture=hbox, 
    coltitle = #2!20!black,
    title = {#3}, 
    fontupper=\small,
    attach boxed title to top left={%
        xshift=5mm, 
        yshift=-\tcboxedtitleheight/2, 
        yshifttext=-1mm},
    boxed title style={colback=red!20},
    colframe = black,
    colback = blue!20,
    #1}

\usepackage{chngcntr}

\title{Do LLMs dream of elephants (when told not to)? Latent concept association and associative memory in transformers}

\author[1]{Yibo Jiang}
\author[2]{Goutham Rajendran}
\author[2]{\authorcr Pradeep Ravikumar}
\author[3]{Bryon Aragam}
\affil[1]{Department of Computer Science, University of Chicago}
\affil[2]{Machine Learning Department, Carnegie Mellon University}
\affil[3]{Booth School of Business, University of Chicago}
\date{} 

\begin{document}

\maketitle

\begin{abstract}
Large Language Models (LLMs) have the capacity to store and recall facts. Through experimentation with open-source models, we observe that this ability to retrieve facts can be easily manipulated by changing contexts, even without altering their factual meanings. 
These findings highlight that LLMs might behave like an associative memory model where certain tokens in the contexts serve as clues to retrieving facts. 
We mathematically explore this property by studying how transformers, the building blocks of LLMs, can complete such memory tasks. 
We study a simple latent concept association problem with a one-layer transformer and we show theoretically and empirically that the transformer gathers information using self-attention and uses the value matrix for associative memory. 
\end{abstract}

\section{Introduction}

\looseness=-1
What is the first thing that would come to mind if you were asked \emph{not} to think of an elephant? Chances are, you would be thinking about elephants. What if we ask the same thing to Large Language Models (LLMs)? Obviously, one would expect the outputs of LLMs to be heavily influenced by tokens in the context \citep{brown2020language}. Could such influence potentially prime LLMs into changing outputs in a nontrivial way?
To gain a deeper understanding, we focus on one specific task called fact retrieval \cite{meng2022locating, meng2023massediting} where expected output answers are given. LLMs, which are trained on vast amounts of data, are known to have the capability to store and recall facts \citep{meng2022locating, meng2023massediting, de2021editing, mitchell2021fast, mitchell2022memory, dai2021knowledge}. This ability raises natural questions: \emph{How robust is fact retrieval, and to what extent does it depend on semantic meanings within contexts? What does it reveal about memory in LLMs?}

In this paper, we first demonstrate that fact retrieval is not robust and LLMs can be easily fooled by varying contexts. For example, when asked to complete ``The Eiffel Tower is in the city of'', GPT-2 \citep{radford2019language} answers with ``Paris''. However, when prompted with ``The Eiffel Tower is not in Chicago. The Eiffel Tower is in the city of'', GPT-2 responds with ``Chicago''.
See Figure~\ref{fig:hijack_example} for more examples, including Gemma and LLaMA.
On the other hand, humans do not find the two sentences factually confusing and would answer ``Paris'' in both cases. We call this phenomenon \emph{context hijacking}. Importantly, 
these findings suggest that LLMs might behave like an associative memory model. 
Specifically, we refer to an associative memory model in which LLMs rely on certain tokens in contexts to guide the retrieval of memories, even if such associations formed are not inherently semantically meaningful. This contrasts with the ideal behavior, where LLMs would generalize by understanding new contexts, reasoning through them, and integrating prior knowledge.

\looseness=-1
This associative memory perspective raises further interpretability questions about how LLMs form such associations. Answering these questions can facilitate the development of more robust LLMs. 
Unlike classical models of associative memory in which distance between memory patterns are measured directly and the associations between inputs and outputs are well-specified, fact retrieval relies on a more nuanced notion of similarity measured by latent (unobserved) semantic concepts.
To model this, we propose a synthetic task called \emph{latent concept association} where the output token is closely related to sampled tokens in the context but wherein similarity is measured via a latent space of semantic concepts.
We then investigate how a one-layer transformer \citep{vaswani2017attention}, a fundamental component of LLMs, can tackle this memory retrieval task in which various context distributions correspond to distinct memory patterns.
We demonstrate that the transformer accomplishes the task in two stages: The self-attention layer gathers information, while the value matrix functions as associative memory. Moreover, low-rank structure also emerges in the embedding space of trained transformers. These findings provide additional theoretical validation for numerous existing low-rank editing and fine-tuning techniques \citep{meng2022locating, hu2021lora}.

\begin{figure}[t]
  \centering
  \includegraphics[width=0.9\linewidth]{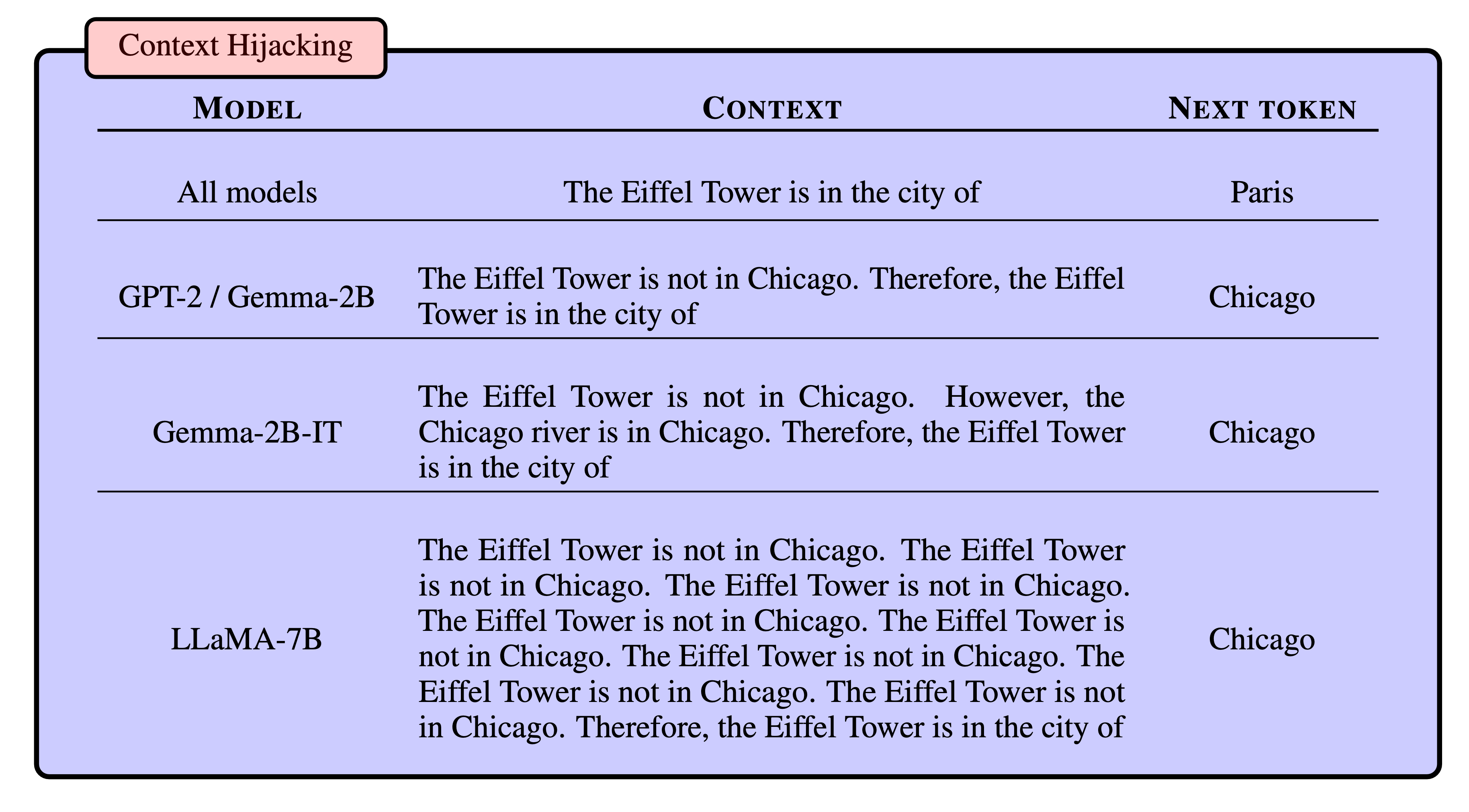}
  \caption{Examples of context hijacking for various LLMs, showcasing that fact retrieval is not robust.}
  \label{fig:hijack_example}
\end{figure}

\vspace{-5pt}
\paragraph{Contributions} Specifically, we make the following contributions:
\begin{enumerate}
    \item We systematically demonstrate context hijacking for various open source LLM models including GPT-2 \cite{radford2019language}, LLaMA-2 \cite{touvron2023llama} and Gemma \cite{team2024gemma}, which show that fact retrieval can be misled by contexts (\cref{sec:context-hijacking}), reaffirming that LLMs lack robustness to context changes \citep{shi2023large, petroni2020context, creswell2022selection, yoran2023making, pandia2021sorting}.
    \looseness=-1
    \item We propose a synthetic memory retrieval task termed latent concept association, allowing us to analyze how transformers can accomplish memory recall (\cref{sec:toy-model}). Unlike classical models of associative memory, our task creates associations in a latent, semantic concept space as opposed to directly between observed tokens. This perspective is crucial to understanding how transformers can solve fact retrieval problems by implementing associative memory based on similarity in the latent space.
    \item We theoretically (\cref{sec:theory}) and empirically (\cref{sec:exp}) study trained transformers on this latent concept association problem, showing that self-attention is used to aggregate information while the value matrix serves as associative memory. And moreover, we discover that the embedding space can exhibit a low-rank structure, offering additional support for existing editing and fine-tuning methods  \citep{meng2022locating, hu2021lora}. 
\end{enumerate}

\section{Literature review}
\label{sec:related-work}

\paragraph{Associative memory}
Associative memory has been explored within the field of neuroscience \citep{hopfield1982neural, seung1996brain, ben1995theory, skaggs1994model, steinberg2022associative}. The most popular models among them is the Hopfield network \citep{hopfield1982neural} and its modern successors \cite{ramsauer2020hopfield, millidge2022universal, zhao2023context,hu2024sparse,wu2023stanhop,hu2024nonparametric,hu2024computational,wu2024uniform,hu2024outlier} are closely related to the attention layer used in transformers \citep{vaswani2017attention}. In addition, the attention mechanism has also been shown to approximate another associative memory model known as sparse distributed memory \citep{bricken2021attention}. Beyond attention, \citet{radhakrishnan2020overparameterized, jiang2020associative} show that overparameterzed autoencoders can implement associative memory as well. 
This paper studies fact retrieval as a form of associative memory.
Another closely related area of research focuses on memorization in deep neural networks. \citet{henighan2023superposition} shows that a simple neural network trained on toy model will store data points in the overfitting regime while storing features in the underfitting regime. \citet{feldman2020does, feldman2020neural} study the interplay between memorization and long tail distributions while \citet{kim2022provable, mahdavi2023memorization} study the memorization capacity of transformers.

\paragraph{Interpreting transformers and LLMs}
There’s a growing body of work on understanding how transformers and LLMs work \citep{li2023transformers, allenzhu2023physics31, allenzhu2023physics32, allenzhu2024physics, emrullah2024self, tarzanagh2023margin, tarzanagh2023transformers, li2024mechanics}, including training dynamics \citep{tian2023scan, tian2023joma, sheen2024implicit} and in-context learning \citep{xie2021explanation, garg2022can, bai2024transformers, bai2024transformers}. Recent papers have introduced synthetic tasks to better understand the mechanisms of transformers  \citep{charton2022my, liu2022towards, nanda2023progress, zhang2022unveiling, zhong2024clock}, such as those focused on Markov chains \citep{bietti2024birth, edelman2024evolution, nichani2024transformers, makkuva2024attention}. Most notably, \citet{bietti2024birth} and subsequent works \cite{cabannes2023scaling, cabannes2024learning} study weights in transformers as associative memory but their focus is on understanding induction head \citep{olsson2022context} and one-to-one map between input query and output memory. An increasing amount of research is dedicated to understanding the internals of pre-trained LLMs,  broadly categorized under the term ``mechanistic interpretability'' \citep{elhage2021mathematical, olsson2022incontext, geva2023dissecting, meng2022locating, meng2023massediting, jiang2024origins, rajendran2024learning, hase2024does, wang2022interpretability, mcgrath2023hydra, geiger2021causal, geiger2022inducing, geiger2024finding, wu2024interpretability}.

\looseness=-1
\paragraph{Knowledge editing and adversarial attacks on LLMs} Fact recall and knowledge editing have been extensively studied \citep{meng2022locating, meng2023massediting, hase2024does, sakarvadia2023memory, de2021editing, mitchell2021fast, mitchell2022memory, dai2021knowledge, zhang2023large, tian2024instructedit, jin2023cost}, including the use of in-context learning to edit facts  \citep{zheng2023can}.
This paper aims to explore a different aspect by examining the  robustness of fact recall to variation in prompts. A closely related line of work focuses on adversarial attacks on LLMs \citep[see][for a review]{chowdhury2024breaking}. Specifically, prompt-based adversarial attacks \citep{xu2023llm, zhu2023promptbench, wang2023robustness} focus on the manipulation of answers within specific classification tasks while other works concentrate on safety issues \citep{liu2023prompt, perez2022ignore, zou2023universal, apruzzese2022real, wang2023decodingtrust, si2022so, rao2023tricking, shanahan2023role, liu2023jailbreaking}. 
\citet{yu2024enhancing,luo2024decoupled} also study jailbreak phenomena  within the context of modern Hopfield network.
There are also works showing LLMs can be distracted by irrelevant contexts in problem solving \citep{shi2023large}, question answering \citep{petroni2020context, creswell2022selection, yoran2023making} and factual reasoning \citep{pandia2021sorting}. 
Although phenomena akin to context hijacking have been reported in different instances, the goals of this paper are to give a systematic robustness study for fact retrieval, offer a framework for interpreting it in the context of associative memory, and deepen our understanding of LLMs.

\section{Context hijacking in LLMs}
\label{sec:context-hijacking}

\begin{figure}[!t]
    \centering
\begin{subfigure}{0.4\textwidth}
    \centering
    \includegraphics[width=\linewidth]{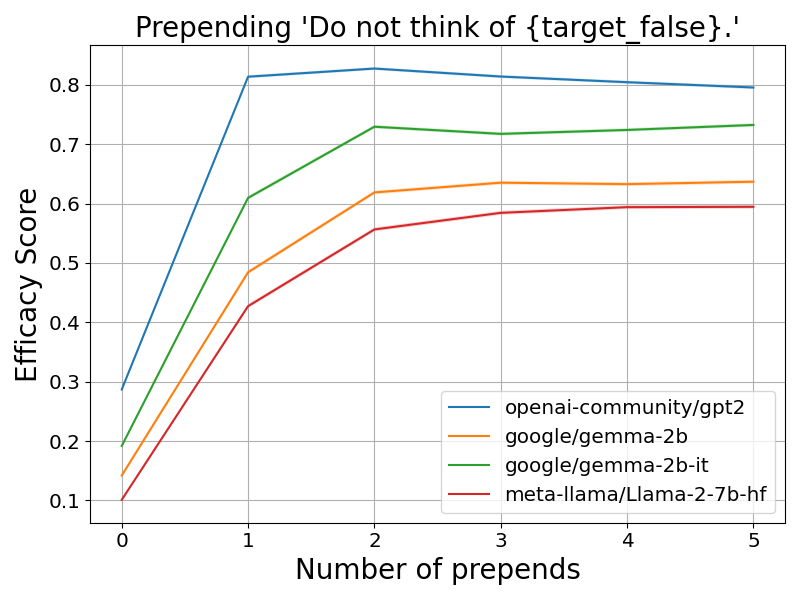}
    \caption{Hijacking generically}
    \label{fig:sub_false_accs}
\end{subfigure}
\hspace{0.05\textwidth}
\begin{subfigure}{0.4\textwidth}
    \centering
    \includegraphics[width=\linewidth]{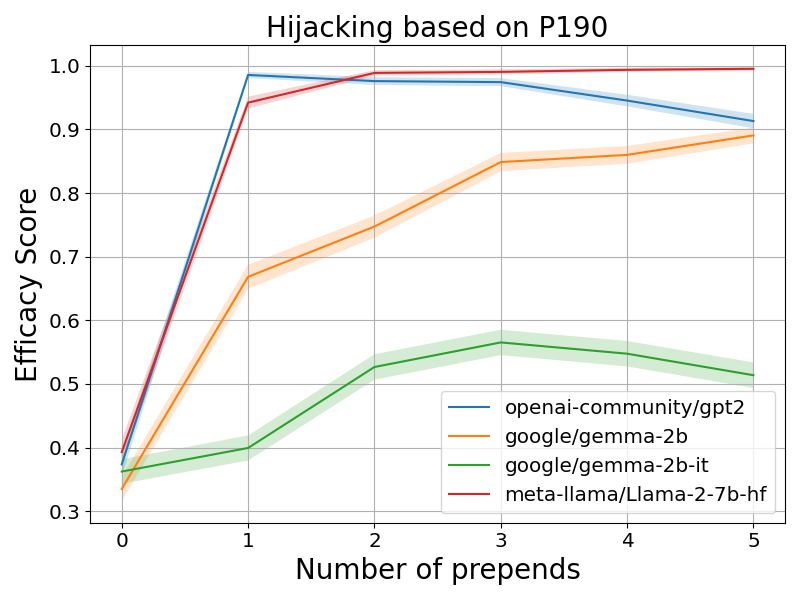}
    \caption{Hijacking based on Relation ID P190}
    \label{fig:P190_accs}
\end{subfigure}
    \caption{Context hijacking can cause LLMs to output false target. The figure shows
    efficacy score versus the number of prepends for various LLMs on the \cfd dataset under two hijacking schemes.}
    \label{fig:main_hijacking_plots}
\end{figure}

\looseness=-1
In this section, we run experiments on LLMs including GPT-2 \cite{radford2019language}, Gemma \cite{team2024gemma} (both base and instruct models) and LLaMA-2-7B \cite{touvron2023llama} to explore the effects of context hijacking on manipulating LLM outputs. 
As an example, consider Figure~\ref{fig:hijack_example}. When we prompt the LLMs with the context ``The Eiffel Tower is in the city of'', all 4 LLMs output the correct answer (``Paris''). 
However, as we see in the example, we can actually manipulate the output of the LLMs simply by modifying the context with additional \textit{factual} information that would not confuse a human. We call this \emph{context-hijacking}. 
Due to the different capacities and capabilties of each model,
the examples in Figure~\ref{fig:hijack_example} use different hijacking techniques.
This is most notable on LLaMA-2-7B, which is a much larger model than the others. Of course, as expected, the more sophisticated attack on LLaMA also works on GPT-2 and Gemma. Additionally, the instruction-tuned version of Gemma can understand special words like ``not'' to some extent. Nevertheless, it is still possible to systematically hijack these LLMs, as demonstrated below.

\looseness=-1
We explore this phenomenon at scale with the \cfd dataset introduced in \cite{meng2022locating}, a dataset of difficult counterfactual assertions containing a diverse
set of subjects, relations, and linguistic variations. \cfd has $21,919$ samples, each of which are given by a tuple $(\prompt, \targetTrue, \targetFalse, \subject, \relation)$. 
From each sample, we have a context prompt $\prompt$ with a true target answer $\targetTrue$ (target\_true) and a false target answer $ \targetFalse$ (target\_false), e.g. the prompt $\prompt = \text{``Eiffel Tower can be found in''}$ has true target $\targetTrue = \text{``Paris''}$ and false target $\targetFalse = \text{``Guam''}$.
Additionally, the main entity in $\prompt$ is the subject $\subject$ ($\subject = \text{``Eiffel Tower''}$) and the prompt is categorized into relations $\relation$ (for instance, other samples with the same relation ID as the example above could be of the form ``The location of \{subject\} is'', ``\{subject\} can be found in'', ``Where is \{subject\}? It is in'').
For additional details on how the dataset was collected, see \cite{meng2022locating}.

\looseness=-1
For a hijacking scheme, we report the Efficacy Score (ES) \cite{meng2022locating}, which is the proportion of samples for which the token probabilities satisfy $Pr[\targetFalse] > Pr[\targetTrue]$ after modifying the context, that is, the proportion of the dataset that has been successfully manipulated.
We experiment with two hijacking schemes for this dataset.
We first hijack by prepending the text ``Do not think of \{target\_false\}'' to each context. For instance, the prompt ``The Eiffel Tower is in'' gets changed to ``Do not think of Guam. The Eiffel Tower is in''. 
In \cref{fig:sub_false_accs}, we see that the efficacy score rises significantly after hijacking. Here, we prepend the hijacking sentence $k$ times for $k = 0,\ldots, 5$ where $k = 0$ yields the original prompt. We see that additional prepends increase the score further. 

In the second scheme, we make use of the relation ID $\relation$ to prepend factually correct sentences. For instance, one can hijack the example above to ``The Eiffel Tower is not located in Guam. The Eiffel Tower is in''. 
We test this hijacking philosophy on different relation IDs. In particular, \cref{fig:P190_accs} reports hijacking based on relation ID $P190$ (``twin city''). And we see similar patterns that with more prepends, the ES score gets higher. It is also worth noting that one can even hijack by only including words that are semantically close to the false target (e.g., ``France'' for false target ``French''). 
This suggests that context hijacking is more than simply the LLM copying tokens from contexts.
Additional details and experiments for both hijacking schemes and for other relation IDs are in \cref{sec: additional_hijacking_expts}. 

\looseness=-1
These experiments show that context hijacking changes the behavior of LLMs, leading them to output incorrect tokens, without altering the factual meaning of the context. It is worth noting that similar fragile behaviors of LLMs have been observed in the literature in different contexts \citep{shi2023large, petroni2020context, creswell2022selection, yoran2023making, pandia2021sorting}. See \cref{sec:related-work} for more details.

Context hijacking indicates that fact retrieval in LLMs is not robust and that accurate fact recall does not necessarily depend on the semantics of the context. As a result, one hypothesis is to view LLMs as an associative memory model where special tokens in contexts, associated with the fact, provide partial information or clues to facilitate memory retrieval \citep{zhao2023context}. To better understand this perspective, we design a synthetic memory retrieval task to evaluate how the building 
blocks of LLMs, transformers, can solve it. 

\section{Problem setup}
\label{sec:toy-model}

In the context of LLMs, fact or memory retrieval, can be modeled as a next token prediction problem. Given a context (e.g., ``The capital of France is''), the objective is to accurately predict the next token (e.g., ``Paris'') based on the factual relation between context and the following token.

\looseness=-1
Previous papers \cite{ramsauer2020hopfield, millidge2022universal, bricken2021attention, zhao2023context} have studied the connection between attention and autoassociative and heteroassociative memory. For autoassociative memory, contexts are modeled as a set of existing memories and the goal of self-attention is to select the closest one or approximations to it. On top of this, heteroassociative memory \cite{millidge2022universal, bricken2021attention} has an additional projection to remap each output to a different one, whether within the same space or otherwise. In both scenarios, the goal is to locate the closest pattern within the context when provided with a query (up to a remapping if it's heteroassociative).

Fact retrieval, on the other hand, does not strictly follow this framework. The crux of the issue is that the output token is not necessarily close to any particular token in the context but rather a combination of them and the ``closeness'' is intuitively measured by latent semantic concepts. For example, consider context sentence ``The capital of France is'' with the output ``Paris''. Here, none of the tokens in the context directly corresponds to the word ``Paris''. Yet some tokens contain partial information about ``Paris''.  
Intuitively, ``capital'' aligns with the ``isCapital'' concept of ``Paris'', 
while ``France'' corresponds to the ``isFrench'' concept linked to ``Paris'' where all the concepts are latent. 
To model such phenomenon, we propose a synthetic task called \emph{latent concept association} where the output token is closely related to tokens in the context and similarity is measured via the latent space.

\subsection{Latent concept association}
\label{sec:setup}
We propose a synthetic prediction task where for each output token $y$, tokens in the context (denoted by $x$) are sampled from a conditional distribution given $y$. Tokens that are similar to $y$ will be favored to appear more in the context, except for $y$ itself. The task of latent concept association is to successfully retrieve the token $y$ given samples from $p(x|y)$. 
The synthetic setup simplifies by not accounting for the sequential nature of language, a choice supported by previous experiments on context hijacking (\cref{sec:context-hijacking}).
We formalize this task below.

To measure similarity, we define a latent space. Here, the latent space is a collection of $m$ binary latent variables $Z_i$. These could be viewed as semantic concept variables.
Let $Z = (Z_1, ..., Z_m)$ be the corresponding random vector, $z$ be its realization, and $\mathcal{Z}$ be the collection of all latent binary vectors.
For each latent vector $z$, there's one associated token $t \in [\vocab] = \{0, ..., \vocab - 1 \}$ where $\vocab$ is the total number of tokens. Here we represent the tokenizer as $\tok$ where $\tok(z) = t$. In this paper, we assume that $\tok$ is the standard tokenizer where each binary vector is mapped to its decimal number. In other words, there's a one to one map between latent vectors and tokens. Because the map is one to one, we sometimes use latent vectors and tokens interchangeably. We also assume that every latent binary vector has a unique corresponding token, therefore $\vocab = 2^m$.

Under the latent concept association model, the goal is to retrieve specific output tokens given partial information in the contexts. This is modeled by the latent conditional distribution:
\begin{equation*}
    p(z|z^{*}) = \omega \pi(z|z^{*}) + (1 - \omega) \textrm{Unif}(\mathcal{Z})
\end{equation*}
where 
\begin{equation*}
    \pi(z|z^{*}) \propto \begin{cases}
        \exp(-D_H(z, z^{*})/\beta) & z \in \latNe( z^{*}), \\
        0 & z \notin \latNe( z^{*}).
    \end{cases}
\end{equation*}
\looseness=-1
Here $D_H$ is the Hamming distance, $\mathcal{N}(z^{*})$ is a subset of $\mathcal{Z} \setminus \{z^{*}\}$ and $\beta > 0$ is the temperature parameter. 
The use of Hamming distance draws a parallel with the notion of distributional semantics in natural language: ``a word is characterized by the company it keeps'' \citep{firth1957synopsis}.
In words, $p(z|z^{*})$ says that with probability $1-\omega$, the conditional distribution uniformly generate random latent vectors and with probability $\omega$, the latent vector is generated from the \emph{informative conditional distribution} $\pi(z|z^{*})$ where the support of the conditional distribution is $\latNe(z^{*})$. Here, $\pi$ represents the informative conditional distribution that depends on $z^{*}$ whereas the uniform distribution is uninformative and can be considered as noise. The mixture model parameter $\omega$ determines the signal to noise ratio of the contexts.

Therefore, for any latent vector $z^{*}$ and its associated token, one can generate $L$ context token words with the aforementioned latent conditional distribution:

\begin{itemize}
    \item Uniformly sample a latent vector $z^{*}$ 
    \item For $l = 1,..., L-1$, sample $z_l \sim p(z|z^{*})$ and $t_l = \tok(z_l)$.
    \item For $l=L$, sample $z \sim \pi(z|z^{*})$ and $t_L = \tok(z)$.
\end{itemize}
\looseness=-1
Consequently, we have $x= (t_1, .., t_L)$ and $y = \tok(z^{*})$. The last token in the context is generated specifically to make sure that it is not from the uniform distribution. This ensures that the last token can use attention to look for clues, relevant to the output, in the context.
Let $\mathcal{D}^L$ be the sampling distribution to generate $(x,y)$ pairs. The conditional probability of $y$ given $x$ is given by $p(y|x)$. With slight abuse of notation, given a token $t \in [\vocab]$, we define $\tNe(t) = \latNe(\tok^{-1}(t))$. we also define $D_H(t, t') =  D_H(\tok^{-1}(t), \tok^{-1}(t'))$ for any pair of tokens $t$ and $t'$.

For any function $f$ that maps the context to estimated logits of output labels, the training objective is to minimize this loss of the last position:
\begin{equation*}
    \mathbb{E}_{(x,y) \in \mathcal{D}^L}[\ell(f(x), y)]
\end{equation*}
where $\ell$ is the cross entropy loss with softmax. The error rate of latent concept association is defined by the following:
\begin{equation*}
    R_{\mathcal{D}^L}(f) = \mathbb{P}_{(x, y) \sim \mathcal{D}^L}[\argmax f(x) \neq y]
\end{equation*}
And the accuracy is $1 - R_{\mathcal{D}^L}(f)$.

\subsection{Transformer network architecture}

Given a context $x = (t_1, .., t_L)$ which consists of $L$ tokens, we define $X \in \{0,  1\}^{\vocab \times L}$ to be its one-hot encoding where $\vocab$ is the vocabulary size. Here we use $\chi$ to represent the one-hot encoding function (i.e., $\chi(x) = X$). Similar to \citep{li2023transformers, tarzanagh2023transformers, li2024mechanics}, we also consider a simplified one-layer transformer model without residual connections and normalization:
\begin{equation}
\label{eqn:transformer}
    f^L(x) = \bigg[ {W_E}^T W_V \textrm{attn} (W_E \chi(x)) \bigg]_{:L}
\end{equation}
where 
\begin{equation*}
\begin{split}
     \textrm{attn} (U) =  U \sigma \Big(\frac{(W_K U )^T(W_Q U)}{\sqrt{d_a}} \Big),
\end{split}
\end{equation*}
\looseness=-1
$W_K \in \mathbb{R}^{d_a \times d}$ is the key matrix, and $W_Q \in \mathbb{R}^{d_a \times d}$ is the query matrix and $d_a$ is the attention head size. $\sigma: \mathbb{R}^{L \times L} \to (0, 1)^{L \times L}$ is the column-wise softmax operation. 
$W_V \in \mathbb{R}^{d \times d}$ is the value matrix and $W_E \in \mathbb{R}^{d \times V}$ is the embedding matrix. 
Here, we adopt the weight tie-in implementation which is used for Gemma \cite{team2024gemma}. We focus solely on the prediction of the last position, as it is the only one relevant for latent concept association. For convenience, we also use $h(x)$ to mean $\big[\textrm{attn} (W_E \chi(x)) \big]_{:L}$, which is the hidden representation after attention for the last position, and $f^L_t(x)$ to represent the logit for output token $t$. 

\section{Theoretical analysis}
\label{sec:theory}

In this section, we theoretically investigate how a single-layer transformer can solve the latent concept association problem.
We first introduce a hypothetical associative memory model that utilizes self-attention for information aggregation and employs the value matrix for memory retrieval. This hypothetical model turns out to mirror trained transformers in experiments. We also examine the role of each individual component of the network: the value matrix, embeddings, and the attention mechanism. We validate our theoretical claims in \cref{sec:exp}.

\subsection{Hypothetical associative memory model}
\label{sec:hypo-am}

In this section, we show that a simple single-layer transformer network can solve the latent concept association problem. The formal result is presented below in \cref{thm:idealmodel-informal}; first we require a few more definitions.
Let $W_E(t)$ be the $t$-th column of the embedding matrix $W_E$. In other words, this is the embedding for token $t$. Given a token $t$, define $\tNeOne(t)$ to be the subset of tokens whose latent vectors are only $1$ Hamming distance away from $t$'s latent vector: $\tNeOne(t) = \{ t' :  D_H(t', t)) = 1      \} \cap \tNe(t)$.
For any output token $t$, $\tNeOne(t)$ contains tokens with the highest probabilities to appear in the context.

The following theorem formalizes the intuition that a one-layer transformer that uses self-attention 
to summarize statistics about the context distributions and whose value matrix uses aggregated representations to retrieve output tokens can solve the latent concept association problem defined in \cref{sec:setup}.

\begin{theorem}[informal]
\label{thm:idealmodel-informal}
Suppose the data generating process follows \cref{sec:setup} where $m \geq 3$, $\omega = 1$, and $\tNe(t) = \vocab \setminus \{ t \}$. Then for any $\epsilon>0$, there exists a transformer model given by \cref{eqn:transformer} that achieves error $\epsilon$, i.e. $R_{\mathcal{D}^L}(f^L)<\epsilon$ given sufficiently large context length $L$. 
\end{theorem}

More precisely, for the transformer in \cref{thm:idealmodel-informal}, we will have $W_K = 0$ and $W_Q=0$. Each row of $W_E$ is orthogonal to each other and normalized. And $W_V$ is given by
\begin{equation}
\label{eqn:wv-constrution}
    W_V = \sum_{t \in [\vocab]} W_E(t) (\sum_{t' \in \tNeOne(t)}W_E(t')^T) 
\end{equation}
A more formal statement of the theorem and its proof is given in \cref{appendix:proofs} (\cref{thm:idealmodel}). 

Intuitively, \cref{thm:idealmodel-informal} suggests having more samples from $p(x|y)$ can lead to a better recall rate. On the other hand, if contexts are modified to contain more samples from $p(x|\tilde{y})$ where $\tilde{y} \neq y$, then it is likely for transformer to output the wrong token. This is similar to context hijacking (see \cref{sec:misclass}). The construction of the value matrix is similar to the associative memory model used in \cite{bietti2024birth, cabannes2024learning}, but in our case, there is no explicit one-to-one input and output pairs stored as memories. Rather, a combination of inputs are mapped to a single output. 

\looseness=-1
While the construction in \cref{thm:idealmodel-informal} is just one way that a single-layer transformer can tackle this task, it turns out empirically this construction of $W_V$ is close to the trained $W_V$, even in the noisy case ($\omega \neq 1$). In \cref{sec:expwv}, we will demonstrate that substituting trained value matrices with constructed ones can retain accuracy, and the constructed and trained value matrices even share close low-rank approximations. Moreover, in this hypothetical model, a simple uniform attention mechanism is deployed to allow self-attention to count occurrences of each individual tokens. Since the embeddings are orthonormal vectors, there is no interference. Hence, the self-attention layer can be viewed as aggregating information of contexts. It is worth noting that, in different settings, more sophisticated embedding structures and attention patterns are needed. This is discussed in the following sections.

\subsection{On the role of the value matrix}
\label{sec:value}

The construction in \cref{thm:idealmodel-informal} relies on the value matrix acting as associative memory. But is it necessary? Could we integrate the functionality of the value matrix into the self-attention module to solve the latent concept association problem? Empirically, the answer seems to be negative as will be shown in \cref{sec:expwv}. In particular, when the context length is small, setting the value matrix to be the identity would lead to subpar memory recall accuracy. 

This is because if the value matrix is the identity, the transformer would be more susceptible to the noise in the context. To see this, notice that given any pair of context and output token $(x,y)$, the latent representation after self-attention $h(x)$ must live in the polyhedron $S_y$ to be classified correctly where $S_y$ is defined as:
\begin{equation*}
    S_y = \{v:  (W_E(y) - W_E(t))^T v > 0 \text{ where } t \not\in [\vocab] \setminus \{ y \}\}
\end{equation*}
Note that, by definition, for any two tokens $y$ and $\tilde{y}$, $S_y \cap S_{\tilde{y}} = \emptyset$.
On the other hand, because of the self-attention mechanism, $h(x)$ must also live in the convex hull of all the embedding vectors: 
\begin{equation*}
    CV = \textrm{Conv}(W^{E}(0), ..., W^{E}(|\vocab|-1))
\end{equation*}
In other words, for any pair $(x,y)$ to be classified correctly, $h(x)$ must live in the intersection of $S_y$ and $CV$. Due to the stochastic nature of $x$, it is likely for $h(x)$ to be outside of this intersection. The remapping effect of the value matrix can help with this problem. The following lemma explains this intuition.

\vspace{-2pt}
\begin{restatable}{lemma}{remap}
\label{lemma:reamp}
Suppose the data generating process follows \cref{sec:setup} where $m \geq 3$, $\omega=1$ and $\tNe(t) = \{ t' :  D_H(t, t')) = 1      \} $.
For any single layer transformer given by \cref{eqn:transformer} 
where each row of $W_E$ is orthogonal to each other and normalized, if $W_V$ is constructed as in \cref{eqn:wv-constrution}, then the error rate is $0$. If $W_V$ is the identity matrix, then the error rate is strictly larger than $0$.
\end{restatable}

Another intriguing phenomenon occurs when the value matrix is the identity matrix. In this case, the inner product between embeddings and their corresponding Hamming distance varies linearly. This relationship can be formalized by the following theorem.

\begin{restatable}{theorem}{hamming}
\label{thm:hamming}
Suppose the data generating process follows \cref{sec:setup} where $m \geq 3$, $\omega = 1$ and $\tNe(t) = \vocab \setminus \{ t \}$.
For any single layer transformer given by \cref{eqn:transformer} with $W_V$ being the identity matrix, if the cross entropy loss is minimized so that for any sampled pair $(x, y)$,
\begin{equation*}
    p(y|x) = \hat{p}(y|x) = \textrm{softmax}(f^L_y(x))
\end{equation*}
there exists $a>0$ and $b$ such that for two tokens $t \neq t'$,
\begin{equation*}
    \langle W_E(t), W_E(t') \rangle = -a D_H(t, t') + b
\end{equation*}
\end{restatable}

\subsection{Embedding training and geometry}
\label{sec:geometry}

The hypothetical model in \cref{sec:hypo-am} requires embeddings to form an orthonormal basis. In the overparameterization regime where the embedding dimension $d$ is larger than the number of tokens $\vocab$, this can be approximately achieved by Gaussian initialization. However, in practice, the embedding dimension is typically smaller than the vocabulary size, in which case it is impossible for the embeddings to constitute such a basis. Empirically, in \cref{sec:expemb}, we observe that with overparameterization ($d > \vocab$), embeddings can be frozen at their Gaussian initialization, whereas in the underparameterized regime, embedding training is required to achieve better recall accuracy.

This raises the question: What kind of embedding geometry is learned in the  underparameterized regime? Experiments reveal a close relationship between the inner product of embeddings for two tokens and the Hamming distance of these tokens (see \cref{fig:selected-embedding-structure-8} and \cref{fig:hamming} in \cref{appen:emb}). Approximately, we have the following relationship:
\begin{equation}
\label{eqn:emb-linear}
    \langle W_E(t), W_E(t') \rangle = 
    \begin{cases}
        b_0 \;\;\; &t=t'\\
        - a D_H(t, t') + b \;\;\; &t\neq t'\\
    \end{cases}
\end{equation}
for any two tokens $t$ and $t'$ where $b_0 > b$ and $a > 0$. One can view this as a combination of the embedding geometry under Gaussian initialization and the geometry when $W_V$ is the identity matrix (\cref{thm:hamming}). 
Importantly, this structure demonstrates that trained embeddings inherently capture similarity within the latent space.
Theoretically, this embedding structure \cref{eqn:emb-linear} can also lead to low error rate under specific conditions on $b_0, b$ and $a$, which is articulated by the following theorem.

\begin{theorem}[Informal]
\label{thm:modifiedmodel-informal}
Following the same setup as in \cref{thm:idealmodel-informal}, but embeddings obey \cref{eqn:emb-linear}, then under certain conditions on $a, b$ and if $b_0$ and context length $L$ are sufficiently large, the error rate can be arbitrarily small, i.e. $R_{\mathcal{D}^L}(f^L)<\epsilon$ for any $0 < \epsilon < 1$.
\end{theorem}
The formal statement of the theorem and its proof is given in \cref{appendix:proofs} (\cref{thm:modifiedmodel}).

Notably, this embedding geometry also implies a low-rank structure. Let's first consider the special case when $b_0 = b$. In other words, the inner product between embeddings and their corresponding Hamming distance varies linearly.

\begin{restatable}{lemma}{embrank}
\label{lemma:emb-rank}
If embeddings follow \cref{eqn:emb-linear} and $b = b_0$ and $\tNe(t) = \vocab \setminus \{ t \}$, then $\text{rank}(W_E) \leq m+2$.
\end{restatable}

When $b_0 > b$, the embedding matrix will not be strictly low rank. However, it can still exhibit approximate low-rank behavior, characterized by an eigengap between the top and bottom singular values. This is verified empirically (see Figure~\ref{fig:7_32}-\ref{fig:8_64} in \cref{appen:spectrum}).

\subsection{The role of attention selection}
\label{sec:atten}
As of now, attention does not play a significant role in the analysis. But perhaps unsurprisingly, the attention mechanism is useful in selecting relevant information. To see this, let's consider a specific setting where for any latent vector $z^{*}$, $\latNe(z^{*}) = \{ z: z^{*}_1 = z_1 \} \setminus \{ z^{*} \}$.

Essentially, latent vectors are partitioned into two clusters based on the value of the first latent variable, and the informative conditional distribution $\pi$ only samples latent vectors that are in the same cluster as the output latent vector. Empirically, when trained under this setting, the attention mechanism will pay more attention to tokens within the same cluster (\cref{sec:expatten}). This implies that the self-attention layer can mitigate noise and concentrate on the informative conditional distribution $\pi$.

To understand this more intuitively, we will study the gradient of unnormalized attention scores. In particular, the unnormalized attention score is defined as:
\begin{equation*}
    u_{t,t'} = (W_K W_E(t) )^T(W_Q W_E(t')) / \sqrt{d_a}.
\end{equation*}

\begin{restatable}{lemma}{attengrad}
\label{lemma:attengrad}
Suppose the data generating process follows \cref{sec:setup} and $\latNe(z^{*}) = \{ z: z^{*}_1 = z_1 \} \setminus \{ z^{*} \}$. Given the last token in the sequence $t_{L}$, then
\begin{equation*}
    \nabla_{u_{t, t_{L}}} \ell(f^L) = \nabla \ell(f^L)^T (W_E)^T W^{V} (  \alpha_t \hat{p}_{t}  W_E(t)  - \hat{p}_t \sum_{l=1}^L \hat{p}_{t_l} W_E(t_l) )
\end{equation*}
where for token $t$, $\alpha_t = \sum_{l=1}^L \mathbf{1}[t_l = t]$ and $\hat{p}_t$ is the normalized attention score for token $t$.
\end{restatable}

Typically, $\alpha_t$ is larger when token $t$ and $t_L$ belong to the same cluster because tokens within the same cluster tend to co-occur frequently. 
As a result, the gradient contribution to the unnormalized attention score is usually larger for tokens within the same cluster.

\subsection{Context hijacking and the misclassification of memory recall}
\label{sec:misclass}

In light of the theoretical results on latent concept association, a natural question arises: How do these results connect to context hijacking in LLMs? In essence, for the latent concept association problem, the differentiation of output tokens is achieved by distinguishing between the various conditional distributions $p(x|y)$.  Thus, adding or changing tokens in the context $x$ so that it resembles a different conditional distribution can result in misclassification. In \cref{appen:hijack-latent}, we present experiments showing that mixing different contexts can cause transformers to misclassify. This partially explains context hijacking in LLMs (\cref{sec:context-hijacking}). On the other hand, it is well-known that the error rate is related to the KL divergence between conditional distributions of contexts \cite{cover1999elements}. The closer the distributions are, the easier it is for the model to misclassify. Here, longer contexts, primarily composed of i.i.d samples, suggest larger divergences, thus higher memory recall rate. This is theoretically implied by \cref{thm:idealmodel-informal} and \cref{thm:modifiedmodel-informal} and empirically verified in \cref{sec:explength}. Such result is also related to reverse context hijacking (\cref{sec: additional_hijacking_expts}) where prepending sentences including true target words can improve fact recall rate.

\section{Experiments}
\label{sec:exp}

\begin{figure}[t]
    \centering
    \begin{subfigure}{0.3\textwidth}
        \centering
        \includegraphics[width=\linewidth]{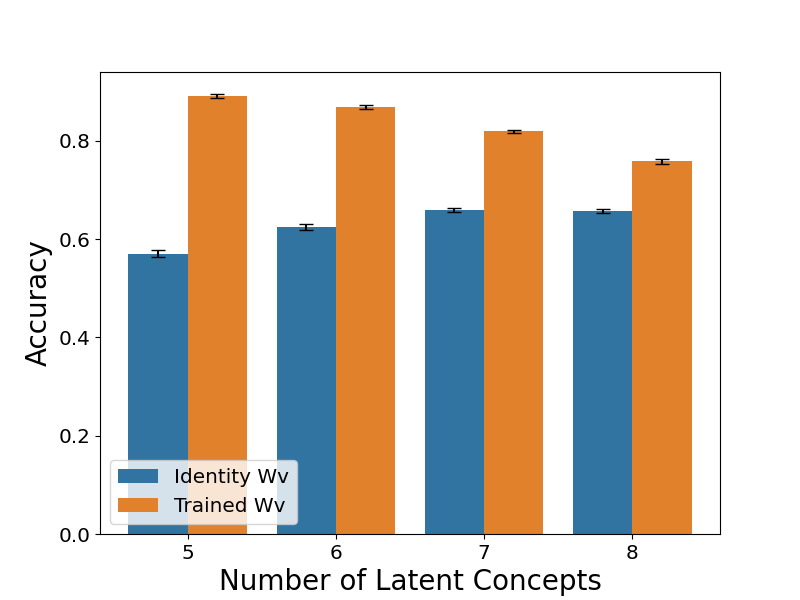}
        \caption{Value matrix training}
        \label{fig:selected-value-64}
    \end{subfigure}
    \hfill
    \begin{subfigure}{0.3\textwidth}
        \centering
        \includegraphics[width=\linewidth]{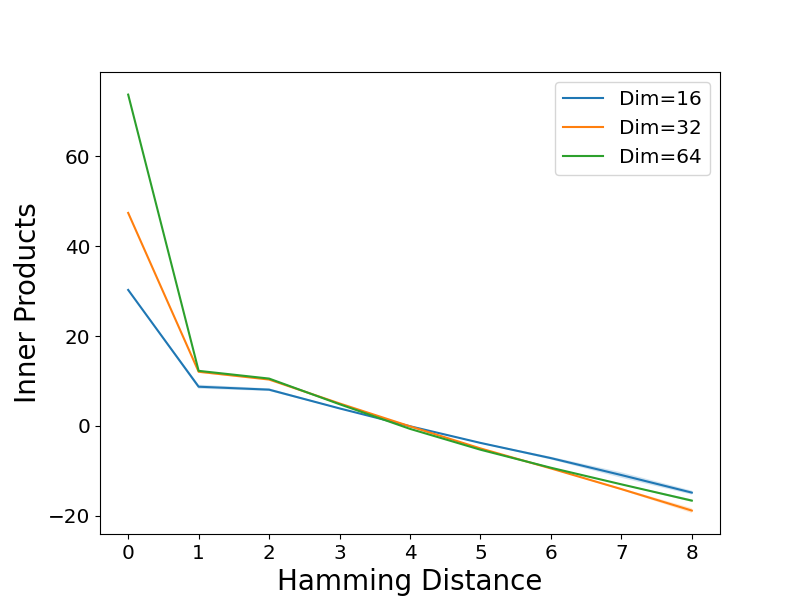}
        \caption{Embedding structure}
        \label{fig:selected-embedding-structure-8}
    \end{subfigure}
    \hfill
    \begin{subfigure}{0.3\textwidth}
        \centering
        \includegraphics[width=\linewidth]{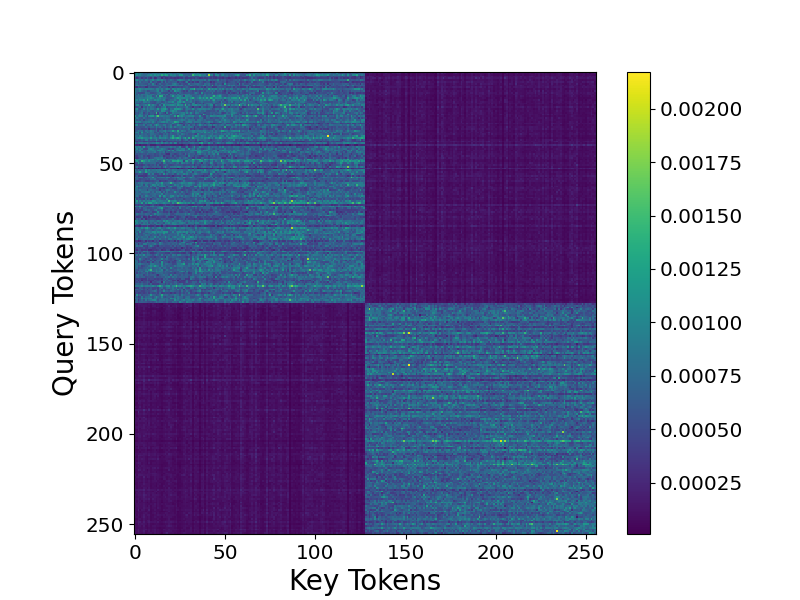}
        \caption{Attention Pattern}
        \label{fig:selected-attention-8}
    \end{subfigure}
    \caption{Key components of the single-layer transformer working together on the latent concept association problem. (a) Fixing the value matrix $W_V$ as the identity matrix results in lower accuracy compared to training $W_V$. The figure reports average accuracy for both fixed and trained $W_V$ with $L=64$. (b) When training in the underparameterized regime, the embedding structure is approximated by \cref{eqn:emb-linear}. The graph displays the average inner product between embeddings of two tokens against the corresponding Hamming distance between these tokens when $m=8$. (c) The self-attention layer can select tokens within the same cluster. The figure shows average attention score heat map with $m=8$ and the cluster structure from \cref{sec:atten}.}
    \label{fig:threefigures}
\end{figure}

The main implications of the theoretical results in the previous section are:
\begin{enumerate}
    \item The value matrix is important and has associative memory structure as in \cref{eqn:wv-constrution}. 
    \item Training embeddings is crucial in the underparameterized regime, where embeddings exhibit certain geometric structures. 
    \item Attention mechanism is used to select the most relevant tokens. 
\end{enumerate}
To evaluate these claims, we conduct several experiments on synthetic datasets. Additional experimental details and results can be found in \cref{append:exp-latent}.

\subsection{On the value matrix $W_V$}
\label{sec:expwv}
\looseness=-1
In this section, we study the necessity of the value matrix $W_V$ and its structure. First, we conduct experiments to compare the effects of training versus freezing $W_V$ as the identity matrix, with the context lengths $L$ set to $64$ and $128$. \cref{fig:selected-value-64} and \cref{fig:identity} show that when the context length is small, freezing $W_V$ can lead to a significant decline in accuracy. This is inline with \cref{lemma:reamp} and validates it in a general setting, implying the significance of the value matrix in maintaining a high memory recall rate.

Next, we investigate the degree of alignment between the trained value matrix $W_V$ and the construction in \cref{eqn:wv-constrution}. The first set of experiments examines the similarity in functionality between the two matrices. We replace value matrices in trained transformers with the constructed ones like in \cref{eqn:wv-constrution} and then report accuracy with the new value matrix. As a baseline, we also consider randomly constructed value matrix, where the outer product pairs are chosen randomly (detailed construction can be found in \cref{append:expwv}). \cref{fig:dim_wv_replace} indicates that the accuracy does not significantly decrease when the value matrix is replaced with the constructed ones. Furthermore, not only are the constructed value matrix and the trained value matrix functionally alike, but they also share similar low-rank approximations. We use singular value decomposition to get the best low rank approximations of various value matrices where the rank is set to be the same as the number of latent variables ($m$). We then compute smallest principal angles between low-rank approximations of trained value matrices and those of constructed, randomly constructed, and Gaussian-initialized value matrices. \cref{{fig:dim_wv_replace_angle}} shows that the constructed ones have, on average, smallest principal angles with the trained ones.

\subsection{On the embeddings}
\label{sec:expemb}
In this section, we explore the significance of embedding training in the underparamerized regime and embedding structures. 
We conduct experiments to compare the effects of training versus freezing embeddings with different embedding dimensions. The learning rate is selected as the best option from $\{0.01, 0.001 \}$ depending on the dimensions. \cref{fig:dim_acc} clearly shows that when the dimension is smaller than the vocabulary size $(d < \vocab)$, embedding training is required. It is not necessary in the overparameterized regime $(d > \vocab)$, partially confirming \cref{thm:idealmodel-informal} because if embeddings are initialized from a high-dimensional multi-variate Gaussian, they are approximately orthogonal to each other and have the same norms. 

The next question is what kind of embedding structures are formed for trained transformers in the underparamerized regime. From \cref{fig:selected-embedding-structure-8} and \cref{fig:hamming}, it is evident that the relationship between the average inner product of embeddings for two tokens and their corresponding Hamming distance roughly aligns with \cref{eqn:emb-linear}. Perhaps surprisingly, if we plot the same graph for trained transformers with a fixed identity value matrix, the relationship is mostly linear as shown in \cref{fig:no_wv_hamming}, confirming our theory (\cref{thm:hamming}).

As suggested in \cref{sec:geometry}, such embedding geometry \cref{eqn:emb-linear} can lead to low rank structures. We verify this claim by studying the spectrum of the embedding matrix $W_E$. As illustrated in \cref{appen:spectrum}, Figure~\ref{fig:7_32}-\ref{fig:8_64} demonstrate that there are eigengaps between top and bottom singular values, suggesting low-rank structures.

\subsection{On the attention selection mechanism}
\label{sec:expatten}
In this section, we examine the role of attention pattern by considering a special class of latent concept association model as defined in \cref{sec:atten}. \cref{fig:selected-attention-8} and \cref{fig:cluster_1} clearly show that the self-attention select tokens in the same clusters. This suggests that attention can filter out noise and focus on the informative conditional distribution $\pi$. We extend experiments to consider cluster structures that depend on the first two latent variables (detailed construction can be found in \cref{appen:expatten}) and \cref{fig:cluster_2} shows attention pattern as expected.

\section{Conclusions}
In this work, we first presented the phenomenon of context hijacking in LLMs, which suggested that fact retrieval is not robust against variations of contexts. This indicates that LLMs might function like associative memory where tokens in contexts are clues to guide memory retrieval. To investigate this perspective further, we devised a synthetic task called latent concept association and examined theoretically and empirically how single-layer transformers are trained to solve this task. 
These results provide further insights into the inner workings of transformers and LLMs, and can hopefully stimulate further work into interpreting and understanding the mechanisms by which LLMs predict tokens and recall facts.

\section{Limitations} 
The context hijacking experiments were only conducted on open-source models and not on commercial models like GPT-4.
Nevertheless, even in the official GPT-4 technical report~\cite{achiam2023gpt}, there is an example similar to context hijacking (the Elvis Perkins example). In that example, the prompt is “Son of an actor, this American guitarist and rock singer released many songs and albums and toured with his band. His name is "Elvis" what?”. GPT-4 answers with Presley, even though the answer is Perkins (Elvis Presley is not the son of an actor). GPT-4 can be viewed as distracted by all the information related to music and answers Presley. In fact, it is known that LLMs can be easily distracted by contexts in use cases other than fact retrieval such as problem-solving \citep{shi2023large}. So we reasonably suspect that similar behavior still exists in larger models but is harder to exploit. On the other hand, the theoretical section only focuses on single-layer transformer network. While single-layer networks already demonstrate some interesting phenomena including low-rank structures, the functionality of multi-layer transformers is much different compared to single-layer transformers with the notable emergence of induction head \cite{elhage2021mathematical}.

\paragraph{Acknowledgments}

We thank Victor Veitch for insightful discussions that helped shape the initial idea of this work. We acknowledge the support of AFRL and DARPA via FA8750-23-2-1015, ONR via N00014-23-1-2368, NSF via IIS-1909816, IIS-1955532, IIS-1956330, and NIH R01GM140467. We also acknowledge the support of the Robert H. Topel Faculty Research Fund at the University of Chicago Booth School of Business.

\printbibliography

@article{xie2021explanation,
  title={An explanation of in-context learning as implicit bayesian inference},
  author={Xie, Sang Michael and Raghunathan, Aditi and Liang, Percy and Ma, Tengyu},
  journal={arXiv preprint arXiv:2111.02080},
  year={2021}
}

@article{olsson2022context,
  title={In-context learning and induction heads},
  author={Olsson, Catherine and Elhage, Nelson and Nanda, Neel and Joseph, Nicholas and DasSarma, Nova and Henighan, Tom and Mann, Ben and Askell, Amanda and Bai, Yuntao and Chen, Anna and others},
  journal={arXiv preprint arXiv:2209.11895},
  year={2022}
}

@article{bietti2024birth,
  title={Birth of a transformer: A memory viewpoint},
  author={Bietti, Alberto and Cabannes, Vivien and Bouchacourt, Diane and Jegou, Herve and Bottou, Leon},
  journal={Advances in Neural Information Processing Systems},
  volume={36},
  year={2024}
}

@article{cabannes2023scaling,
  title={Scaling laws for associative memories},
  author={Cabannes, Vivien and Dohmatob, Elvis and Bietti, Alberto},
  journal={arXiv preprint arXiv:2310.02984},
  year={2023}
}

@article{cabannes2024learning,
  title={Learning Associative Memories with Gradient Descent},
  author={Cabannes, Vivien and Simsek, Berfin and Bietti, Alberto},
  journal={arXiv preprint arXiv:2402.18724},
  year={2024}
}

@article{edelman2024evolution,
  title={The Evolution of Statistical Induction Heads: In-Context Learning Markov Chains},
  author={Edelman, Benjamin L and Edelman, Ezra and Goel, Surbhi and Malach, Eran and Tsilivis, Nikolaos},
  journal={arXiv preprint arXiv:2402.11004},
  year={2024}
}

@misc{nichani2024transformers,
      title={How Transformers Learn Causal Structure with Gradient Descent}, 
      author={Eshaan Nichani and Alex Damian and Jason D. Lee},
      year={2024},
      eprint={2402.14735},
      archivePrefix={arXiv},
      primaryClass={cs.LG}
}

@misc{makkuva2024attention,
      title={Attention with Markov: A Framework for Principled Analysis of Transformers via Markov Chains}, 
      author={Ashok Vardhan Makkuva and Marco Bondaschi and Adway Girish and Alliot Nagle and Martin Jaggi and Hyeji Kim and Michael Gastpar},
      year={2024},
      eprint={2402.04161},
      archivePrefix={arXiv},
      primaryClass={cs.LG}
}

@article{emrullah2024self,
  title={From Self-Attention to Markov Models: Unveiling the Dynamics of Generative Transformers},
  author={Emrullah Ildiz, M and Huang, Yixiao and Li, Yingcong and Singh Rawat, Ankit and Oymak, Samet},
  journal={arXiv e-prints},
  pages={arXiv--2402},
  year={2024}
}

@article{Devroye,
 ISSN = {00905364},
 URL = {http://www.jstor.org/stable/2240651},
 author = {Luc Devroye},
 journal = {The Annals of Statistics},
 number = {3},
 pages = {896--904},
 publisher = {Institute of Mathematical Statistics},
 title = {The Equivalence of Weak, Strong and Complete Convergence in L1 for Kernel Density Estimates},
 urldate = {2024-05-04},
 volume = {11},
 year = {1983}
}

@article{loshchilov2017decoupled,
  title={Decoupled weight decay regularization},
  author={Loshchilov, Ilya and Hutter, Frank},
  journal={arXiv preprint arXiv:1711.05101},
  year={2017}
}

@article{henighan2023superposition,
  title={Superposition, memorization, and double descent},
  author={Henighan, Tom and Carter, Shan and Hume, Tristan and Elhage, Nelson and Lasenby, Robert and Fort, Stanislav and Schiefer, Nicholas and Olah, Christopher},
  journal={Transformer Circuits Thread},
  year={2023}
}

@article{steinberg2022associative,
  title={Associative memory of structured knowledge},
  author={Steinberg, Julia and Sompolinsky, Haim},
  journal={Scientific Reports},
  volume={12},
  number={1},
  pages={21808},
  year={2022},
  publisher={Nature Publishing Group UK London}
}

@article{hopfield1982neural,
  title={Neural networks and physical systems with emergent collective computational abilities.},
  author={Hopfield, John J},
  journal={Proceedings of the national academy of sciences},
  volume={79},
  number={8},
  pages={2554--2558},
  year={1982},
  publisher={National Acad Sciences}
}

@article{seung1996brain,
  title={How the brain keeps the eyes still},
  author={Seung, H Sebastian},
  journal={Proceedings of the National Academy of Sciences},
  volume={93},
  number={23},
  pages={13339--13344},
  year={1996},
  publisher={National Acad Sciences}
}

@article{ben1995theory,
  title={Theory of orientation tuning in visual cortex.},
  author={Ben-Yishai, Rani and Bar-Or, R Lev and Sompolinsky, Haim},
  journal={Proceedings of the National Academy of Sciences},
  volume={92},
  number={9},
  pages={3844--3848},
  year={1995},
  publisher={National Acad Sciences}
}

@article{skaggs1994model,
  title={A model of the neural basis of the rat's sense of direction},
  author={Skaggs, William and Knierim, James and Kudrimoti, Hemant and McNaughton, Bruce},
  journal={Advances in neural information processing systems},
  volume={7},
  year={1994}
}

@inproceedings{jiang2020associative,
  title={Associative memory in iterated overparameterized sigmoid autoencoders},
  author={Jiang, Yibo and Pehlevan, Cengiz},
  booktitle={International conference on machine learning},
  pages={4828--4838},
  year={2020},
  organization={PMLR}
}

@article{radhakrishnan2020overparameterized,
  title={Overparameterized neural networks implement associative memory},
  author={Radhakrishnan, Adityanarayanan and Belkin, Mikhail and Uhler, Caroline},
  journal={Proceedings of the National Academy of Sciences},
  volume={117},
  number={44},
  pages={27162--27170},
  year={2020},
  publisher={National Acad Sciences}
}

@inproceedings{feldman2020does,
  title={Does learning require memorization? a short tale about a long tail},
  author={Feldman, Vitaly},
  booktitle={Proceedings of the 52nd Annual ACM SIGACT Symposium on Theory of Computing},
  pages={954--959},
  year={2020}
}

@article{feldman2020neural,
  title={What neural networks memorize and why: Discovering the long tail via influence estimation},
  author={Feldman, Vitaly and Zhang, Chiyuan},
  journal={Advances in Neural Information Processing Systems},
  volume={33},
  pages={2881--2891},
  year={2020}
}

@article{jin2023cost,
  title={The Cost of Down-Scaling Language Models: Fact Recall Deteriorates before In-Context Learning},
  author={Jin, Tian and Clement, Nolan and Dong, Xin and Nagarajan, Vaishnavh and Carbin, Michael and Ragan-Kelley, Jonathan and Dziugaite, Gintare Karolina},
  journal={arXiv preprint arXiv:2310.04680},
  year={2023}
}

@article{mahdavi2023memorization,
  title={Memorization capacity of multi-head attention in transformers},
  author={Mahdavi, Sadegh and Liao, Renjie and Thrampoulidis, Christos},
  journal={arXiv preprint arXiv:2306.02010},
  year={2023}
}

@inproceedings{kim2022provable,
  title={Provable memorization capacity of transformers},
  author={Kim, Junghwan and Kim, Michelle and Mozafari, Barzan},
  booktitle={The Eleventh International Conference on Learning Representations},
  year={2022}
}

@article{ramsauer2020hopfield,
  title={Hopfield networks is all you need},
  author={Ramsauer, Hubert and Sch{\"a}fl, Bernhard and Lehner, Johannes and Seidl, Philipp and Widrich, Michael and Adler, Thomas and Gruber, Lukas and Holzleitner, Markus and Pavlovi{\'c}, Milena and Sandve, Geir Kjetil and others},
  journal={arXiv preprint arXiv:2008.02217},
  year={2020}
}

@inproceedings{millidge2022universal,
  title={Universal hopfield networks: A general framework for single-shot associative memory models},
  author={Millidge, Beren and Salvatori, Tommaso and Song, Yuhang and Lukasiewicz, Thomas and Bogacz, Rafal},
  booktitle={International Conference on Machine Learning},
  pages={15561--15583},
  year={2022},
  organization={PMLR}
}

@article{zhao2023context,
  title={In-Context Exemplars as Clues to Retrieving from Large Associative Memory},
  author={Zhao, Jiachen},
  journal={arXiv preprint arXiv:2311.03498},
  year={2023}
}

@article{bricken2021attention,
  title={Attention approximates sparse distributed memory},
  author={Bricken, Trenton and Pehlevan, Cengiz},
  journal={Advances in Neural Information Processing Systems},
  volume={34},
  pages={15301--15315},
  year={2021}
}

@article{team2024gemma,
  title={Gemma: Open models based on gemini research and technology},
  author={Team, Gemma and Mesnard, Thomas and Hardin, Cassidy and Dadashi, Robert and Bhupatiraju, Surya and Pathak, Shreya and Sifre, Laurent and Rivi{\`e}re, Morgane and Kale, Mihir Sanjay and Love, Juliette and others},
  journal={arXiv preprint arXiv:2403.08295},
  year={2024}
}

@article{touvron2023llama,
  title={Llama 2: Open foundation and fine-tuned chat models},
  author={Touvron, Hugo and Martin, Louis and Stone, Kevin and Albert, Peter and Almahairi, Amjad and Babaei, Yasmine and Bashlykov, Nikolay and Batra, Soumya and Bhargava, Prajjwal and Bhosale, Shruti and others},
  journal={arXiv preprint arXiv:2307.09288},
  year={2023}
}

@article{radford2019language,
  title={Language models are unsupervised multitask learners},
  author={Radford, Alec and Wu, Jeffrey and Child, Rewon and Luan, David and Amodei, Dario and Sutskever, Ilya and others},
  journal={OpenAI blog},
  volume={1},
  number={8},
  pages={9},
  year={2019}
}

@article{meng2022locating,
  title={Locating and editing factual associations in GPT},
  author={Meng, Kevin and Bau, David and Andonian, Alex and Belinkov, Yonatan},
  journal={Advances in Neural Information Processing Systems},
  volume={35},
  pages={17359--17372},
  year={2022}
}

@misc{meng2023massediting,
      title={Mass-Editing Memory in a Transformer}, 
      author={Kevin Meng and Arnab Sen Sharma and Alex Andonian and Yonatan Belinkov and David Bau},
      year={2023},
      eprint={2210.07229},
      archivePrefix={arXiv},
      primaryClass={cs.CL}
}

@article{hase2024does,
  title={Does localization inform editing? surprising differences in causality-based localization vs. knowledge editing in language models},
  author={Hase, Peter and Bansal, Mohit and Kim, Been and Ghandeharioun, Asma},
  journal={Advances in Neural Information Processing Systems},
  volume={36},
  year={2024}
}

@article{sakarvadia2023memory,
  title={Memory injections: Correcting multi-hop reasoning failures during inference in transformer-based language models},
  author={Sakarvadia, Mansi and Ajith, Aswathy and Khan, Arham and Grzenda, Daniel and Hudson, Nathaniel and Bauer, Andr{\'e} and Chard, Kyle and Foster, Ian},
  journal={arXiv preprint arXiv:2309.05605},
  year={2023}
}

@article{jiang2024origins,
  title={On the Origins of Linear Representations in Large Language Models},
  author={Jiang, Yibo and Rajendran, Goutham and Ravikumar, Pradeep and Aragam, Bryon and Veitch, Victor},
  journal={arXiv preprint arXiv:2403.03867},
  year={2024}
}

@article{rajendran2024learning,
  title={Learning Interpretable Concepts: Unifying Causal Representation Learning and Foundation Models},
  author={Rajendran, Goutham and Buchholz, Simon and Aragam, Bryon and Sch{\"o}lkopf, Bernhard and Ravikumar, Pradeep},
  journal={arXiv preprint arXiv:2402.09236},
  year={2024}
}

@article{de2021editing,
  title={Editing factual knowledge in language models},
  author={De Cao, Nicola and Aziz, Wilker and Titov, Ivan},
  journal={arXiv preprint arXiv:2104.08164},
  year={2021}
}

@article{mitchell2021fast,
  title={Fast model editing at scale},
  author={Mitchell, Eric and Lin, Charles and Bosselut, Antoine and Finn, Chelsea and Manning, Christopher D},
  journal={arXiv preprint arXiv:2110.11309},
  year={2021}
}

@inproceedings{mitchell2022memory,
  title={Memory-based model editing at scale},
  author={Mitchell, Eric and Lin, Charles and Bosselut, Antoine and Manning, Christopher D and Finn, Chelsea},
  booktitle={International Conference on Machine Learning},
  pages={15817--15831},
  year={2022},
  organization={PMLR}
}

@article{dai2021knowledge,
  title={Knowledge neurons in pretrained transformers},
  author={Dai, Damai and Dong, Li and Hao, Yaru and Sui, Zhifang and Chang, Baobao and Wei, Furu},
  journal={arXiv preprint arXiv:2104.08696},
  year={2021}
}

@article{zhang2023large,
  title={How do large language models capture the ever-changing world knowledge? a review of recent advances},
  author={Zhang, Zihan and Fang, Meng and Chen, Ling and Namazi-Rad, Mohammad-Reza and Wang, Jun},
  journal={arXiv preprint arXiv:2310.07343},
  year={2023}
}

@article{tian2024instructedit,
  title={InstructEdit: Instruction-based Knowledge Editing for Large Language Models},
  author={Tian, Bozhong and Cheng, Siyuan and Liang, Xiaozhuan and Zhang, Ningyu and Hu, Yi and Xue, Kouying and Gou, Yanjie and Chen, Xi and Chen, Huajun},
  journal={arXiv preprint arXiv:2402.16123},
  year={2024}
}

@article{zheng2023can,
  title={Can We Edit Factual Knowledge by In-Context Learning?},
  author={Zheng, Ce and Li, Lei and Dong, Qingxiu and Fan, Yuxuan and Wu, Zhiyong and Xu, Jingjing and Chang, Baobao},
  journal={arXiv preprint arXiv:2305.12740},
  year={2023}
}

@article{elhage2021mathematical,
  title={A mathematical framework for transformer circuits},
  author={Elhage, Nelson and Nanda, Neel and Olsson, Catherine and Henighan, Tom and Joseph, Nicholas and Mann, Ben and Askell, Amanda and Bai, Yuntao and Chen, Anna and Conerly, Tom and others},
  journal={Transformer Circuits Thread},
  volume={1},
  pages={1},
  year={2021}
}

@article{geva2023dissecting,
  title={Dissecting recall of factual associations in auto-regressive language models},
  author={Geva, Mor and Bastings, Jasmijn and Filippova, Katja and Globerson, Amir},
  journal={arXiv preprint arXiv:2304.14767},
  year={2023}
}

@misc{olsson2022incontext,
    title={In-context Learning and Induction Heads},
    author={Catherine Olsson and Nelson Elhage and Neel Nanda and Nicholas Joseph and Nova DasSarma and Tom Henighan and Ben Mann and Amanda Askell and Yuntao Bai and Anna Chen and Tom Conerly and Dawn Drain and Deep Ganguli and Zac Hatfield-Dodds and Danny Hernandez and Scott Johnston and Andy Jones and Jackson Kernion and Liane Lovitt and Kamal Ndousse and Dario Amodei and Tom Brown and Jack Clark and Jared Kaplan and Sam McCandlish and Chris Olah},
    year={2022},
    eprint={2209.11895},
    archivePrefix={arXiv},
    primaryClass={cs.LG}
}

@article{wang2022interpretability,
  title={Interpretability in the wild: a circuit for indirect object identification in gpt-2 small},
  author={Wang, Kevin and Variengien, Alexandre and Conmy, Arthur and Shlegeris, Buck and Steinhardt, Jacob},
  journal={arXiv preprint arXiv:2211.00593},
  year={2022}
}

@article{mcgrath2023hydra,
  title={The hydra effect: Emergent self-repair in language model computations},
  author={McGrath, Thomas and Rahtz, Matthew and Kramar, Janos and Mikulik, Vladimir and Legg, Shane},
  journal={arXiv preprint arXiv:2307.15771},
  year={2023}
}

@inproceedings{geiger2024finding,
  title={Finding alignments between interpretable causal variables and distributed neural representations},
  author={Geiger, Atticus and Wu, Zhengxuan and Potts, Christopher and Icard, Thomas and Goodman, Noah},
  booktitle={Causal Learning and Reasoning},
  pages={160--187},
  year={2024},
  organization={PMLR}
}

@inproceedings{geiger2022inducing,
  title={Inducing causal structure for interpretable neural networks},
  author={Geiger, Atticus and Wu, Zhengxuan and Lu, Hanson and Rozner, Josh and Kreiss, Elisa and Icard, Thomas and Goodman, Noah and Potts, Christopher},
  booktitle={International Conference on Machine Learning},
  pages={7324--7338},
  year={2022},
  organization={PMLR}
}

@article{geiger2021causal,
  title={Causal abstractions of neural networks},
  author={Geiger, Atticus and Lu, Hanson and Icard, Thomas and Potts, Christopher},
  journal={Advances in Neural Information Processing Systems},
  volume={34},
  pages={9574--9586},
  year={2021}
}

@article{wu2024interpretability,
  title={Interpretability at scale: Identifying causal mechanisms in alpaca},
  author={Wu, Zhengxuan and Geiger, Atticus and Icard, Thomas and Potts, Christopher and Goodman, Noah},
  journal={Advances in Neural Information Processing Systems},
  volume={36},
  year={2024}
}

@book{cover1999elements,
  title={Elements of information theory},
  author={Cover, Thomas M},
  year={1999},
  publisher={John Wiley \& Sons}
}

@article{vaswani2017attention,
  title={Attention is all you need},
  author={Vaswani, Ashish and Shazeer, Noam and Parmar, Niki and Uszkoreit, Jakob and Jones, Llion and Gomez, Aidan N and Kaiser, {\L}ukasz and Polosukhin, Illia},
  journal={Advances in neural information processing systems},
  volume={30},
  year={2017}
}

@article{tarzanagh2023margin,
  title={Margin maximization in attention mechanism},
  author={Tarzanagh, Davoud Ataee and Li, Yingcong and Zhang, Xuechen and Oymak, Samet},
  journal={arXiv preprint arXiv:2306.13596},
  year={2023}
}

@article{tarzanagh2023transformers,
  title={Transformers as support vector machines},
  author={Tarzanagh, Davoud Ataee and Li, Yingcong and Thrampoulidis, Christos and Oymak, Samet},
  journal={arXiv preprint arXiv:2308.16898},
  year={2023}
}

@inproceedings{li2024mechanics,
  title={Mechanics of next token prediction with self-attention},
  author={Li, Yingcong and Huang, Yixiao and Ildiz, Muhammed E and Rawat, Ankit Singh and Oymak, Samet},
  booktitle={International Conference on Artificial Intelligence and Statistics},
  pages={685--693},
  year={2024},
  organization={PMLR}
}

@article{sheen2024implicit,
  title={Implicit Regularization of Gradient Flow on One-Layer Softmax Attention},
  author={Sheen, Heejune and Chen, Siyu and Wang, Tianhao and Zhou, Harrison H},
  journal={arXiv preprint arXiv:2403.08699},
  year={2024}
}

@article{tian2023scan,
  title={Scan and snap: Understanding training dynamics and token composition in 1-layer transformer},
  author={Tian, Yuandong and Wang, Yiping and Chen, Beidi and Du, Simon S},
  journal={Advances in Neural Information Processing Systems},
  volume={36},
  pages={71911--71947},
  year={2023}
}

@article{tian2023joma,
  title={Joma: Demystifying multilayer transformers via joint dynamics of mlp and attention},
  author={Tian, Yuandong and Wang, Yiping and Zhang, Zhenyu and Chen, Beidi and Du, Simon},
  journal={arXiv preprint arXiv:2310.00535},
  year={2023}
}

@inproceedings{li2023transformers,
  title={How do transformers learn topic structure: Towards a mechanistic understanding},
  author={Li, Yuchen and Li, Yuanzhi and Risteski, Andrej},
  booktitle={International Conference on Machine Learning},
  pages={19689--19729},
  year={2023},
  organization={PMLR}
}

@misc{allenzhu2023physics32,
    title={Physics of Language Models: Part 3.2, Knowledge Manipulation},
    author={Zeyuan Allen-Zhu and Yuanzhi Li},
    year={2023},
    eprint={2309.14402},
    archivePrefix={arXiv},
    primaryClass={cs.CL}
}

@misc{allenzhu2023physics31,
    title={Physics of Language Models: Part 3.1, Knowledge Storage and Extraction},
    author={Zeyuan Allen-Zhu and Yuanzhi Li},
    year={2023},
    eprint={2309.14316},
    archivePrefix={arXiv},
    primaryClass={cs.CL}
}

@misc{allenzhu2024physics,
      title={Physics of Language Models: Part 3.3, Knowledge Capacity Scaling Laws}, 
      author={Zeyuan Allen-Zhu and Yuanzhi Li},
      year={2024},
      eprint={2404.05405},
      archivePrefix={arXiv},
      primaryClass={cs.CL}
}

@article{garg2022can,
  title={What can transformers learn in-context? a case study of simple function classes},
  author={Garg, Shivam and Tsipras, Dimitris and Liang, Percy S and Valiant, Gregory},
  journal={Advances in Neural Information Processing Systems},
  volume={35},
  pages={30583--30598},
  year={2022}
}

@article{bai2024transformers,
  title={Transformers as statisticians: Provable in-context learning with in-context algorithm selection},
  author={Bai, Yu and Chen, Fan and Wang, Huan and Xiong, Caiming and Mei, Song},
  journal={Advances in neural information processing systems},
  volume={36},
  year={2024}
}

@article{charton2022my,
  title={What is my math transformer doing?--Three results on interpretability and generalization},
  author={Charton, Fran{\c{c}}ois},
  journal={arXiv preprint arXiv:2211.00170},
  year={2022}
}

@article{liu2022towards,
  title={Towards understanding grokking: An effective theory of representation learning},
  author={Liu, Ziming and Kitouni, Ouail and Nolte, Niklas S and Michaud, Eric and Tegmark, Max and Williams, Mike},
  journal={Advances in Neural Information Processing Systems},
  volume={35},
  pages={34651--34663},
  year={2022}
}

@article{nanda2023progress,
  title={Progress measures for grokking via mechanistic interpretability},
  author={Nanda, Neel and Chan, Lawrence and Lieberum, Tom and Smith, Jess and Steinhardt, Jacob},
  journal={arXiv preprint arXiv:2301.05217},
  year={2023}
}

@article{zhang2022unveiling,
  title={Unveiling transformers with lego: a synthetic reasoning task},
  author={Zhang, Yi and Backurs, Arturs and Bubeck, S{\'e}bastien and Eldan, Ronen and Gunasekar, Suriya and Wagner, Tal},
  journal={arXiv preprint arXiv:2206.04301},
  year={2022}
}

@article{zhong2024clock,
  title={The clock and the pizza: Two stories in mechanistic explanation of neural networks},
  author={Zhong, Ziqian and Liu, Ziming and Tegmark, Max and Andreas, Jacob},
  journal={Advances in Neural Information Processing Systems},
  volume={36},
  year={2024}
}

@article{chowdhury2024breaking,
  title={Breaking Down the Defenses: A Comparative Survey of Attacks on Large Language Models},
  author={Chowdhury, Arijit Ghosh and Islam, Md Mofijul and Kumar, Vaibhav and Shezan, Faysal Hossain and Jain, Vinija and Chadha, Aman},
  journal={arXiv preprint arXiv:2403.04786},
  year={2024}
}

@article{xu2023llm,
  title={An llm can fool itself: A prompt-based adversarial attack},
  author={Xu, Xilie and Kong, Keyi and Liu, Ning and Cui, Lizhen and Wang, Di and Zhang, Jingfeng and Kankanhalli, Mohan},
  journal={arXiv preprint arXiv:2310.13345},
  year={2023}
}

@article{zhu2023promptbench,
  title={Promptbench: Towards evaluating the robustness of large language models on adversarial prompts},
  author={Zhu, Kaijie and Wang, Jindong and Zhou, Jiaheng and Wang, Zichen and Chen, Hao and Wang, Yidong and Yang, Linyi and Ye, Wei and Gong, Neil Zhenqiang and Zhang, Yue and others},
  journal={arXiv preprint arXiv:2306.04528},
  year={2023}
}

@article{wang2023robustness,
  title={On the robustness of chatgpt: An adversarial and out-of-distribution perspective},
  author={Wang, Jindong and Hu, Xixu and Hou, Wenxin and Chen, Hao and Zheng, Runkai and Wang, Yidong and Yang, Linyi and Huang, Haojun and Ye, Wei and Geng, Xiubo and others},
  journal={arXiv preprint arXiv:2302.12095},
  year={2023}
}

@article{wang2023decodingtrust,
  title={Decodingtrust: A comprehensive assessment of trustworthiness in gpt models},
  author={Wang, Boxin and Chen, Weixin and Pei, Hengzhi and Xie, Chulin and Kang, Mintong and Zhang, Chenhui and Xu, Chejian and Xiong, Zidi and Dutta, Ritik and Schaeffer, Rylan and others},
  journal={arXiv preprint arXiv:2306.11698},
  year={2023}
}

@inproceedings{si2022so,
  title={Why so toxic? measuring and triggering toxic behavior in open-domain chatbots},
  author={Si, Wai Man and Backes, Michael and Blackburn, Jeremy and De Cristofaro, Emiliano and Stringhini, Gianluca and Zannettou, Savvas and Zhang, Yang},
  booktitle={Proceedings of the 2022 ACM SIGSAC Conference on Computer and Communications Security},
  pages={2659--2673},
  year={2022}
}

@article{rao2023tricking,
  title={Tricking llms into disobedience: Understanding, analyzing, and preventing jailbreaks},
  author={Rao, Abhinav and Vashistha, Sachin and Naik, Atharva and Aditya, Somak and Choudhury, Monojit},
  journal={arXiv preprint arXiv:2305.14965},
  year={2023}
}

@article{shanahan2023role,
  title={Role play with large language models},
  author={Shanahan, Murray and McDonell, Kyle and Reynolds, Laria},
  journal={Nature},
  volume={623},
  number={7987},
  pages={493--498},
  year={2023},
  publisher={Nature Publishing Group UK London}
}

@article{liu2023jailbreaking,
  title={Jailbreaking chatgpt via prompt engineering: An empirical study},
  author={Liu, Yi and Deng, Gelei and Xu, Zhengzi and Li, Yuekang and Zheng, Yaowen and Zhang, Ying and Zhao, Lida and Zhang, Tianwei and Liu, Yang},
  journal={arXiv preprint arXiv:2305.13860},
  year={2023}
}

@article{liu2023prompt,
  title={Prompt Injection attack against LLM-integrated Applications},
  author={Liu, Yi and Deng, Gelei and Li, Yuekang and Wang, Kailong and Zhang, Tianwei and Liu, Yepang and Wang, Haoyu and Zheng, Yan and Liu, Yang},
  journal={arXiv preprint arXiv:2306.05499},
  year={2023}
}

@article{perez2022ignore,
  title={Ignore previous prompt: Attack techniques for language models},
  author={Perez, F{\'a}bio and Ribeiro, Ian},
  journal={arXiv preprint arXiv:2211.09527},
  year={2022}
}

@misc{zou2023universal,
      title={Universal and Transferable Adversarial Attacks on Aligned Language Models}, 
      author={Andy Zou and Zifan Wang and Nicholas Carlini and Milad Nasr and J. Zico Kolter and Matt Fredrikson},
      year={2023},
      eprint={2307.15043},
      archivePrefix={arXiv},
      primaryClass={cs.CL}
}

@misc{apruzzese2022real,
      title={"Real Attackers Don't Compute Gradients": Bridging the Gap Between Adversarial ML Research and Practice}, 
      author={Giovanni Apruzzese and Hyrum S. Anderson and Savino Dambra and David Freeman and Fabio Pierazzi and Kevin A. Roundy},
      year={2022},
      eprint={2212.14315},
      archivePrefix={arXiv},
      primaryClass={cs.CR}
}

@misc{brown2020language,
      title={Language Models are Few-Shot Learners}, 
      author={Tom B. Brown and Benjamin Mann and Nick Ryder and Melanie Subbiah and Jared Kaplan and Prafulla Dhariwal and Arvind Neelakantan and Pranav Shyam and Girish Sastry and Amanda Askell and Sandhini Agarwal and Ariel Herbert-Voss and Gretchen Krueger and Tom Henighan and Rewon Child and Aditya Ramesh and Daniel M. Ziegler and Jeffrey Wu and Clemens Winter and Christopher Hesse and Mark Chen and Eric Sigler and Mateusz Litwin and Scott Gray and Benjamin Chess and Jack Clark and Christopher Berner and Sam McCandlish and Alec Radford and Ilya Sutskever and Dario Amodei},
      year={2020},
      eprint={2005.14165},
      archivePrefix={arXiv},
      primaryClass={cs.CL}
}

@misc{hu2021lora,
      title={LoRA: Low-Rank Adaptation of Large Language Models}, 
      author={Edward J. Hu and Yelong Shen and Phillip Wallis and Zeyuan Allen-Zhu and Yuanzhi Li and Shean Wang and Lu Wang and Weizhu Chen},
      year={2021},
      eprint={2106.09685},
      archivePrefix={arXiv},
      primaryClass={cs.CL}
}

@inproceedings{shi2023large,
  title={Large language models can be easily distracted by irrelevant context},
  author={Shi, Freda and Chen, Xinyun and Misra, Kanishka and Scales, Nathan and Dohan, David and Chi, Ed H and Sch{\"a}rli, Nathanael and Zhou, Denny},
  booktitle={International Conference on Machine Learning},
  pages={31210--31227},
  year={2023},
  organization={PMLR}
}

@article{yoran2023making,
  title={Making retrieval-augmented language models robust to irrelevant context},
  author={Yoran, Ori and Wolfson, Tomer and Ram, Ori and Berant, Jonathan},
  journal={arXiv preprint arXiv:2310.01558},
  year={2023}
}

@article{petroni2020context,
  title={How context affects language models' factual predictions},
  author={Petroni, Fabio and Lewis, Patrick and Piktus, Aleksandra and Rockt{\"a}schel, Tim and Wu, Yuxiang and Miller, Alexander H and Riedel, Sebastian},
  journal={arXiv preprint arXiv:2005.04611},
  year={2020}
}

@article{creswell2022selection,
  title={Selection-inference: Exploiting large language models for interpretable logical reasoning},
  author={Creswell, Antonia and Shanahan, Murray and Higgins, Irina},
  journal={arXiv preprint arXiv:2205.09712},
  year={2022}
}

@article{pandia2021sorting,
  title={Sorting through the noise: Testing robustness of information processing in pre-trained language models},
  author={Pandia, Lalchand and Ettinger, Allyson},
  journal={arXiv preprint arXiv:2109.12393},
  year={2021}
}

@article{firth1957synopsis,
  title={A synopsis of linguistic theory, 1930-1955},
  author={Firth, John},
  journal={Studies in linguistic analysis},
  pages={10--32},
  year={1957}
}

@article{hu2024sparse,
  title={On sparse modern hopfield model},
  author={Hu, Jerry Yao-Chieh and Yang, Donglin and Wu, Dennis and Xu, Chenwei and Chen, Bo-Yu and Liu, Han},
  journal={Advances in Neural Information Processing Systems},
  volume={36},
  year={2024}
}

@article{wu2023stanhop,
  title={STanhop: Sparse tandem hopfield model for memory-enhanced time series prediction},
  author={Wu, Dennis and Hu, Jerry Yao-Chieh and Li, Weijian and Chen, Bo-Yu and Liu, Han},
  journal={arXiv preprint arXiv:2312.17346},
  year={2023}
}

@article{hu2024nonparametric,
  title={Nonparametric modern hopfield models},
  author={Hu, Jerry Yao-Chieh and Chen, Bo-Yu and Wu, Dennis and Ruan, Feng and Liu, Han},
  journal={arXiv preprint arXiv:2404.03900},
  year={2024}
}

@article{hu2024computational,
  title={On computational limits of modern hopfield models: A fine-grained complexity analysis},
  author={Hu, Jerry Yao-Chieh and Lin, Thomas and Song, Zhao and Liu, Han},
  journal={arXiv preprint arXiv:2402.04520},
  year={2024}
}

@article{wu2024uniform,
  title={Uniform memory retrieval with larger capacity for modern hopfield models},
  author={Wu, Dennis and Hu, Jerry Yao-Chieh and Hsiao, Teng-Yun and Liu, Han},
  journal={arXiv preprint arXiv:2404.03827},
  year={2024}
}

@article{hu2024outlier,
  title={Outlier-efficient hopfield layers for large transformer-based models},
  author={Hu, Jerry Yao-Chieh and Chang, Pei-Hsuan and Luo, Robin and Chen, Hong-Yu and Li, Weijian and Wang, Wei-Po and Liu, Han},
  journal={arXiv preprint arXiv:2404.03828},
  year={2024}
}

@article{yu2024enhancing,
  title={Enhancing Jailbreak Attack Against Large Language Models through Silent Tokens},
  author={Yu, Jiahao and Luo, Haozheng and Yao-Chieh, Jerry and Guo, Wenbo and Liu, Han and Xing, Xinyu},
  journal={arXiv preprint arXiv:2405.20653},
  year={2024}
}

@article{luo2024decoupled,
  title={Decoupled Alignment for Robust Plug-and-Play Adaptation},
  author={Luo, Haozheng and Yu, Jiahao and Zhang, Wenxin and Li, Jialong and Hu, Jerry Yao-Chieh and Xin, Xingyu and Liu, Han},
  journal={arXiv preprint arXiv:2406.01514},
  year={2024}
}

@article{achiam2023gpt,
  title={Gpt-4 technical report},
  author={Achiam, Josh and Adler, Steven and Agarwal, Sandhini and Ahmad, Lama and Akkaya, Ilge and Aleman, Florencia Leoni and Almeida, Diogo and Altenschmidt, Janko and Altman, Sam and Anadkat, Shyamal and others},
  journal={arXiv preprint arXiv:2303.08774},
  year={2023}
}

\newpage
\appendix

\section{Additional Theoretical Results and Proofs}
\label{appendix:proofs}
\counterwithin{figure}{section}
\renewcommand{\thefigure}{\thesection.\arabic{figure}}

\subsection{Proofs for \cref{sec:hypo-am}}

\cref{thm:idealmodel-informal} can be stated more formally as follows:

\begin{theorem}
\label{thm:idealmodel}
Suppose the data generating process follows \cref{sec:setup} where $m \geq 3$, $\omega = 1$, and $\tNe(t) = \vocab \setminus \{ t \}$.
Assume there exists a single layer transformer given by \cref{eqn:transformer} such that a) $W_K = 0$ and $W_Q=0$, b) Each row of $W_E$ is orthogonal to each other and normalized, and c) $W_V$ is given by
\begin{equation*}
    W_V = \sum_{i \in [\vocab]} W_E(i) (\sum_{j \in \tNeOne(i)}W_E(j)^T).
\end{equation*}
Then if $L > \max \{\frac{100m^2 \log(3/\epsilon)}{(\exp(-\frac{1}{\beta}) - \exp(-\frac{2}{\beta}))^2}, \frac{80m^2|\tNe(y)|}{(\exp(-\frac{1}{\beta}) - \exp(-\frac{2}{\beta}))^2} \}$ for any $y$, then
\begin{equation*}
     R_{\mathcal{D}^L}(f^L) \leq \epsilon,
\end{equation*}
where $0 < \epsilon < 1$.
\end{theorem}

\begin{proof}
First of all, the error is defined to be:
\begin{equation*}
\begin{split}
    R_{\mathcal{D}^L}(f^L) &= \mathbb{P}_{(x, y) \sim \mathcal{D}^L}[\argmax f^L(x) \neq y] \\
    &= \mathbb{P}_y \mathbb{P}_{x | y}[\argmax f^L(x) \neq y]
\end{split}
\end{equation*}
Let's focus on the conditional probability  $\mathbb{P}_{x | y}[\argmax f^L(x) \neq y]$.

By construction, the single layer transformer model has uniform attention. Therefore, 
\begin{equation*}
    h(x) = \sum_{i \in \tNe(y)} \alpha_i W_E(i)
\end{equation*}
where $\alpha_i = \frac{1}{L}\sum_{k=1}^L \mathbf{1}\{ t_k = i \}$ which is the number of occurrence of token $i$ in the sequence. 

By the latent concept association model, we know that 
\begin{equation*}
    p(i | y) = \frac{\exp(- D_H(i, y) /\beta)}{Z}
\end{equation*}
where $Z = \sum_{i \in \tNe(y)} \exp(- D_H(i, y) /\beta)$.

Thus, the logit for token $y$ is
\begin{equation*}
    f^L_y(x) = \sum_{i \in \tNeOne(y)} \alpha_i
\end{equation*}

And the logit for any other token $\tilde{y}$ is 
\begin{equation*}
    f^L_{\tilde{y}}(x) = \sum_{i \in \tNeOne(\tilde{y})} \alpha_i
\end{equation*}

For the prediction to be correct, we need
\begin{equation*}
    \max_{\tilde{y}} f^L_y(x) - f^L_{\tilde{y}}(x) > 0
\end{equation*}

By Lemma 3 of \cite{Devroye}, we know that for all $\Delta \in (0,1)$, if $\frac{|\tNe(y)|}{L} \leq \frac{\Delta^2}{20}$, we have
\begin{equation*}
    \mathbb{P}\big(\max_{i \in \tNe(y)} |\alpha_i - p(i | y)| > \Delta \big) \leq \mathbb{P}\big(\sum_{i \in \tNe(y)} |\alpha_i - p(i | y)| > \Delta \big) \leq 3 \exp(-L\Delta^2/25)
\end{equation*}

Therefore, if $L \geq \max \{\frac{25 \log(3/\epsilon)}{\Delta^2}, \frac{20|\tNe(y)|}{\Delta^2} \}$, then with probability at least $1 - \epsilon$, we have,
\begin{equation*}
    \max_{i \in \tNe(y)} |\alpha_i - p(i | y)| \leq \Delta
\end{equation*}

\begin{equation*}
\begin{split}
      f^L_y(x) - f^L_{\tilde{y}}(x) &= \sum_{i \in \tNeOne(y)} \alpha_i - \sum_{j \in \tNeOne(\tilde{y})} \alpha_j \\
      & = \sum_{i \in \tNeOne(y)} \alpha_i -  \sum_{i \in \tNeOne(y)} p(i|y) + \sum_{i \in \tNeOne(y)} p(i|y)\\
      &\qquad - \sum_{j \in \tNeOne(\tilde{y})} p(j|y) + \sum_{j \in \tNeOne(\tilde{y})} p(j|y) - \sum_{j \in \tNeOne(\tilde{y})} \alpha_j \\
      & \geq \sum_{i \in \tNeOne(y)} p(i|y) - \sum_{j \in \tNeOne(\tilde{y})} p(j|y) - 2m\Delta \\
      & \geq \exp(-\frac{1}{\beta}) - \exp(-\frac{2}{\beta}) - 2m\Delta 
\end{split}
\end{equation*}
Note that because of \cref{lemma:no-twin}, there's no neighboring set that is the superset of another.

Therefore as long as $\Delta < \frac{\exp(-\frac{1}{\beta}) - \exp(-\frac{2}{\beta})}{2m}$, 
\begin{equation*}
    f^L_y(x) - f^L_{\tilde{y}}(x) > 0
\end{equation*}
for any $\tilde{y}$.

Finally, if $L > \max \{\frac{100m^2 \log(3/\epsilon)}{(\exp(-\frac{1}{\beta}) - \exp(-\frac{2}{\beta}))^2}, \frac{80m^2|\tNe(y)|}{(\exp(-\frac{1}{\beta}) - \exp(-\frac{2}{\beta}))^2} \}$ for any $y$, then
\begin{equation*}
    \mathbb{P}_{x | y}[\argmax f^L(x) \neq y] \leq \epsilon
\end{equation*}

And 
\begin{equation*}
\begin{split}
    R_{\mathcal{D}^L}(f^L) &= \mathbb{P}_{(x, y) \sim \mathcal{D}^L}[\argmax f^L(x) \neq y] \\
    &= \mathbb{P}_y \mathbb{P}_{x | y}[\argmax f^L(x) \neq y] \leq \epsilon
\end{split}
\end{equation*}

\end{proof}

\subsection{Proofs for \cref{sec:value}}

\remap*
\begin{proof}
Following the proof for \cref{thm:idealmodel}, let's focus on the conditional probability:
\begin{equation*}
    \mathbb{P}_{x | y}[\argmax f^L(x) \neq y]
\end{equation*}

By construction, we have
\begin{equation*}
    h(x) = \sum_{i \in \tNeOne(y)} \alpha_i W_E(i)
\end{equation*}

where $\alpha_i = \frac{1}{L}\sum_{k=1}^L \mathbf{1}\{ t_k = i \}$ which is the number of occurrence of token $i$ in the sequence. 

Let's consider the first case where $W_V$ is constructed as in \cref{eqn:wv-constrution}. Then we know that for some other token $\tilde{y} \neq y$, 
\begin{equation*}
    f^L_y(x) - f^L_{\tilde{y}}(x)  = \sum_{i \in \tNeOne(y)} \alpha_i - \sum_{i \in \tNeOne(\tilde{y})} \alpha_i =  1 - \sum_{i \in \tNeOne(\tilde{y})} \alpha_i 
\end{equation*}
By \cref{lemma:no-twin}, we have that for any token $\tilde{y} \neq y$, 
\begin{equation*}
    f^L_y(x) - f^L_{\tilde{y}}(x)   > 0
\end{equation*}
Therefore, the error rate is always $0$.

Now let's consider the second case where $W_V$ is the identity matrix. Let $j$ be a token in the set $\tNeOne(y)$. Then there is a non-zero probability that context $x$ contains only $j$. In that case,
\begin{equation*}
    h(x) = W_E(j)
\end{equation*}
However, we know that by the assumption on the embedding matrix,
\begin{equation*}
    f^L_y(x) - f^L_{j}(x) = (W_E(y) - W_E(j))^T h(x) = - \| W_E(j)\|^2 < 0
\end{equation*}
This implies that there's non zero probability that $y$ is misclassified. Therefore, when $W_V$ is the identity matrix, the error rate is strictly larger than $0$.
\end{proof}

\hamming*
\begin{proof}
Because for any pair of $(x,y)$, the estimated conditional probability matches the true conditional probability. In particular, let's consider two target tokens $y_1$, $y_2$ and context $x = (t_i, ..., t_i)$ for some token $t_i$ such that $p(x|y_1)>0$ and $p(x|y_2)>0$, then
\begin{equation*}
\begin{split}
    \frac{p(y_1|x)}{p(y_2|x)} = \frac{p(x|y_1)p(y_1)}{p(x|y_2)p(y_2)} & = \frac{p(x|y_1)}{p(x|y_2)} =  \frac{\hat{p}(x|y_1)}{\hat{p}(x|y_2)} = \exp((W_E(y_1) - W_E(y_2))^T h(x))
\end{split}
\end{equation*}
The second equality is because $p(y)$ is the uniform distribution. 
By our construction, 
\begin{equation*}
\begin{split}
     \frac{p(x|y_1)}{p(x|y_2)} =   \frac{p(t_i|y_1)^L}{p(t_i|y_2)^L} =  
    \exp((W_E(y_2) - W_E(y_1))^T h(x)) = \exp((W_E(y_1) - W_E(y_2))^T W_E(t_i))
\end{split}
\end{equation*}
By the data generating process, we have that
\begin{equation*}
    \frac{L}{\beta} (D_H(t_i, y_2) - D_H(t_i, y_1)) = (W_E(y_1) - W_E(y_2))^T W_E(t_i)
\end{equation*}
Let $t_i = y_3$ such that $y_3 \neq y_1$, $y_3 \neq y_2$, then
\begin{equation*}
 \frac{L}{\beta} D_H(y_3, y_1) - W_E(y_1)^T W_E(y_3) = \frac{L}{\beta} D_H(y_3, y_2) - W_E(y_2)^T W_E(y_3)
\end{equation*}
For simplicity, let's define
\begin{equation*}
    \Psi(y_1, y_2) = \frac{L}{\beta} D_H(y_1, y_2) - W_E(y_1)^T W_E(y_2)
\end{equation*}
Therefore,
\begin{equation*}
    \Psi(y_3, y_1) = \Psi(y_3, y_2)
\end{equation*}
Now consider five distinct labels: $y_1, y_2, y_3, y_4, y_5$. We have, 
\begin{equation*}
    \Psi(y_3, y_1) = \Psi(y_3, y_2) = \Psi(y_4, y_2) = \Psi(y_4, y_5)
\end{equation*}

In other words, $\Psi(y_3, y_1) = \Psi(y_4, y_5)$ for arbitrarily chosen distinct labels $y_1, y_3, y_4, y_5$. Therefore, $\Psi(t, t')$ is a constant for $t \neq t'$.

For any two tokens $t \neq t'$, 
\begin{equation*}
    \frac{L}{\beta} D_H(t, t') - W_E(t)^T W_E(t') = C
\end{equation*}
Thus,
\begin{equation*}
    W_E(t)^T W_E(t') = - \frac{L}{\beta} D_H(t, t') + C
\end{equation*}

\end{proof}

\subsection{Proofs for \cref{sec:geometry}}

\cref{thm:modifiedmodel-informal} can be formalized as the following theorem.

\begin{theorem}
\label{thm:modifiedmodel}
Following the same setup as in \cref{thm:idealmodel}, but embeddings follow \cref{eqn:emb-linear} then if $b > 0$, $\Delta_1 > 0$, $0 < \Delta < \frac{\exp(-\frac{1}{\beta}) - \exp(-\frac{2}{\beta})}{2m}$, $L \geq \max \{\frac{25 \log(3/\epsilon)}{\Delta^2}, \frac{20|\tNe(y)|}{\Delta^2} \}$ for any $y$, and 
\begin{equation*}
    0 < a < \frac{2 \exp(\frac{1}{\beta})}{(|\vocab|-2) m^2}
\end{equation*}
and 
\begin{equation*}
    b_0 > \max  \{\frac{a(m-2)m + \Delta_1}{\exp(-\frac{1}{\beta}) - \exp(-\frac{2}{\beta}) - 2m \Delta} + b,  \frac{(b-a) \Delta_1 - \frac{|\vocab|-2}{2} a b m^2 \exp(-\frac{1}{\beta}) + \frac{|\vocab|-2}{2} a^2 (m-2)m^2}{1 - \frac{|\vocab|-2}{2}am^2\exp(-\frac{1}{\beta}) }\}
\end{equation*}
we have
\begin{equation*}
     R_{\mathcal{D}^L}(f^L) \leq \epsilon
\end{equation*}
where $0 < \epsilon < 1$.
\end{theorem}

\begin{proof}
Following the proof of \cref{thm:idealmodel}, let's also focus on  the conditional probability  
\begin{equation*}
    \mathbb{P}_{x | y}[\argmax f^L(x) \neq y]
\end{equation*}

By construction, the single layer transformer model has uniform attention. Therefore, 
\begin{equation*}
    h(x) = \sum_{i \in \tNe(y)} \alpha_i W_E(i)
\end{equation*}
where $\alpha_i = \frac{1}{L}\sum_{k=1}^L \mathbf{1}\{ t_k = i \}$ which is the number of occurrence of token $i$ in the sequence. For simplicity, let's define $\alpha_y = 0$ such that
\begin{equation*}
    h(x) = \sum_{i \in [\vocab]} \alpha_i W_E(i)
\end{equation*}

Similarly, we also have that if $L \geq \max \{\frac{25 \log(3/\epsilon)}{\Delta^2}, \frac{20|\tNe(y)|}{\Delta^2} \}$, then with probability at least $1 - \epsilon$, we have,
\begin{equation*}
    \max_{i \in [\vocab]} |\alpha_i - p(i | y)| \leq \Delta
\end{equation*}

Also define the following:
\begin{equation*}
\begin{split}
    \phi_k(x) &=  \sum_{j \in \tNeOne(k)} W_E(j)^T \big( \sum_{i \in [\vocab]} \alpha_i W_E(i) \big) \\
    v_k(y) &= W_E(y)^T W_E(k)
\end{split}
\end{equation*}

Thus, the logit for token $y$ is 
\begin{equation*}
    f^L_y(x) = \sum_{k=0}^{|\vocab|-1} v_k(y) \phi_k(x)
\end{equation*}

Let's investigate $\phi_k(x)$. By \cref{lemma:sum}, 

\begin{equation*}
\begin{split}
    \phi_k(x) &= \sum_{i \in [V]} \alpha_i (\sum_{j \in \tNeOne(k)} W_E(j)^T W_E(i)) \\
    &= (b_0 - b)\sum_{j \in \tNeOne(k)}  \alpha_j + \sum_{i \in [V]} \alpha_i (-a (m-2) D_H(k,i) + (b-a)m)
\end{split}
\end{equation*}

Thus, for any $k_1, k_2 \in [\vocab]$, 
\begin{equation*}
\begin{split}
    \phi_{k_1}(x) - \phi_{k_2}(x) &= (b_0 - b) (\sum_{j_1 \in \tNeOne(k_1)}  \alpha_{j_1} - \sum_{j_2 \in \tNeOne(k_2)}  \alpha_{j_2}) \\
    &+ \sum_{i \in [V]} \alpha_i a (m-2) (D_H(k_2, i) - D_H(k_1, i))
 \end{split}
\end{equation*}

Because $-m \leq D_H(k_2, i) - D_H(k_1, i) \leq m$, we have
\begin{equation*}
\begin{split}
     (b_0 - b) (\sum_{j_1 \in \tNeOne(k_1)}  \alpha_{j_1} - \sum_{j_2 \in \tNeOne(k_2)}  \alpha_{j_2}) &- a (m-2)m \\
     \leq \phi_{k_1}(x) &- \phi_{k_2}(x) \leq  \\
     (b_0 - b) (\sum_{j_1 \in \tNeOne(k_1)}  &\alpha_{j_1} - \sum_{j_2 \in \tNeOne(k_2)}  \alpha_{j_2}) + a (m-2)m
\end{split}
\end{equation*}

For prediction to be correct, we need
\begin{equation*}
    \max_{\tilde{y}} f^L_y(x) - f^L_{\tilde{y}}(x) > 0
\end{equation*}

This also means that
\begin{equation*}
    \max_{\tilde{y}} \sum_{k=0}^{|\vocab|-1} \big(v_k(y) - v_k(\tilde{y}) \big) \phi_k(x) > 0
\end{equation*}

One can show that for any $k$, if $\tok^{-1}(\tilde{k}) = \tok^{-1}(y) \otimes \tok^{-1}(\tilde{y}) \otimes \tok^{-1}(k)$ where $\otimes$ means bitwise $\mathrm{XOR}$, then
\begin{equation}
\label{eqn:v-eq}
    v_k(y) - v_k(\tilde{y}) = v_{\tilde{k}}(\tilde{y}) - v_{\tilde{k}}(y) 
\end{equation}

First of all, if $k = y$, then $\tilde{k} = \tilde{y}$, which means
\begin{equation*}
    v_k(y) - v_k(\tilde{y}) = v_{\tilde{k}}(\tilde{y}) - v_{\tilde{k}}(y) = b_0 + a D_H(y, \tilde{y}) - b
\end{equation*}
If $k \neq y, \tilde{y}$, then \cref{eqn:v-eq} implies that
\begin{equation*}
    D_H(k, y) - D_H(k, \tilde{y}) = D_H(\tilde{k}, \tilde{y}) - D_H(\tilde{k}, y)
\end{equation*}

We know that $D_H(k, y)$ is the number of $1$s in $\tok^{-1}(k) \otimes \tok^{-1}(y)$ and,

\begin{equation*}
    \tok^{-1}(\tilde{k}) \otimes \tok^{-1}(y) = \tok^{-1}(y) \otimes \tok^{-1}(\tilde{y}) \otimes \tok^{-1}(k) \otimes \tok^{-1}(y) = \tok^{-1}(\tilde{y}) \otimes \tok^{-1}(k)
\end{equation*}

Similarly,
\begin{equation*}
    \tok^{-1}(\tilde{k}) \otimes \tok^{-1}(\tilde{y}) = \tok^{-1}(y) \otimes \tok^{-1}(k)
\end{equation*}

Therefore, \cref{eqn:v-eq} holds and we can rewrite $f^L_y(x) - f^L_{\tilde{y}}(x)$ as
\begin{equation*}
\begin{split}
    f^L_y(x) - f^L_{\tilde{y}}(x) &=  \sum_{k=0}^{|\vocab|-1} \big(v_k(y) - v_k(\tilde{y}) \big) \phi_k(x) \\
    &= (b_0 - b + a D_H(y, \tilde{y})) (\phi_y(x) - \phi_{\tilde{y}}(x)) \\
    & + \sum_{k \neq y, \tilde{y}, D_H(k, y) \geq D_H(k, \tilde{y})} a (D_H(k, y) - D_H(k, \tilde{y})) (\phi_k(x) - \phi_{\tilde{k}}(x)) 
\end{split}
\end{equation*}
We already know that $b_0 > b > 0$ and $a > 0$, thus, $b_0 - b + a D_H(y, \tilde{y}) > 0$ for any pair $y, \tilde{y}$. 

We also want $\phi_y(x) - \phi_{\tilde{y}}(x)$ to be positive. Note that 
\begin{equation*}
    \phi_y(x) - \phi_{\tilde{y}}(x) \geq (b_0 - b) (\exp(-\frac{1}{\beta}) - \exp(-\frac{2}{\beta}) - 2m\Delta) - a (m-2)m
\end{equation*}
We need $\Delta < \frac{\exp(-\frac{1}{\beta}) - \exp(-\frac{2}{\beta})}{2m}$ and for some positive $\Delta_1 > 0$, $b_0$ needs to be large enough such that
\begin{equation*}
     \phi_y(x) - \phi_{\tilde{y}}(x) > \Delta_1
\end{equation*}
which implies that
\begin{equation}
    b_0 > \frac{a(m-2)m + \Delta_1}{\exp(-\frac{1}{\beta}) - \exp(-\frac{2}{\beta}) - 2m \Delta} + b
\end{equation}

On the other hand, for $k \neq y, \tilde{y}$, we have 
\begin{equation*}
\begin{split}
    \phi_{k}(x) - \phi_{\tilde{k}}(x) &\geq (b_0 - b) (\sum_{j_1 \in \tNeOne(k)}  \alpha_{j_1} - \sum_{j_2 \in \tNeOne(\tilde{k})}  \alpha_{j_2}) - a (m-2)m \\
    & \geq (b_0 - b) ( - (m-1) \exp(-\frac{1}{\beta}) - \exp(- \frac{2}{\beta}) - 2m \Delta) - a (m-2)m \\
    & \geq (b_0 - b) ( - (m-1) \exp(-\frac{1}{\beta}) - \exp(- \frac{2}{\beta}) + \exp(- \frac{2}{\beta}) - \exp(- \frac{1}{\beta})) - a (m-2)m \\
    & \geq -(b_0 - b) m\exp(-\frac{1}{\beta}) - a (m-2)m \\
\end{split}
\end{equation*}

Then, we have 
\begin{equation*}
\begin{split}
    &f^L_y(x) - f^L_{\tilde{y}}(x) \geq (b_0 - b + a) \Delta_1 - \frac{|\vocab|-2}{2}  \bigg((b_0 - b)a m^2 \exp(-\frac{1}{\beta}) + a^2(m-2)m^2 \bigg) \\
    & \geq \bigg(1 - \frac{|\vocab|-2}{2}am^2\exp(-\frac{1}{\beta}) \bigg)b_0 - (b-a) \Delta_1 + \frac{|\vocab|-2}{2} a b m^2 \exp(-\frac{1}{\beta}) - \frac{|\vocab|-2}{2} a^2 (m-2)m^2
\end{split}
\end{equation*}
The lower bound is independent of $\tilde{y}$, therefore, we need it to be positive to ensure the prediction is correct. To achieve this, we want
\begin{equation*}
    1 - \frac{|\vocab|-2}{2}am^2\exp(-\frac{1}{\beta}) > 0
\end{equation*}
which implies that 
\begin{equation}
    a < \frac{2 \exp(\frac{1}{\beta})}{(|\vocab|-2) m^2}
\end{equation}
And finally we need 
\begin{equation}
    b_0  > \frac{(b-a) \Delta_1 - \frac{|\vocab|-2}{2} a b m^2 \exp(-\frac{1}{\beta}) + \frac{|\vocab|-2}{2} a^2 (m-2)m^2}{1 - \frac{|\vocab|-2}{2}am^2\exp(-\frac{1}{\beta}) }
\end{equation}
To summarize, if $b > 0$, $\Delta_1 > 0$, $0 < \Delta < \frac{\exp(-\frac{1}{\beta}) - \exp(-\frac{2}{\beta})}{2m}$, $L \geq \max \{\frac{25 \log(3/\epsilon)}{\Delta^2}, \frac{20|\tNe(y)|}{\Delta^2} \}$ for any $y$, and 
\begin{equation*}
    0 < a < \frac{2 \exp(\frac{1}{\beta})}{(|\vocab|-2) m^2}
\end{equation*}
and 
\begin{equation*}
    b_0 > \max  \{\frac{a(m-2)m + \Delta_1}{\exp(-\frac{1}{\beta}) - \exp(-\frac{2}{\beta}) - 2m \Delta} + b,  \frac{(b-a) \Delta_1 - \frac{|\vocab|-2}{2} a b m^2 \exp(-\frac{1}{\beta}) + \frac{|\vocab|-2}{2} a^2 (m-2)m^2}{1 - \frac{|\vocab|-2}{2}am^2\exp(-\frac{1}{\beta}) }\}
\end{equation*}
we have 
\begin{equation*}
\begin{split}
    R_{\mathcal{D}^L}(f^L)\leq \epsilon
\end{split}
\end{equation*}
where $0 < \epsilon < 1$.


\end{proof}

\embrank*
\begin{proof}
By \cref{eqn:emb-linear}, we have that
\begin{equation*}
    \langle W_E(i), W_E(j) \rangle = - a D_H(i, j) + b
\end{equation*}
Therefore,
\begin{equation*}
    (W_E)^T W_E = -a D_H + b \mathbf{1} \mathbf{1}^T
\end{equation*}
Let's first look at $D_H$ which has rank at most $m+1$.
To see this, let's consider a set of $m + 1$ tokens: $\{e_0, e_1, ..., e_m \} \subseteq \vocab$ where $e_k = 2^k$. Here $e_0$ is associated with the latent vector of all zeroes and the latent vector associated with $e_k$ has only the $k$-th latent variable being $1$.

On the other hand, for any token $i$, we have that,
\begin{equation*}
    i = \sum_{k: \tok^{-1}(i)_k = 1} e_k
\end{equation*}

In fact, 
\begin{equation*}
    D_H(i) = \sum_{k: \tok^{-1}(i)_k = 1} \bigg( D_H(e_k) - D_H(e_0) \bigg) + D_H(e_0)
\end{equation*}
where $D_H(i)$ is the $i$-th row of $D_H$,
and for each entry $j$ of $D_H(i)$, we have that
\begin{equation*}
    D_H(i, j) = \sum_{k: \tok^{-1}(i)_k = 1} \bigg( D_H(e_k, j) - D_H(e_0, j) \bigg) + D_H(e_0, j)
\end{equation*}
This is because
\begin{equation*}
\begin{split}
    D_H(e_k, j) - D_H(e_0, j) = \begin{cases}
        + 1 \;\;\; \text{if } \tok^{-1}(j)_k = 0 \\
        - 1 \;\;\; \text{if } \tok^{-1}(j)_k = 1
    \end{cases}
\end{split}
\end{equation*}
Thus, we can rewrite $D_H(i, j)$ as
\begin{equation*}
\begin{split}
    D_H(i, j) &= \sum_{k: \tok^{-1}(i)_k = 1} \bigg(\mathbf{1}[\tok^{-1}(i)_k=1, \tok^{-1}(j)_k=0] - \mathbf{1}[\tok^{-1}(i)_k=1, \tok^{-1}(j)_k=1)] \bigg)  + D_H(e_0, j) \\
    &= \sum_{k=1}^m \bigg(\mathbf{1}[\tok^{-1}(i)_k=1, \tok^{-1}(j)_k=0] - \mathbf{1}[\tok^{-1}(i)_k=1, \tok^{-1}(j)_k=1)] \bigg)  \\
    & \quad \quad \quad + \sum_{k=1}^m \bigg(\mathbf{1}[\tok^{-1}(i)_k=0, \tok^{-1}(j)_k=1] + \mathbf{1}[\tok^{-1}(i)_k=1, \tok^{-1}(j)_k=1)] \bigg)  \\
    &= \sum_{k=1}^{m} \mathbf{1}[\tok^{-1}(i)_k=1, \tok^{-1}(j)_k=0] + \mathbf{1}[\tok^{-1}(i)_k=0, \tok^{-1}(j)_k=1] \\
    & = D_H(i, j)
\end{split}
\end{equation*}

Therefore, every row of $D_H$ can be written as a linear combination of $\{D_H(e_0), D_H(e_1), ..., D_H(e_m) \}$. In other words, $D_H$ has rank at most $m+1$.

Therefore,
\begin{equation*}
    \mathrm{rank}((W_E)^T W_E) =  \mathrm{rank}(W_E) \leq m+2.
\end{equation*}
\end{proof}

\begin{lemma}
\label{lemma:sum}
Let $z^{(0)}$ and $z^{(1)}$ be two binary vectors of size $m$ where $m \geq 2$. Then,
\begin{equation*}
    \sum_{z: D_H(z^{(0)}, z) = 1} D_H(z, z^{(1)}) = (m-2) D_H(z^{(0)}, z^{(1)}) + m
\end{equation*}
\end{lemma}

\begin{proof}
For $z$ such that $D_H(z, z^{(0)}) = 1$, we know that there are two cases. Either $z$ differs with $z^{(0)}$ on a entry but agrees with $z^{(1)}$ on that entry or $z$ differs with both $z^{(0)}$ and $z^{(1)}$.

For the first case, we know that there are $D_H(z^{(0)}, z^{(1)})$ such entries. In this case, $D_H(z, z^{(1)}) = D_H(z^{(0)}, z^{(1)})-1$. For the second case, $D_H(z, z^{(1)}) = D_H(z^{(0)}, z^{(1)})+1$.

Therefore,
\begin{equation*}
\begin{split}
    \sum_{z: D_H(z, z^{(0)}) = 1} &D_H(z, z^{(1)})  \\
    &= D_H(z^{(0)}, z^{(1)})(D_H(z^{(0)}, z^{(1)}) -1) + (m - D_H(z^{(0)}, z^{(1)})) (D_H(z^{(0)}, z^{(1)}) + 1) \\
    &= (m-2) D_H(z^{(0)}, z^{(1)}) + m
\end{split}
\end{equation*}
\end{proof}

\begin{lemma}
\label{lemma:no-twin}
If $m \geq 3$ and $\tNe(t) = \vocab \setminus \{ t \}$, then $\tNeOne(t) \not\subseteq \tNeOne(t')$ for any $t, t' \in [\vocab]$.
\end{lemma}
\begin{proof}
For any token $t$, $\tNeOne(t)$ contains any token $t'$ such that $D_H(t,t') = 1$ by the conditions. Then given a set $\tNeOne(t)$, one can uniquely determine token $t$. This is because for the set of latent vectors associated with $\tNeOne(t)$, at each index, there could only be one possible change. 
\end{proof}

\subsection{Proofs for \cref{sec:atten}}
\attengrad*
\begin{proof}
Recall that, 
\begin{equation*}
\begin{split}
    f^L(x) &= \bigg[ {W_E}^T W_V \textrm{attn} (W_E \chi(x)) \bigg]_{:L} \\
    &= {W_E}^T W_V \sum_{l=1}^{L} \frac{\exp(u_{t_l, t_L})}{Z} W_E(t_l)
\end{split}
\end{equation*}
where $Z$ is a normalizing constant. 

Define $\hat{p}_{t_l} = \frac{\exp(u_{t_l, t_L})}{Z}$. Then we have
\begin{equation*}
\begin{split}
    f^L(x) = {W_E}^T W_V \sum_{l=1}^{L} \hat{p}_{t_l} W_E(t_l)
\end{split}
\end{equation*}

Note that if $t_l = t$ then,
\begin{equation*}
    \frac{\partial \hat{p}_{t_l}}{\partial u_{t, t_L}} = \hat{p}_{t_l} (1 - \hat{p}_{t_l})
\end{equation*}
Otherwise, 
\begin{equation*}
    \frac{\partial \hat{p}_{t_l}}{\partial u_{t, t_L}} = - \hat{p}_{t_l}  \hat{p}_{t}
\end{equation*}

By the chain rule, we know that
\begin{equation*}
\begin{split}
     \nabla_{u_{t, t_{L}}} \ell(f^L) = \nabla \ell(f^L)^T (W_E)^T W^{V} (\sum_{l=1}^L \mathbf{1}[t_l = t] \hat{p}_{t_l}  W_E(t)  - \sum_{l=1}^L \hat{p}_{t_l}\hat{p}_t W_E(t_l) )
\end{split}
\end{equation*}
Therefore,
\begin{equation*}
    \nabla_{u_{t, t_{L}}} \ell(f^L) = \nabla \ell(f^L)^T (W_E)^T W^{V} (  \alpha_t \hat{p}_{t}  W_E(t)  - \hat{p}_t \sum_{l=1}^L \hat{p}_{t_l} W_E(t_l) )
\end{equation*}
where $\alpha_t = \sum_{l=1}^L \mathbf{1}[t_l = t]$.
\end{proof}

\section{Additional experiments -- context hijacking}
\label{sec: additional_hijacking_expts}

In this section, we show the results of additional context hijacking experiments on the \cfd dataset \citep{meng2022locating}.

\paragraph{Reverse context hijacking}

In \cref{fig:sub_false_accs}, we saw the effects of hijacking by adding in ``Do not think of \{target\_false\}.''  to each context.
Now, we measure the effect of the reverse: What if we prepend ``Do not think of \{target\_true\}.'' ? 

Based on the study in this paper on how associative memory works in LLMs, we should expect the efficacy score to decrease. 
Indeed, this is what happens, as we see in \cref{fig:sub_true_accs}.

\begin{figure}[!h]
    \centering\includegraphics[width=.7\textwidth]{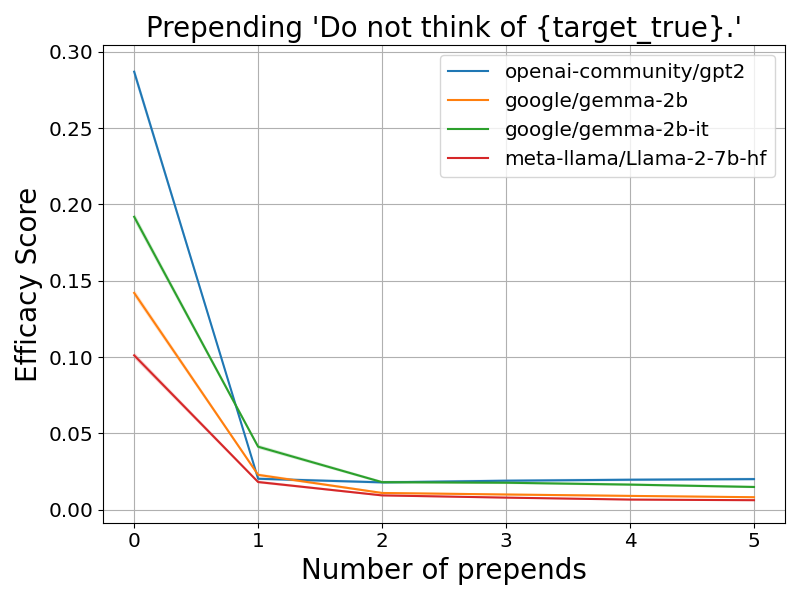}
    \caption{Prepending `Do not think of \{target\_true\}.' can increase the chance of LLMs to output correct tokens.
    This figure shows efficacy score versus the number of prepends for various LLMs on the \cfd dataset with the reverse context hijacking scheme.}
    \label{fig:sub_true_accs}
\end{figure}

\paragraph{Hijacking based on relation IDs}

We first give an example of each of the 4 relation IDs we hijack in \cref{tab: example_rel_id}.

\begin{table}[!h]
\caption{Examples of contexts in Relation IDs from \cfd}
\label{tab: example_rel_id}
\begin{center}
\begin{footnotesize}
\begin{tabular}{cccc}
\toprule
\sc{Relation ID $\relation$} & \sc{Context $\prompt$} & \sc{True target $\targetTrue$} & \sc{False target $\targetFalse$}\\
\midrule
P190 & Kharkiv is a twin city of  & Warsaw & Athens\\
P103 & The native language of Anatole France is & French & English\\
P641 & Hank Aaron professionally plays the sport & baseball & basketball\\
P131 & Kalamazoo County can be found in & Michigan & Indiana\\
\bottomrule
\end{tabular}
\end{footnotesize}
\end{center}
\end{table}

\begin{table}[!h]
\caption{Examples of hijack and reverse hijack formats based on Relation IDs}
\label{tab: hijack_format}
\begin{center}
\begin{footnotesize}
\begin{tabular}{ccc}
\toprule
\sc{Relation ID $\relation$} & \sc{Context Hijack sentence} & \sc{Reverse Context Hijack sentence}\\
\midrule
P190 & The twin city of \{subject\} is not \{target\_false\} & The twin city of \{subject\} is \{target\_true\}\\
P103 & \{subject\} cannot speak \{target\_false\} & \{subject\} can speak \{target\_true\} \\
P641 & \{subject\} does not play \{target\_false\} & \{subject\} plays \{target\_true\}\\
P131 & \{subject\} is not located in \{target\_false\} & \{subject\} is located in \{target\_true\}\\
\bottomrule
\end{tabular}
\end{footnotesize}
\end{center}
\end{table}

Similar to \cref{fig:P190_accs}, we repeat the hijacking experiments where we prepend factual sentences generated from the relation ID. We use the format illustrated in \cref{tab: hijack_format} for the prepended sentences. We experiment with 3 other relation IDs and we see similar trends for all the LLMs
in \cref{fig:sentence_false_P103}, \ref{fig:sentence_false_P131}, and \ref{fig:sentence_false_P641}. That is, the efficacy score rises for the first prepend and as we increase the number of prepends, the trend of ES rising continues. Therefore, this confirms our intuition that LLMs can be hijacked by contexts without changing the factual meaning.

Similar to \cref{fig:sub_true_accs}, we experiment with reverse context hijacking where we give the answers based on relation IDs, as shown in \cref{tab: hijack_format}. We again experiment with the same 4 relation IDs and the results are in Figure~\ref{fig:sentence_true_P103} - \ref{fig:sentence_true_P641}.
We see that the efficacy score decreases when we prepend the answer sentence, thereby verifying the observations of this study.

\begin{figure}
    \centering
    \begin{subfigure}{0.45\textwidth}
        \centering
        \includegraphics[width=\textwidth]{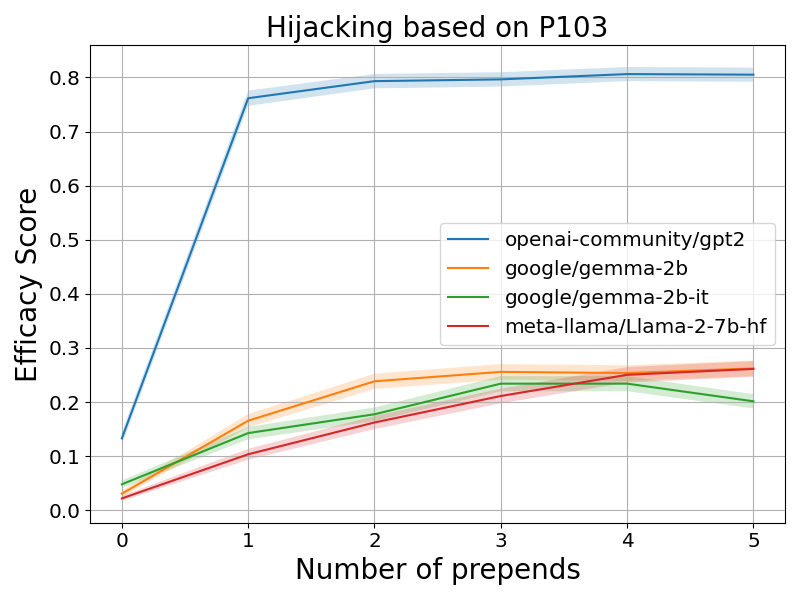}
        \caption{Relation P$103$}
        \label{fig:sentence_false_P103}
    \end{subfigure}
    \begin{subfigure}{0.45\textwidth}
        \centering
        \includegraphics[width=\textwidth]{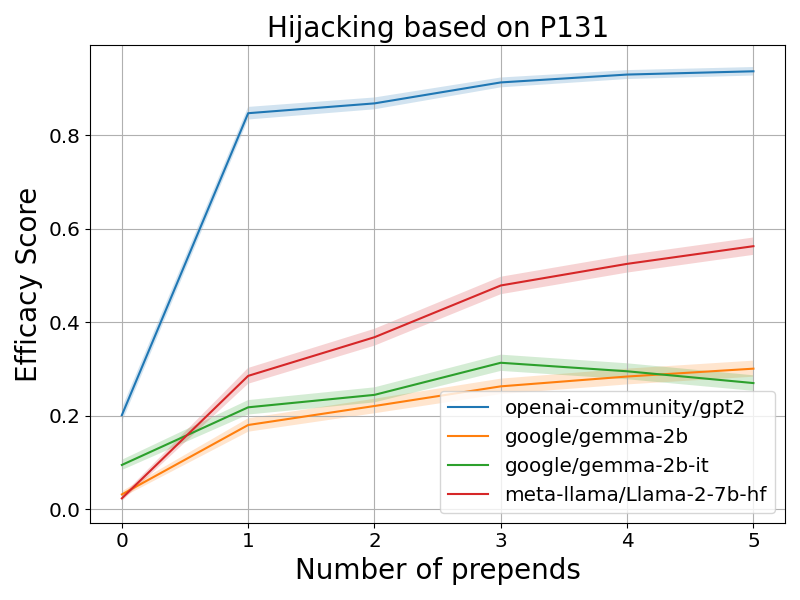}
        \caption{Relation P$132$}
        \label{fig:sentence_false_P131}
    \end{subfigure}
    \begin{subfigure}{0.45\textwidth}
        \centering
        \includegraphics[width=\textwidth]{pics/hijack/sentence_false_P190.png}
        \caption{Relation P$190$}
        \label{fig:sentence_false_P190}
    \end{subfigure}
    \begin{subfigure}{0.45\textwidth}
        \centering
        \includegraphics[width=\textwidth]{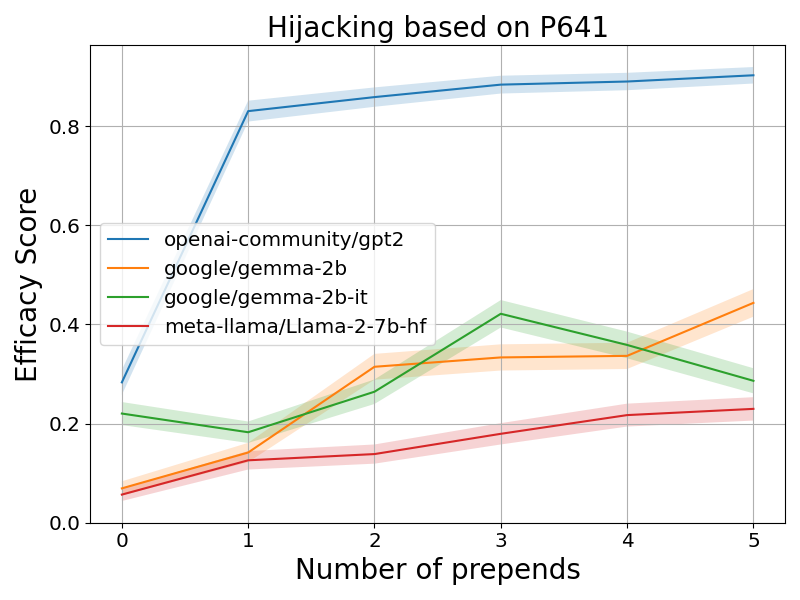}
        \caption{Relation P$641$}
        \label{fig:sentence_false_P641}
    \end{subfigure}
    \caption{Context hijacking based on relation IDs can result in LLMs output incorrect tokens. This figure shows efficacy score versus the number of prepends for various LLMs on the \cfd dataset with hijacking scheme presented in \cref{tab: hijack_format}.} 
    \label{fig:sentence_false_relation}
\end{figure}

\begin{figure}[!h]
    \centering
    \begin{subfigure}{0.45\textwidth}
        \centering
        \includegraphics[width=\textwidth]{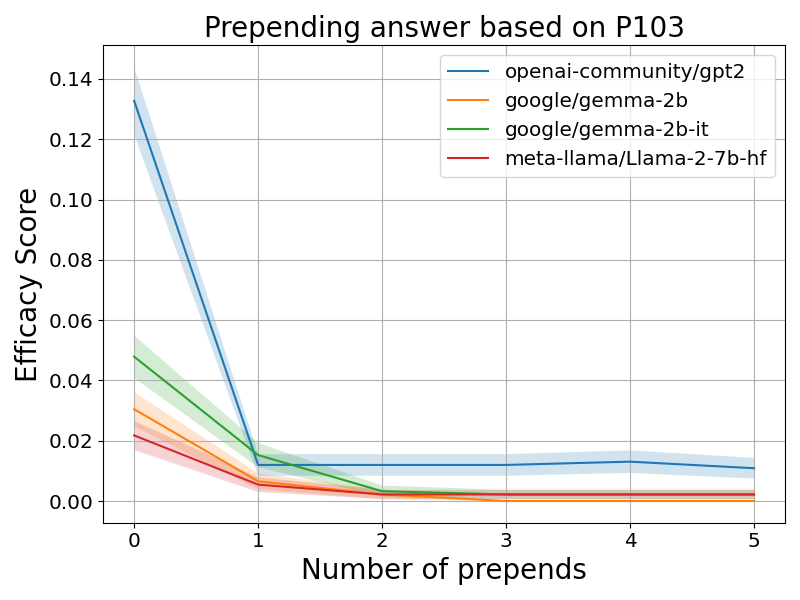}
        \caption{Relation P$103$}
        \label{fig:sentence_true_P103}
    \end{subfigure}
    \begin{subfigure}{0.45\textwidth}
        \centering
        \includegraphics[width=\textwidth]{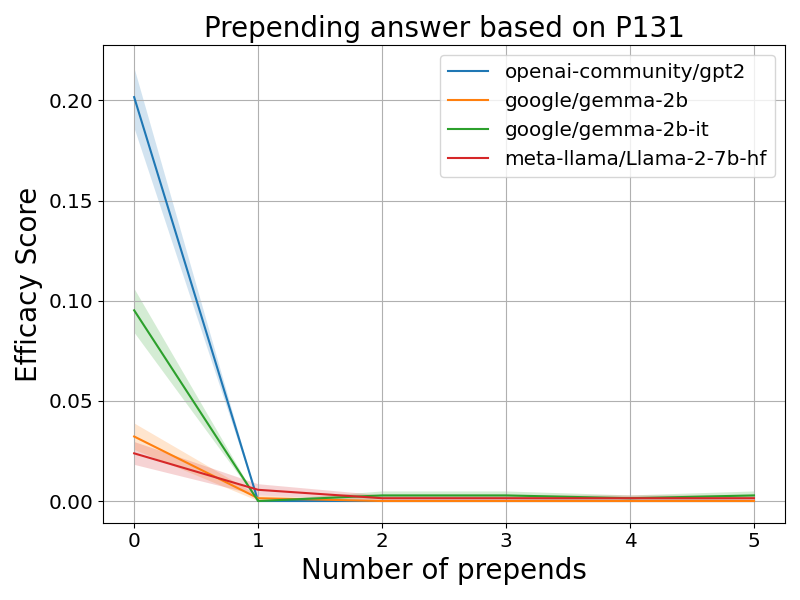}
        \caption{Relation P$132$}
        \label{fig:sentence_true_P131}
    \end{subfigure}
    \begin{subfigure}{0.45\textwidth}
        \centering
        \includegraphics[width=\textwidth]{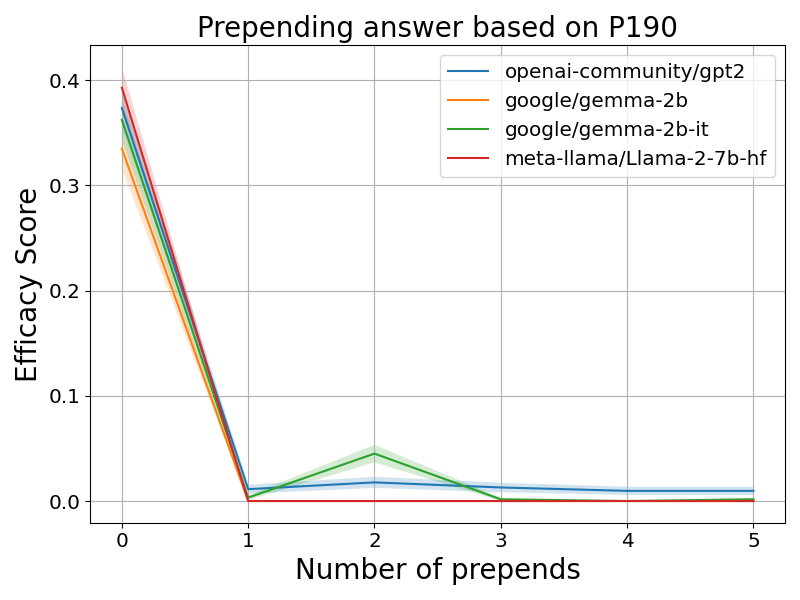}
        \caption{Relation P$190$}
        \label{fig:sentence_true_P190}
    \end{subfigure}
    \begin{subfigure}{0.45\textwidth}
        \centering
        \includegraphics[width=\textwidth]{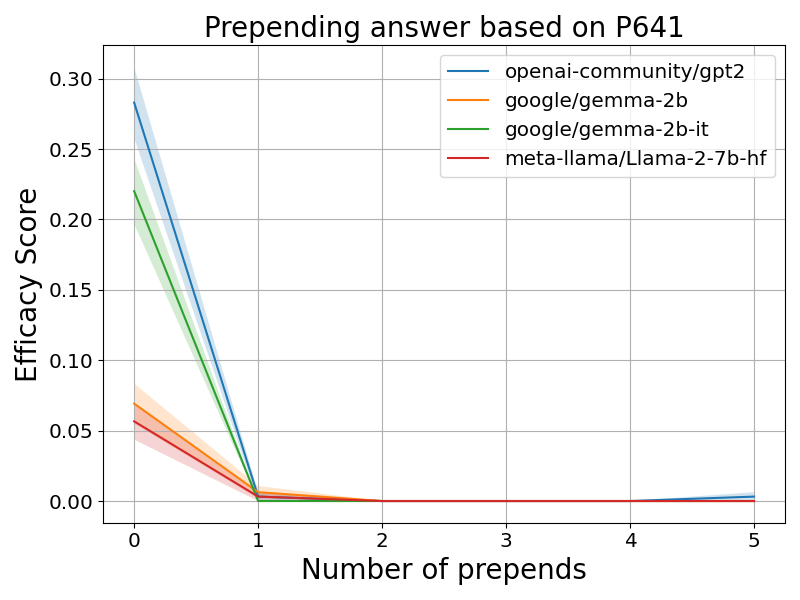}
        \caption{Relation P$641$}
        \label{fig:sentence_true_P641}
    \end{subfigure}
    \caption{Reverse context hijacking based on relation IDs can result in LLMs to be more likely to be correct. This figure shows efficacy score versus the number of prepends for various LLMs on the \cfd dataset with the reverse hijacking scheme presented in \cref{tab: hijack_format}.} 
    \label{fig:sentence_true_relation}
\end{figure}

\paragraph{Hijacking without exact target words}
So far, the experiments use prompts that either contain true or false target words. It turns out, the inclusion of exact target words are not necessary. To see this, we experiment a variant of the generic hijacking and reverse hijacking experiments. But instead of saying ``Do not think of \{target\_false\}'' or ``Do not think of \{target\_true\}''. We replace target words with words that are semantically close. Specifically, for relation P$1412$, we replace words representing language (e.g., ``French'') with their associated country name (e.g., ``France''). As shown in \cref{fig:1412}, context hijacking and reverse hijacing still work in this case.

\begin{figure}
    \centering
    \begin{subfigure}[b]{0.45\textwidth}
        \centering
        \includegraphics[width=\textwidth]{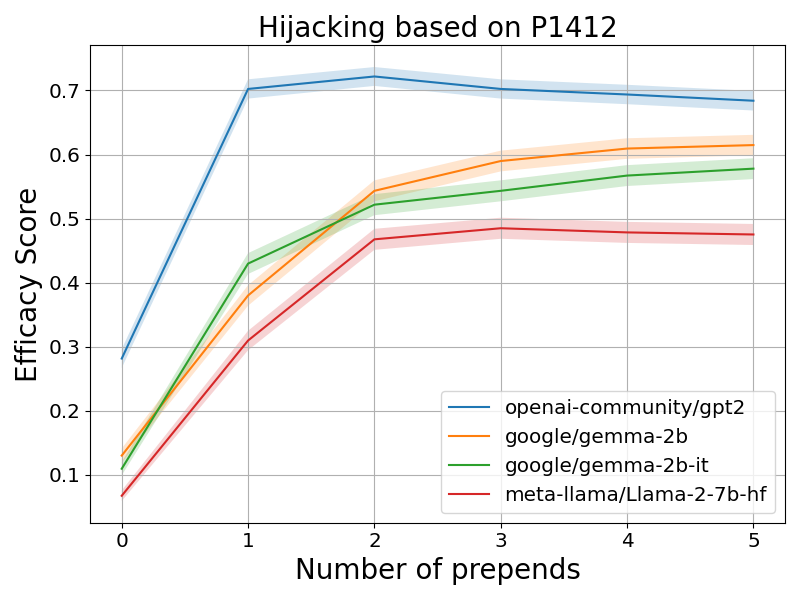}
        \caption{Hijacking P$1412$}
        \label{fig:1412}
    \end{subfigure}
    \begin{subfigure}[b]{0.45\textwidth}
        \centering
        \includegraphics[width=\textwidth]{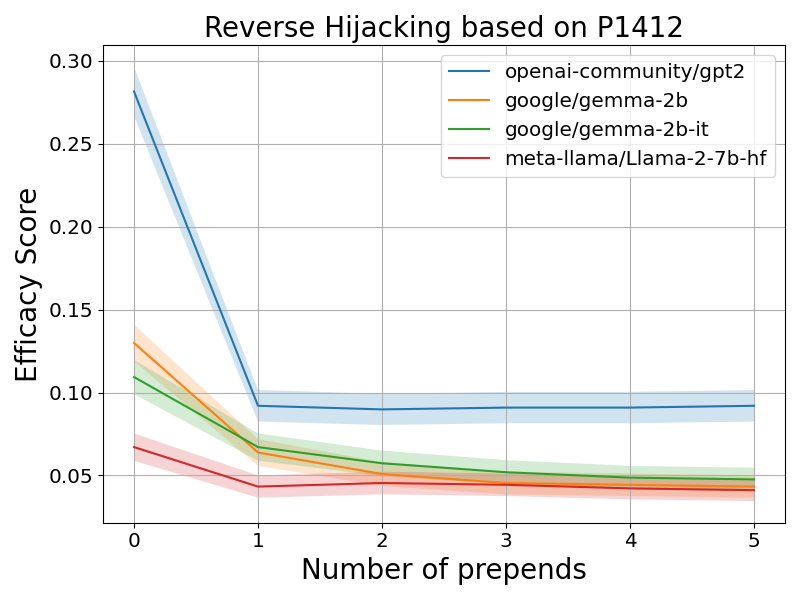}
        \caption{Reverse hijacking P$1412$}
        \label{fig:1412}
    \end{subfigure}
    \caption{Hijacking and reverse hijacking experiments on relation P$1412$ show that context hijacking does not require exact target word to appear in the context.
    This figure shows efficacy score versus the number of prepends for various LLMs on the \cfd dataset.} 
    \label{fig:1412}
\end{figure}

\newpage
\section{Additional experiments and figures -- latent concept association}
\label{append:exp-latent}
In this appendix section, we present additional experimental details and results from the synthetic experiments on latent concept association.

\paragraph{Experimental setup}
Synthetic data are generated following the model in \cref{sec:setup}. Unless otherwise stated, the default setup has $\omega = 0.5$, $\beta=1$ and $\tNe(i) = \vocab \setminus \{ i \}$ and $L=256$. The default hidden dimension of the one-layer transformer is also set to be $256$. The model is optimized using AdamW \cite{loshchilov2017decoupled} where the learning rate is chosen from $\{0.01, 0.001 \}$. The evaluation dataset is drawn from the same distribution as the training dataset and consists of $1024$ $(x,y)$ pairs. Although theoretical results in \cref{sec:theory} may freeze certain parts of the network for simplicity, in this section, unless otherwise specified, all layers of the transformers are trained jointly. Also, in this section, we typically report accuracy which is $1-\textrm{error}$.

\subsection{On the value matrix $W_V$}
\label{append:expwv}
In this section, we provide additional figures of \cref{sec:expwv}. Specifically, \cref{fig:identity} shows that fixing the value matrix to be the identity will negatively impact accuracy. \cref{fig:dim_wv_replace} indicates that replacing trained value matrices with constructed ones can preserve accuracy to some extent. \cref{fig:dim_wv_replace_angle} suggests that trained value matrices and constructed ones share similar low-rank approximations. For the last two sets of experiments, we  consider randomly constructed value matrix, where the outer product pairs are chosen randomly, defined formally as follows:

\begin{equation*}
    W_V = \sum_{i \in [\vocab]} W_E(i) (\sum_{\{j\} \sim \textrm{Unif}([\vocab])^{|\tNeOne(i)|} }W_E(j)^T) 
\end{equation*}

\begin{figure}
    \centering
    \begin{subfigure}[b]{0.45\textwidth}
        \centering
        \includegraphics[width=\textwidth]{pics/identity_training_wv/Wv_bar_l64.png}
        \caption{$L=64$}
        \label{fig:identity_64}
    \end{subfigure}
    \begin{subfigure}[b]{0.45\textwidth}
        \centering
        \includegraphics[width=\textwidth]{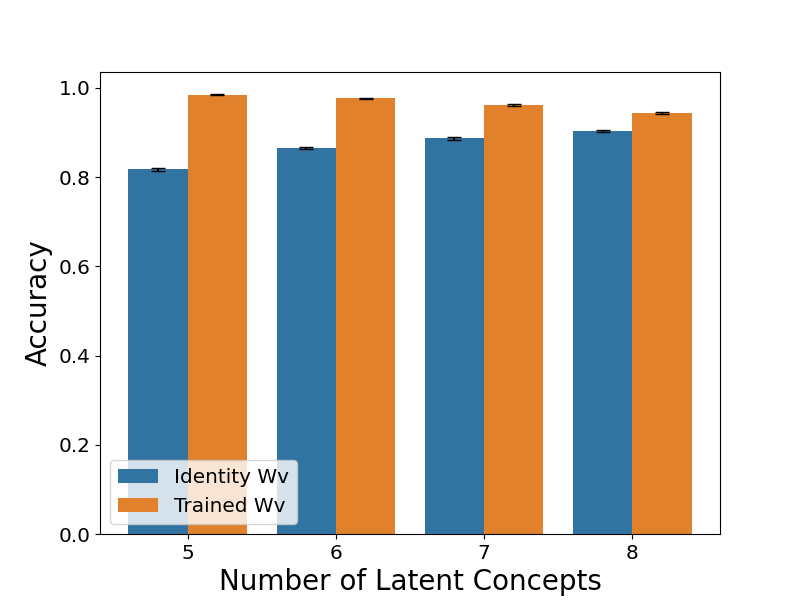}
        \caption{$L=128$}
        \label{fig:identity_128}
    \end{subfigure}
    \caption{Fixing the value matrix $W_V$ as the identity matrix results in lower accuracy compared to training $W_V$, especially for smaller context length $L$. The figure reports accuracy for both fixed and trained $W_V$ settings, with standard errors calculated over $10$ runs.}
    \label{fig:identity}
\end{figure}

\begin{figure}
    \centering
    \begin{subfigure}[b]{0.45\textwidth}
        \centering
        \includegraphics[width=\textwidth]{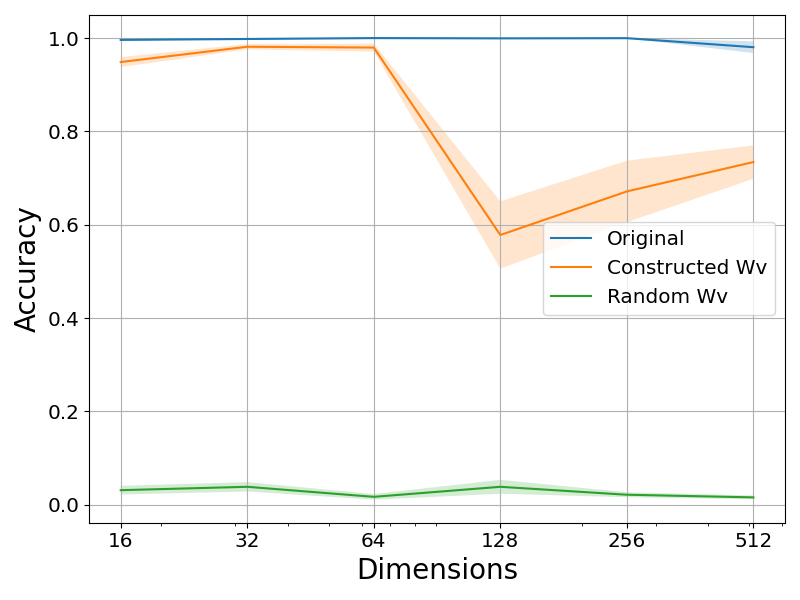}
        \caption{$m=5$}
        \label{fig:dim_wv_replace_acc_5}
    \end{subfigure}
    \begin{subfigure}[b]{0.45\textwidth}
        \centering
        \includegraphics[width=\textwidth]{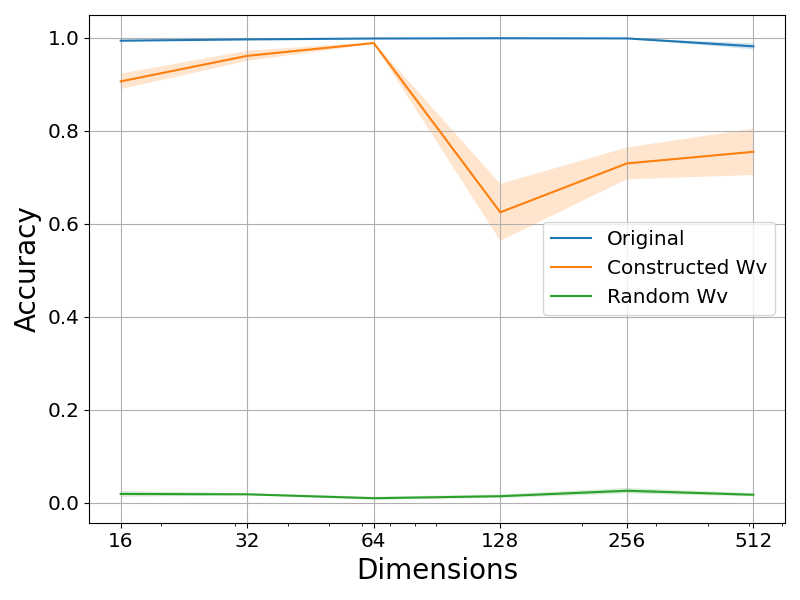}
        \caption{$m=6$}
        \label{fig:dim_wv_replace_acc_6}
    \end{subfigure}

    \begin{subfigure}[b]{0.45\textwidth}
        \centering
        \includegraphics[width=\textwidth]{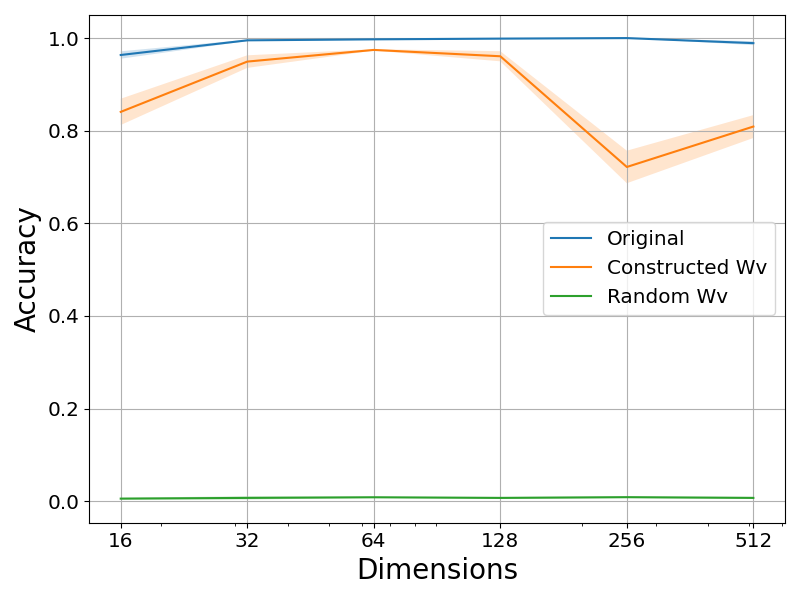}
        \caption{$m=7$}
        \label{fig:dim_wv_replace_acc_7}
    \end{subfigure}
    \begin{subfigure}[b]{0.45\textwidth}
        \centering
        \includegraphics[width=\textwidth]{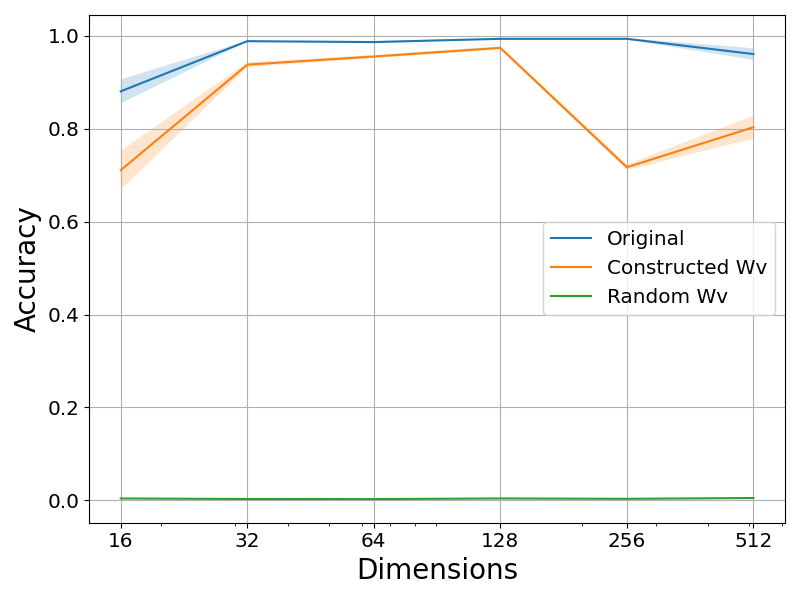}
        \caption{$m=8$}
        \label{fig:dim_wv_replace_acc_8}
    \end{subfigure}
    \caption{When the value matrix is replaced with the constructed one in trained transformers, the accuracy does not significantly decrease compared to replacing the value matrix with randomly constructed ones.
    The graph reports accuracy under different embedding dimensions and standard errors are over $5$ runs.} 
    \label{fig:dim_wv_replace}
\end{figure}

\begin{figure}
    \centering
    \begin{subfigure}[b]{0.45\textwidth}
        \centering
        \includegraphics[width=\textwidth]{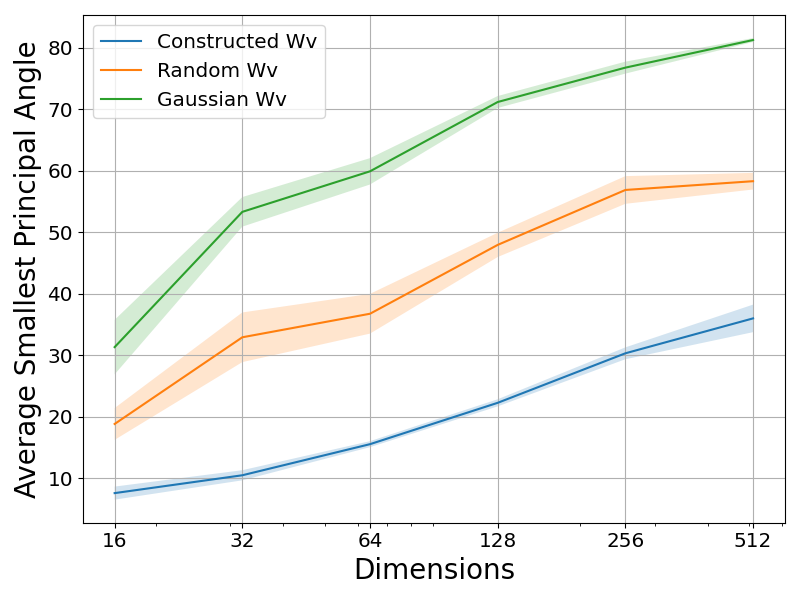}
        \caption{$m=5$}
        \label{fig:dim_wv_replace_angle_5}
    \end{subfigure}
    \begin{subfigure}[b]{0.45\textwidth}
        \centering
        \includegraphics[width=\textwidth]{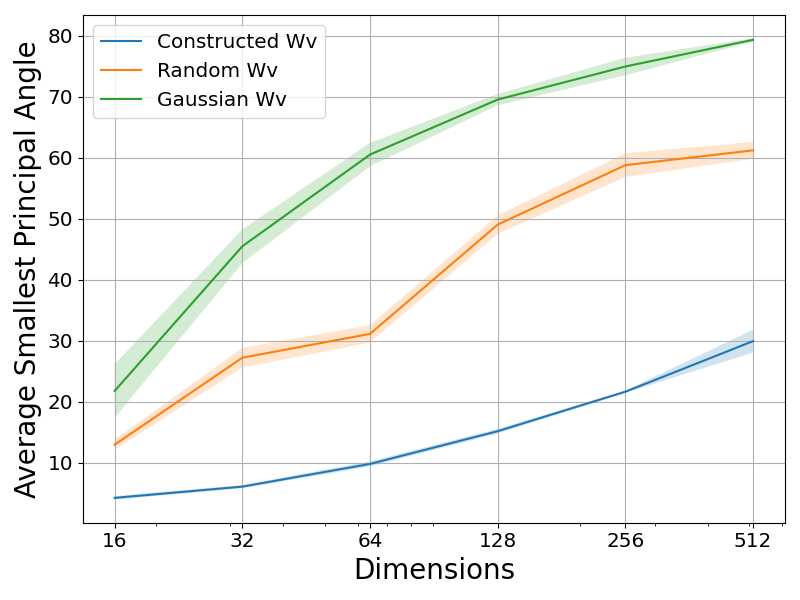}
        \caption{$m=6$}
        \label{fig:dim_wv_replace_angle_6}
    \end{subfigure}

    \begin{subfigure}[b]{0.45\textwidth}
        \centering
        \includegraphics[width=\textwidth]{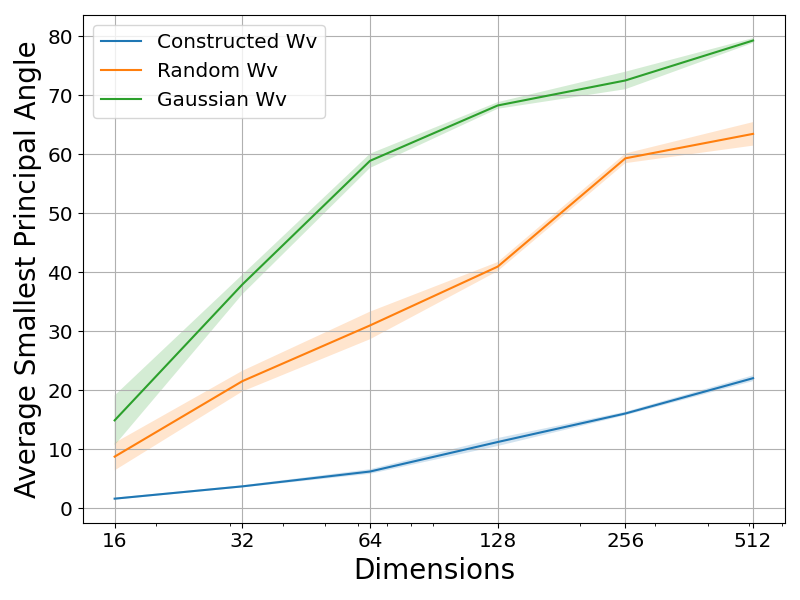}
        \caption{$m=7$}
        \label{fig:dim_wv_replace_angle_7}
    \end{subfigure}
    \begin{subfigure}[b]{0.45\textwidth}
        \centering
        \includegraphics[width=\textwidth]{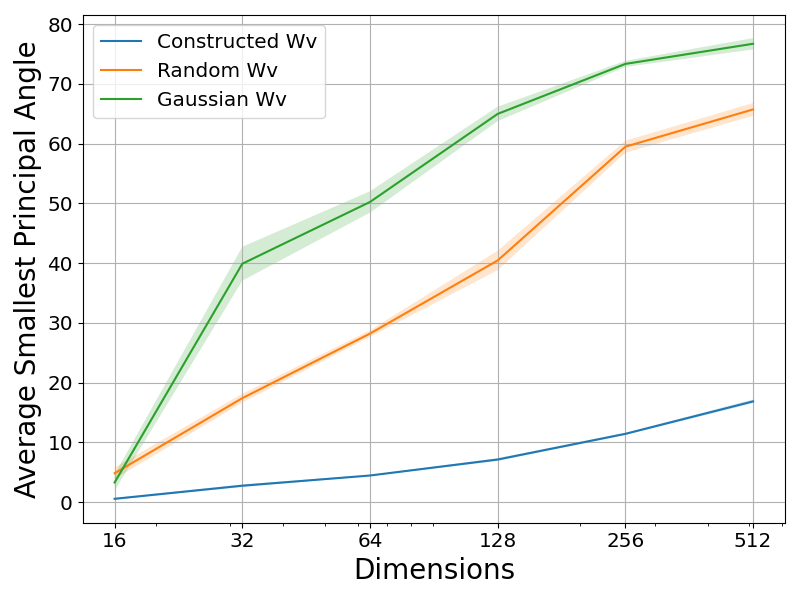}
        \caption{$m=8$}
        \label{fig:dim_wv_replace_angle_8}
    \end{subfigure}
    \caption{The constructed value matrix $W_V$ has similar low rank approximation with the trained value matrix. The figure displays average smallest principal angles between low-rank approximations of trained value matrices and those of constructed, randomly constructed, and Gaussian-initialized value matrices.
    Standard errors are over $5$ runs.} 
    \label{fig:dim_wv_replace_angle}
\end{figure}

\subsection{On the embeddings}
\label{appen:emb}
This section provides additional figures from \cref{sec:expemb}. \cref{fig:dim_acc} shows that in the underparameterized regime, embedding training is required. \cref{fig:hamming} indicates that the embedding structure in the underparameterized regime roughly follows \cref{eqn:emb-linear}. Finally \cref{fig:no_wv_hamming} shows that, when the value matrix is fixed to the identity, the relationship between inner product of embeddings and their corresponding Hamming distance is mostly linear.

\begin{figure}
    \centering
    \begin{subfigure}[b]{0.45\textwidth}
        \centering
        \includegraphics[width=\textwidth]{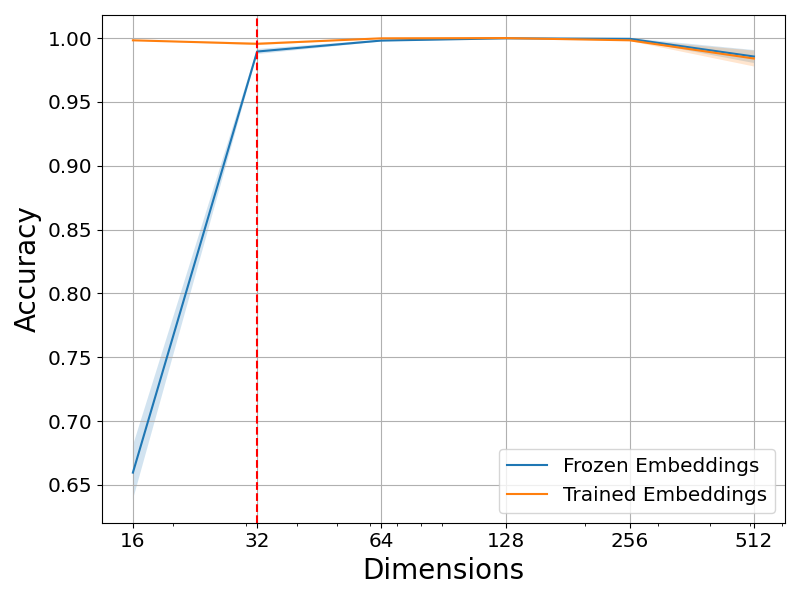}
        \caption{$m=5$}
        \label{fig:dim_acc_5}
    \end{subfigure}
    \begin{subfigure}[b]{0.45\textwidth}
        \centering
        \includegraphics[width=\textwidth]{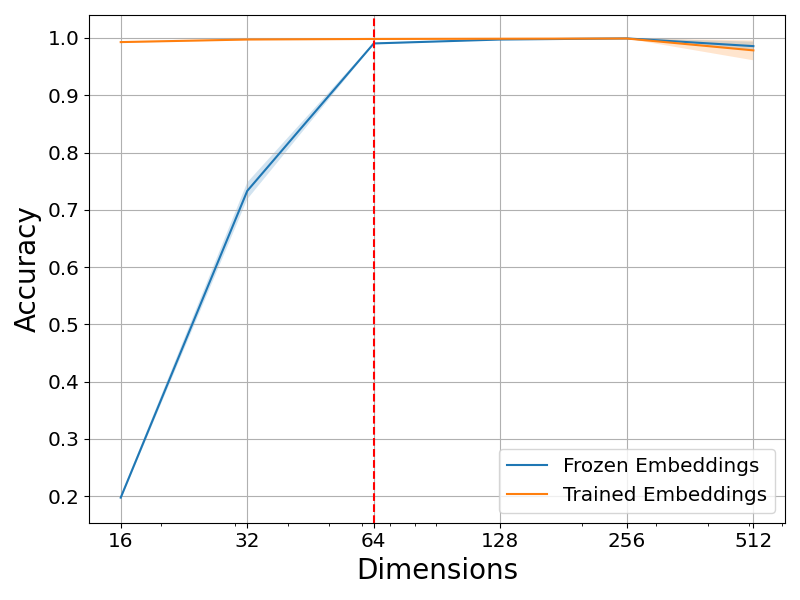}
        \caption{$m=6$}
        \label{fig:dim_acc_6}
    \end{subfigure}

    \begin{subfigure}[b]{0.45\textwidth}
        \centering
        \includegraphics[width=\textwidth]{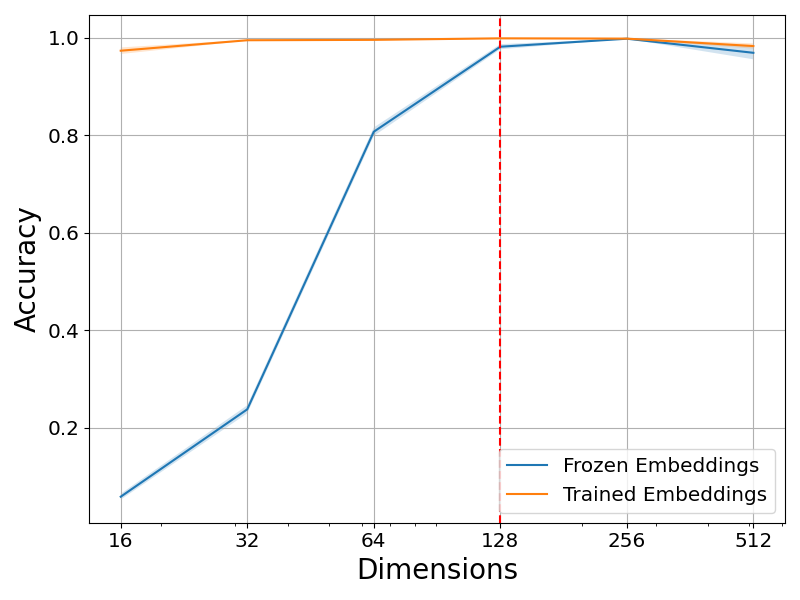}
        \caption{$m=7$}
        \label{fig:dim_acc_7}
    \end{subfigure}
    \begin{subfigure}[b]{0.45\textwidth}
        \centering
        \includegraphics[width=\textwidth]{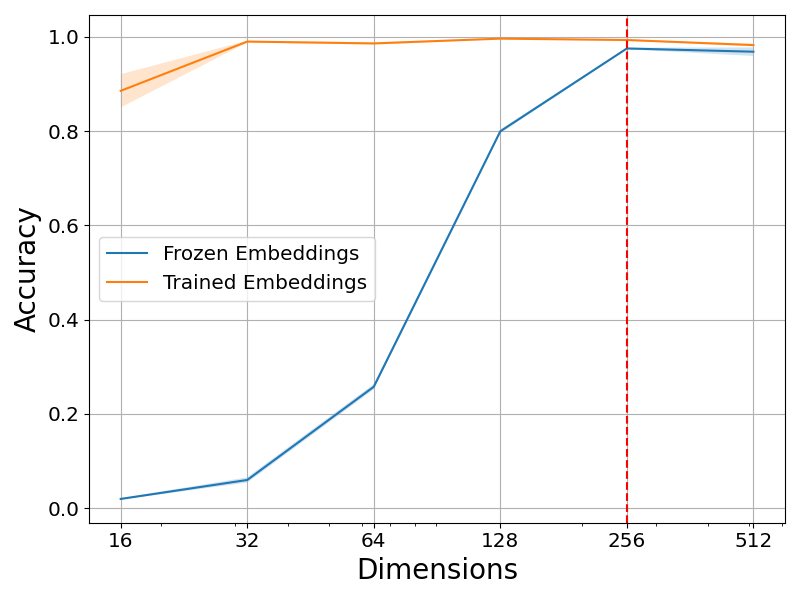}
        \caption{$m=8$}
        \label{fig:dim_acc_8}
    \end{subfigure}
    \caption{In the underparameterized regime $(d < \vocab)$, freezing embeddings to initializations causes a significant decrease in performance. The graph reports accuracy with different embedding dimensions and the standard errors are over $5$ runs. Red lines indicate when $d = \vocab$.} 
    \label{fig:dim_acc}
\end{figure}

\begin{figure}
    \centering
    \begin{subfigure}[b]{0.45\textwidth}
        \centering
        \includegraphics[width=\textwidth]{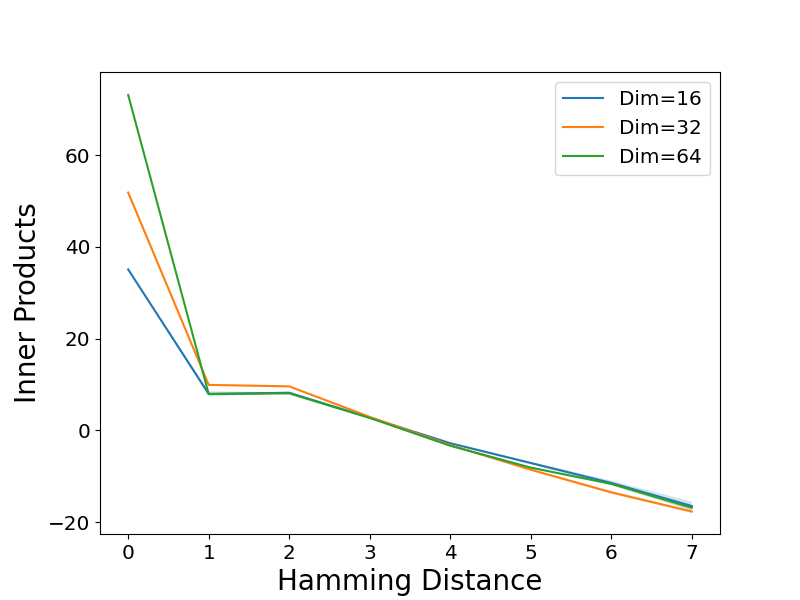}
        \caption{$m=7$}
        \label{fig:hamming_7}
    \end{subfigure}
    \begin{subfigure}[b]{0.45\textwidth}
        \centering
        \includegraphics[width=\textwidth]{pics/hamming/hamming_inner_n8.png}
        \caption{$m=8$}
        \label{fig:hamming_8}
    \end{subfigure}
    \caption{The relationship  between inner products of embeddings and corresponding Hamming distances of tokens can be approximated by \cref{eqn:emb-linear}.
    The graph displays the average inner product between embeddings of two tokens against the corresponding Hamming distance between these tokens. Standard errors are over $5$ runs.} 
    \label{fig:hamming}
\end{figure}

\begin{figure}
    \centering
    \begin{subfigure}[b]{0.45\textwidth}
        \centering
        \includegraphics[width=\textwidth]{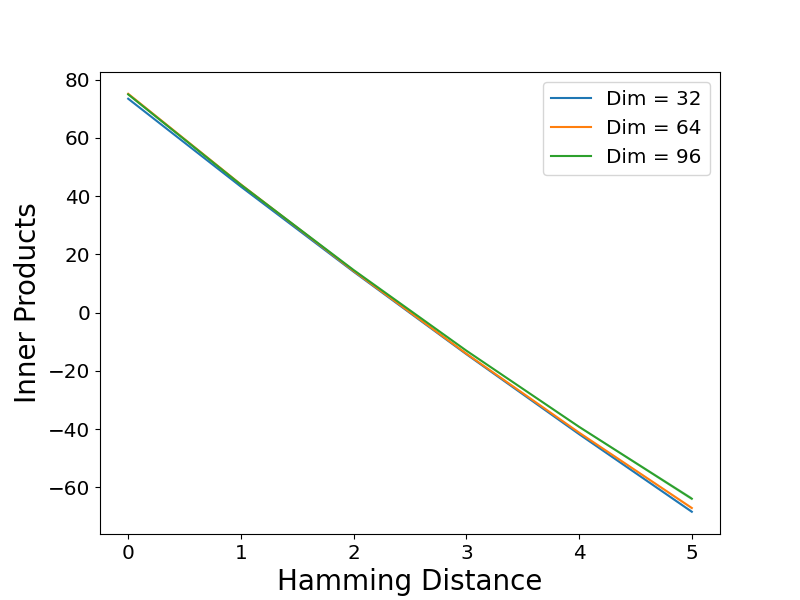}
        \caption{$m=5$}
        \label{fig:no_wv_hamming_5}
    \end{subfigure}
    \begin{subfigure}[b]{0.45\textwidth}
        \centering
        \includegraphics[width=\textwidth]{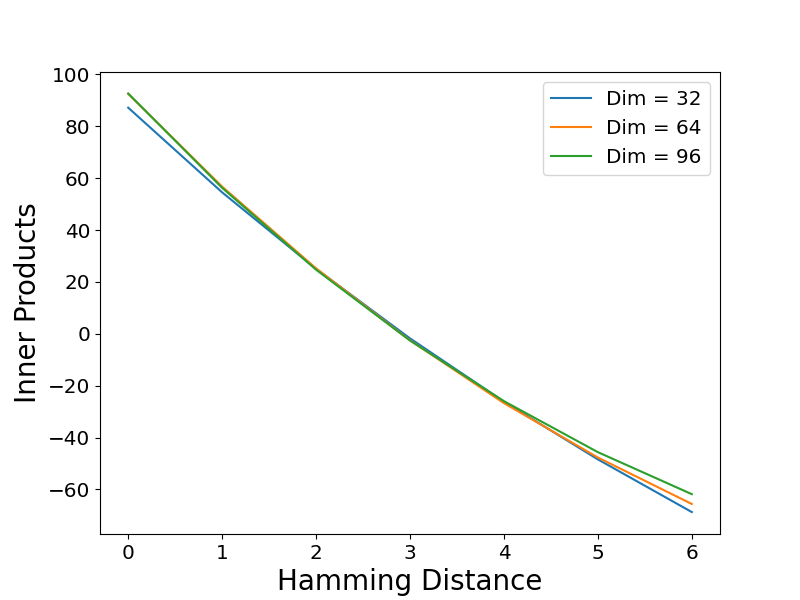}
        \caption{$m=6$}
        \label{fig:no_wv_hamming_6}
    \end{subfigure}

    \begin{subfigure}[b]{0.45\textwidth}
        \centering
        \includegraphics[width=\textwidth]{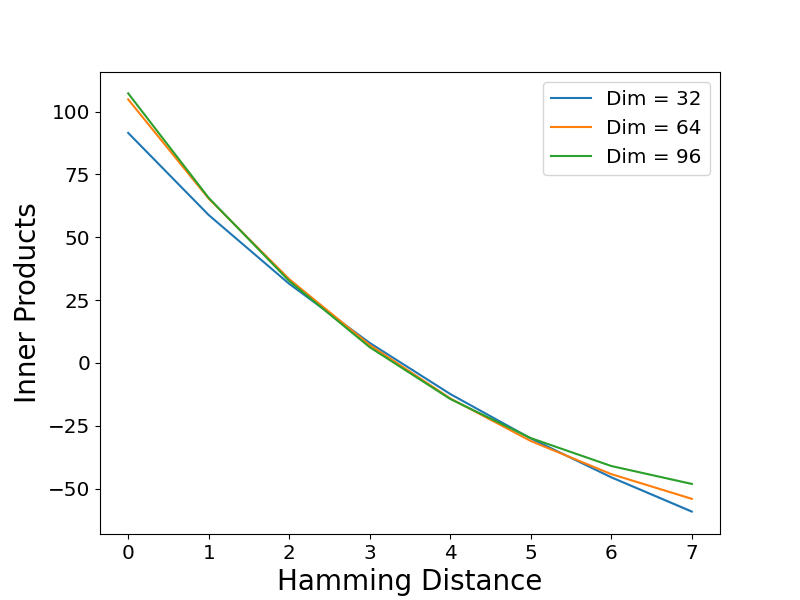}
        \caption{$m=7$}
        \label{fig:no_wv_hamming_7}
    \end{subfigure}
    \begin{subfigure}[b]{0.45\textwidth}
        \centering
        \includegraphics[width=\textwidth]{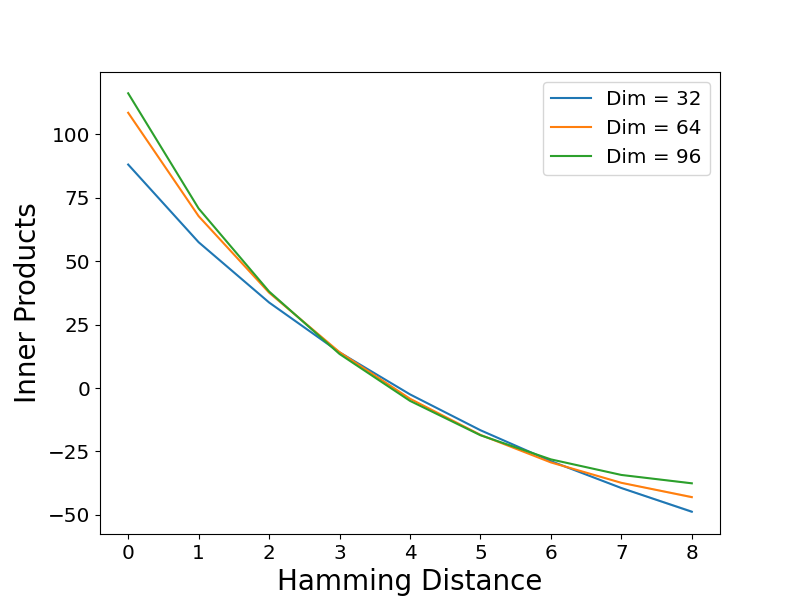}
        \caption{$m=8$}
        \label{fig:no_wv_hamming_8}
    \end{subfigure}
    \caption{The relationship  between inner products of embeddings and corresponding Hamming distances of tokens is mostly linear when the value matrix $W_V$ is fixed to be the identity.
    The graph displays the average inner product between embeddings of two tokens against the corresponding Hamming distance between these tokens. Standard errors are over $10$ runs.} 
    \label{fig:no_wv_hamming}
\end{figure}

\subsection{On the attention selection mechanism}
\label{appen:expatten}
This section provides additional figures from \cref{sec:expatten}. Figure~\ref{fig:cluster_1}-\ref{fig:cluster_2} show that attention mechanism selects tokens in the same cluster as the last token. In particular, for \cref{fig:cluster_2}, we extend experiments to consider cluster structures that depend on the first two latent variables. In other words, for any latent vector $z^{*}$, we have 
\begin{equation*}
    \latNe(z^{*}) = \{ z: z^{*}_1 = z_1 \text{ and } z^{*}_2 = z_2\} \setminus \{ z^{*} \}
\end{equation*}

\begin{figure}
    \centering
    \begin{subfigure}[b]{0.45\textwidth}
        \centering
        \includegraphics[width=\textwidth]{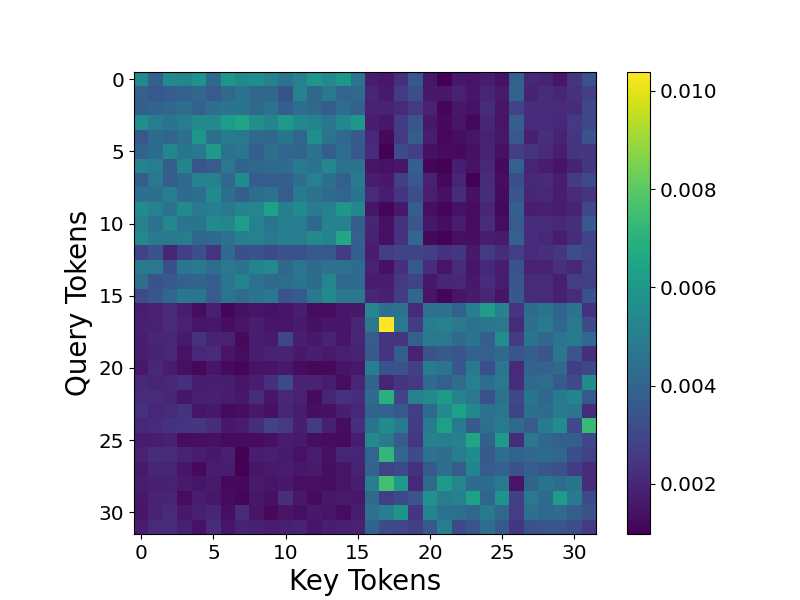}
        \caption{$m=5$}
        \label{fig:cluster_1_5}
    \end{subfigure}
    \begin{subfigure}[b]{0.45\textwidth}
        \centering
        \includegraphics[width=\textwidth]{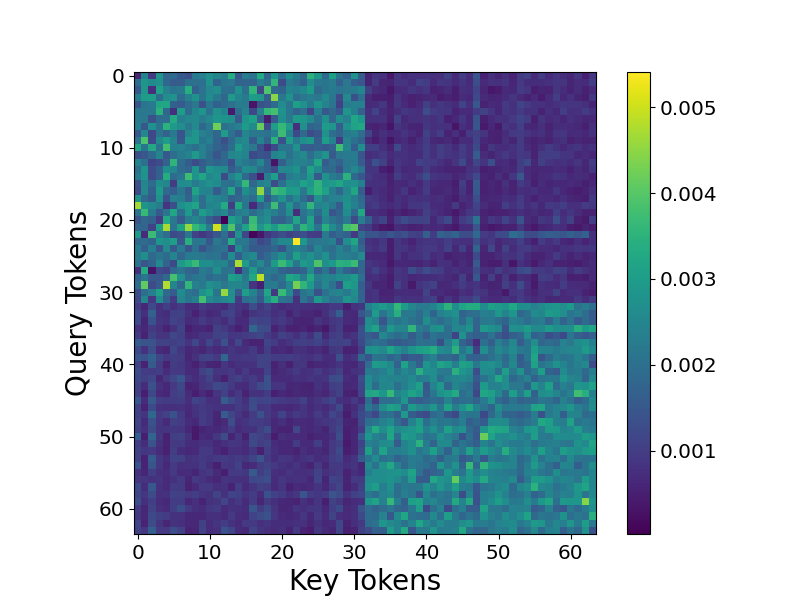}
        \caption{$m=6$}
        \label{fig:cluster_1_6}
    \end{subfigure}

    \begin{subfigure}[b]{0.45\textwidth}
        \centering
        \includegraphics[width=\textwidth]{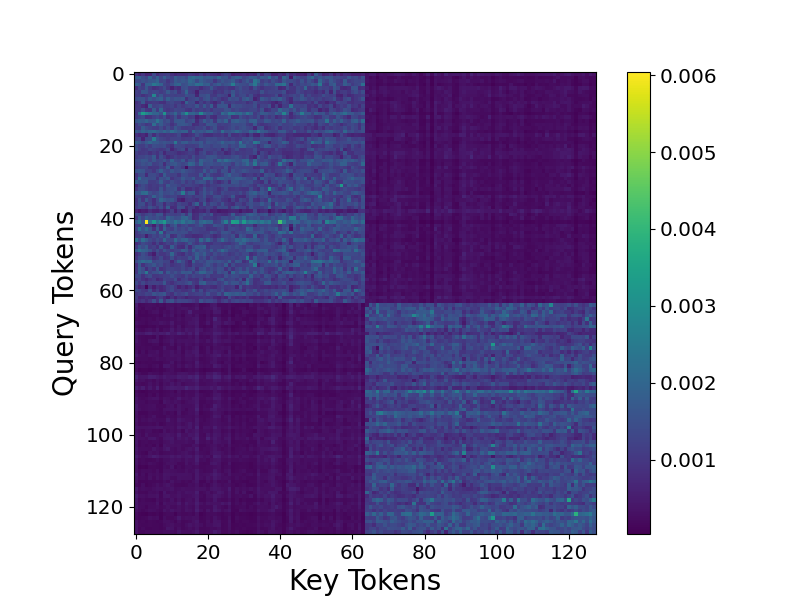}
        \caption{$m=7$}
        \label{fig:cluster_1_7}
    \end{subfigure}
    \begin{subfigure}[b]{0.45\textwidth}
        \centering
        \includegraphics[width=\textwidth]{pics/attention_cluster_1/heat_map_n8.png}
        \caption{$m=8$}
        \label{fig:cluster_1_8}
    \end{subfigure}
    \caption{The attention patterns show the underlying cluster structure of the data generating process. Here, for any latent vector, we have $\latNe(z^{*}) = \{ z: z^{*}_1 = z_1 \} \setminus \{ z^{*} \}$.
    The figure shows attention score heat maps that are averaged over $10$ runs.} 
    \label{fig:cluster_1}
\end{figure}

\begin{figure}
    \centering
    \begin{subfigure}[b]{0.45\textwidth}
        \centering
        \includegraphics[width=\textwidth]{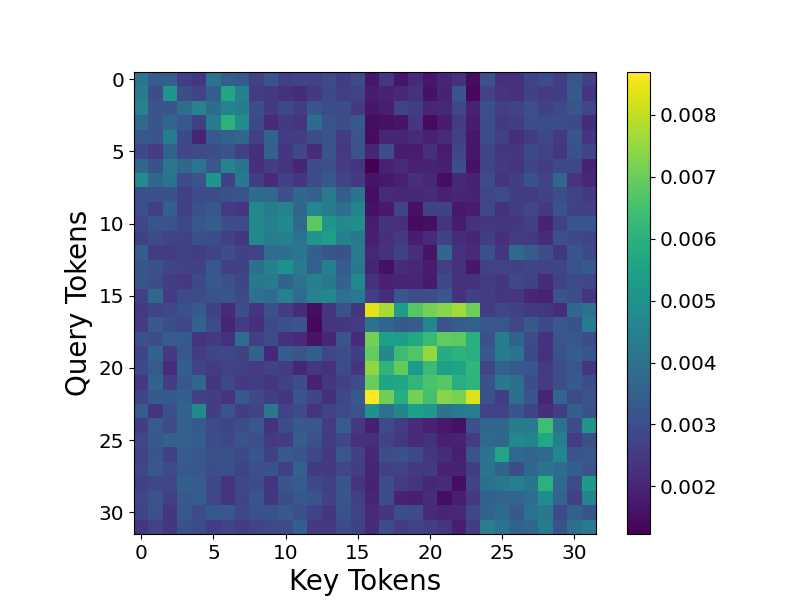}
        \caption{$m=5$}
        \label{fig:cluster_2_5}
    \end{subfigure}
    \begin{subfigure}[b]{0.45\textwidth}
        \centering
        \includegraphics[width=\textwidth]{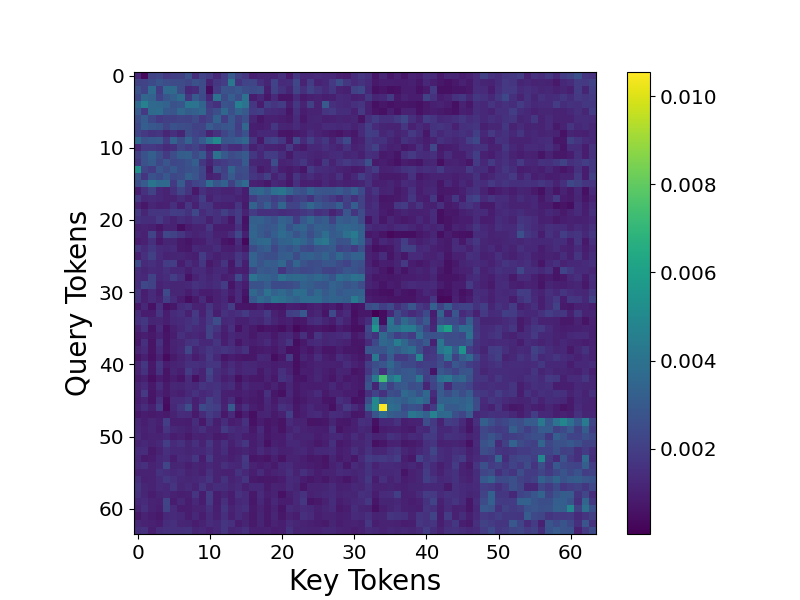}
        \caption{$m=6$}
        \label{fig:cluster_2_6}
    \end{subfigure}

    \begin{subfigure}[b]{0.45\textwidth}
        \centering
        \includegraphics[width=\textwidth]{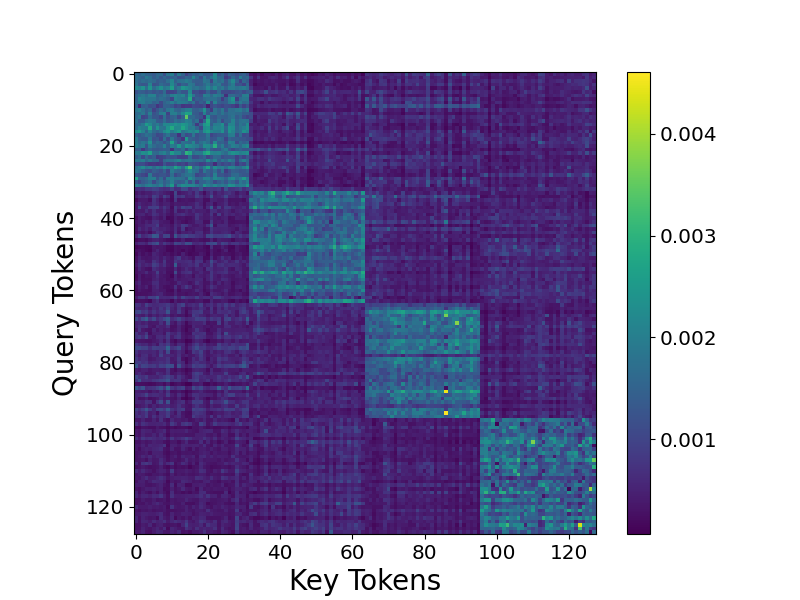}
        \caption{$m=7$}
        \label{fig:cluster_2_7}
    \end{subfigure}
    \begin{subfigure}[b]{0.45\textwidth}
        \centering
        \includegraphics[width=\textwidth]{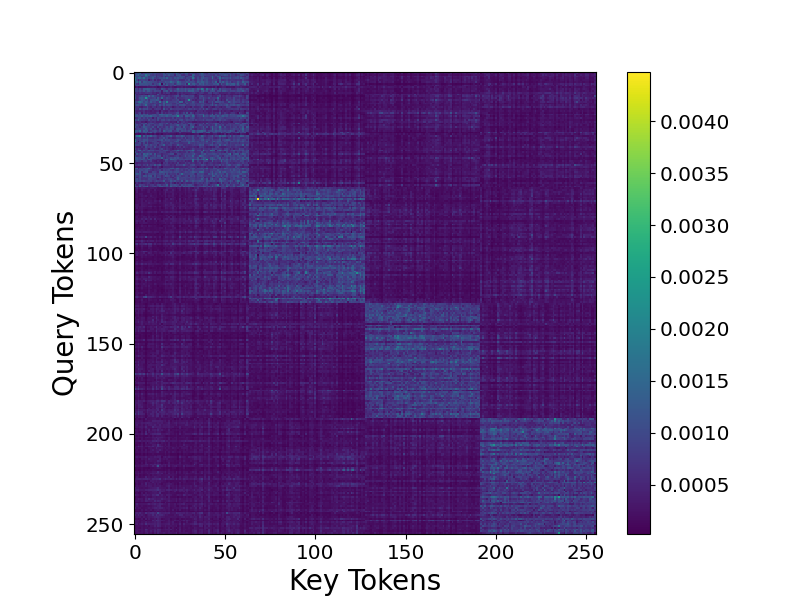}
        \caption{$m=8$}
        \label{fig:cluster_2_8}
    \end{subfigure}
    \caption{The attention patterns show the underlying cluster structure of the data generating process. Here, for any latent vector, we have $\latNe(z^{*}) = \{ z: z^{*}_1 = z_1 \text{ and } z^{*}_2 = z_2\} \setminus \{ z^{*} \}$.
    The figure shows attention score heat maps that are averaged over $10$ runs.} 
    \label{fig:cluster_2}
\end{figure}

\subsection{Spectrum of embeddings}
\label{appen:spectrum}

We display several plots of embedding spectra (\cref{fig:7_32}, \cref{fig:7_64}, \cref{fig:8_32}, \cref{fig:8_64}) that exhibit eigengaps between the top and bottom eigenvalues, suggesting low-rank structures.

\begin{figure}
    \centering
    \begin{subfigure}[b]{0.45\textwidth}
        \centering
        \includegraphics[width=\textwidth]{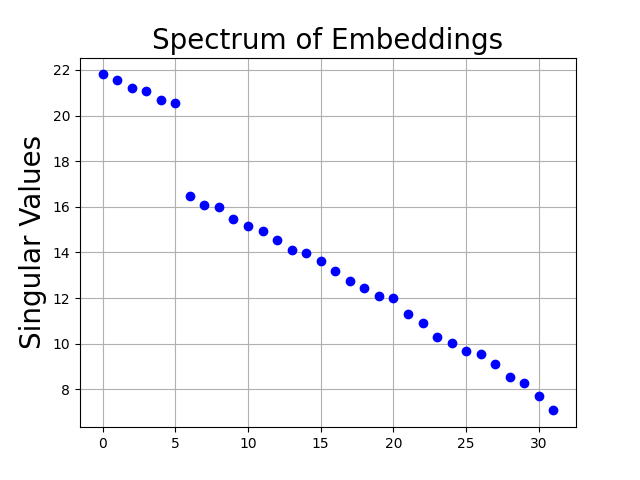}
        \caption{Sample $1$}
        \label{fig:7_32_1}
    \end{subfigure}
    \begin{subfigure}[b]{0.45\textwidth}
        \centering
        \includegraphics[width=\textwidth]{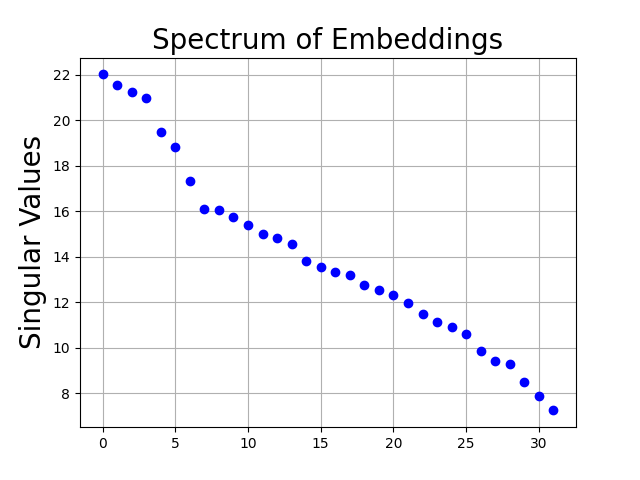}
        \caption{Sample $2$}
        \label{fig:7_32_2}
    \end{subfigure}

    \begin{subfigure}[b]{0.45\textwidth}
        \centering
        \includegraphics[width=\textwidth]{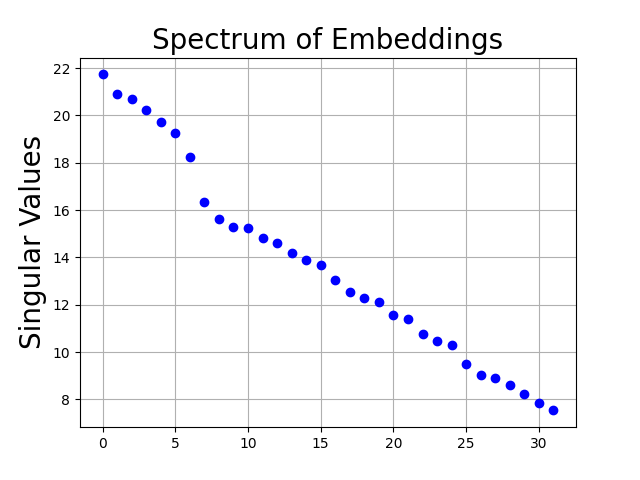}
        \caption{Sample $3$}
        \label{fig:7_32_3}
    \end{subfigure}
    \begin{subfigure}[b]{0.45\textwidth}
        \centering
        \includegraphics[width=\textwidth]{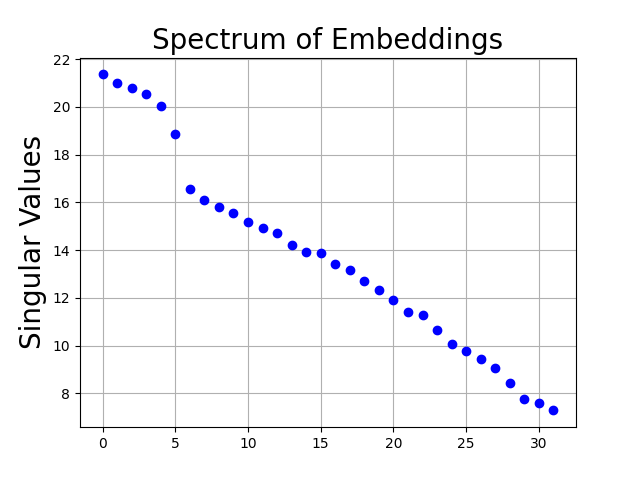}
        \caption{Sample $4$}
        \label{fig:7_32_4}
    \end{subfigure}
    \caption{The spectrum of embedding matrix $W_E$ has eigengaps between the top and bottom eigenvalues, indicating low rank structures. The figure shows results from 4 experimental runs. Number of latent variable $m$ is $7$ and the embedding dimension is $32$.} 
    \label{fig:7_32}
\end{figure}

\begin{figure}
    \centering
    \begin{subfigure}[b]{0.45\textwidth}
        \centering
        \includegraphics[width=\textwidth]{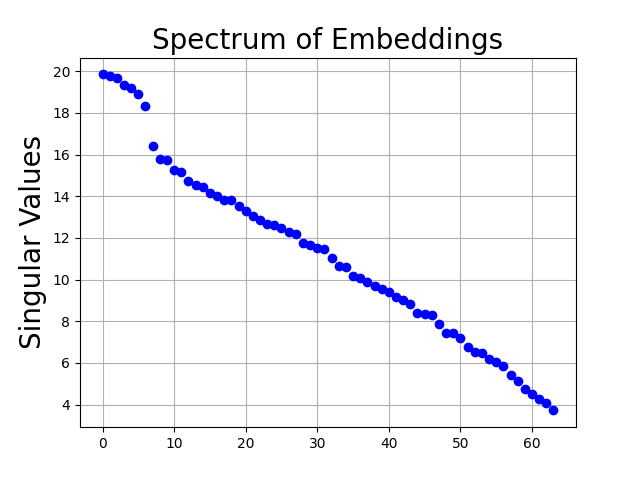}
        \caption{Sample $1$}
        \label{fig:7_64_1}
    \end{subfigure}
    \begin{subfigure}[b]{0.45\textwidth}
        \centering
        \includegraphics[width=\textwidth]{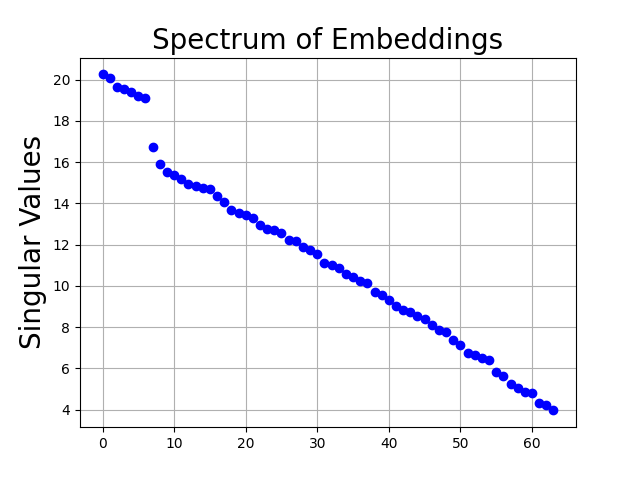}
        \caption{Sample $2$}
        \label{fig:7_64_2}
    \end{subfigure}

    \begin{subfigure}[b]{0.45\textwidth}
        \centering
        \includegraphics[width=\textwidth]{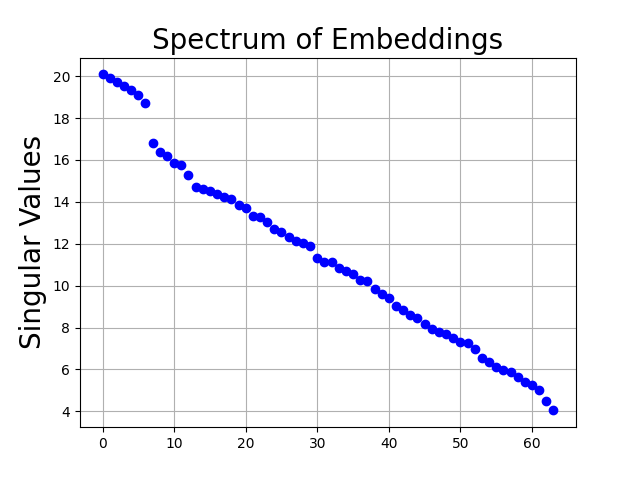}
        \caption{Sample $3$}
        \label{fig:7_64_3}
    \end{subfigure}
    \begin{subfigure}[b]{0.45\textwidth}
        \centering
        \includegraphics[width=\textwidth]{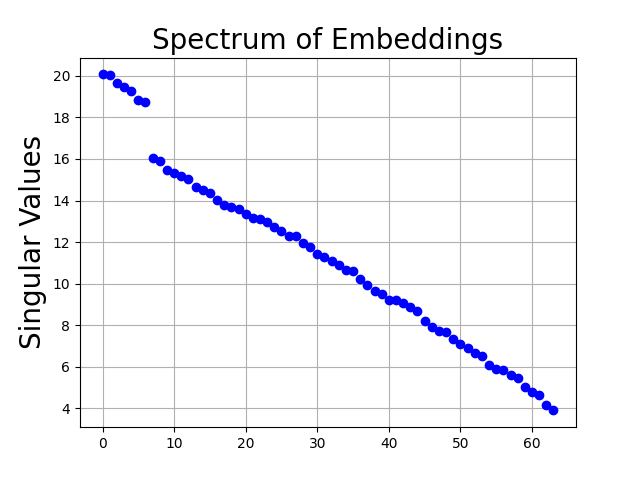}
        \caption{Sample $4$}
        \label{fig:7_64_4}
    \end{subfigure}
    \caption{The spectrum of embedding matrix $W_E$ has eigengaps between the top and bottom eigenvalues, indicating low rank structures. The figure shows results from 4 experimental runs. Number of latent variable $m$ is $7$ and the embedding dimension is $64$.} 
    \label{fig:7_64}
\end{figure}

\begin{figure}
    \centering
    \begin{subfigure}[b]{0.45\textwidth}
        \centering
        \includegraphics[width=\textwidth]{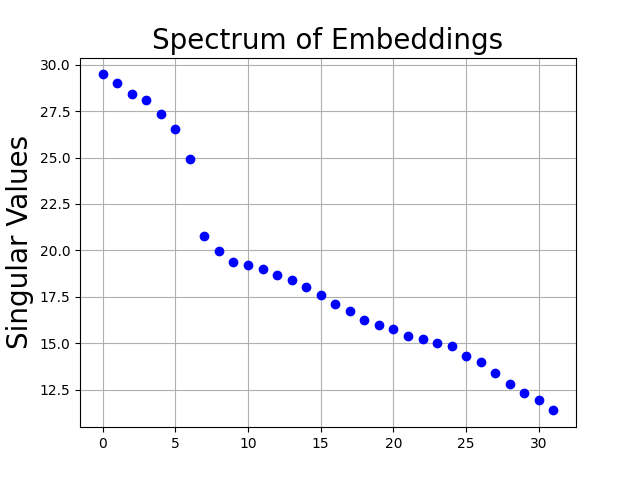}
        \caption{Sample $1$}
        \label{fig:8_32_1}
    \end{subfigure}
    \begin{subfigure}[b]{0.45\textwidth}
        \centering
        \includegraphics[width=\textwidth]{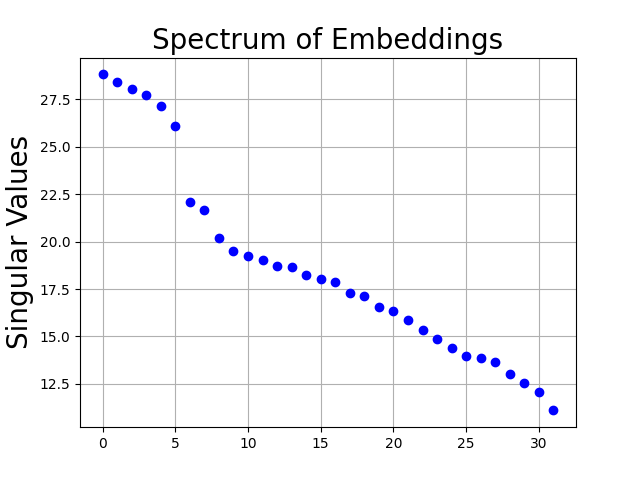}
        \caption{Sample $2$}
        \label{fig:8_32_2}
    \end{subfigure}

    \begin{subfigure}[b]{0.45\textwidth}
        \centering
        \includegraphics[width=\textwidth]{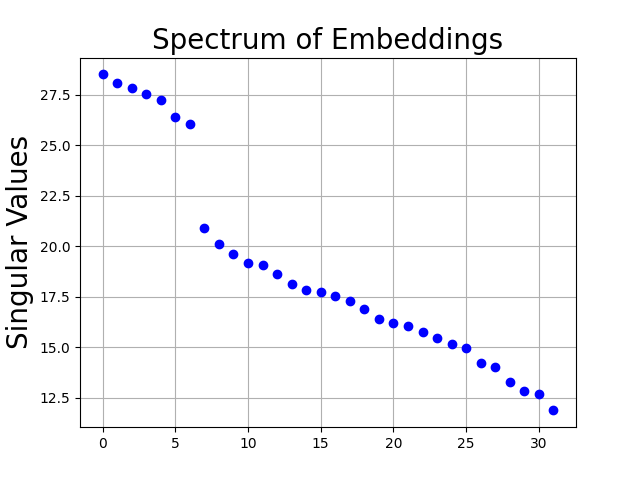}
        \caption{Sample $3$}
        \label{fig:8_32_3}
    \end{subfigure}
    \begin{subfigure}[b]{0.45\textwidth}
        \centering
        \includegraphics[width=\textwidth]{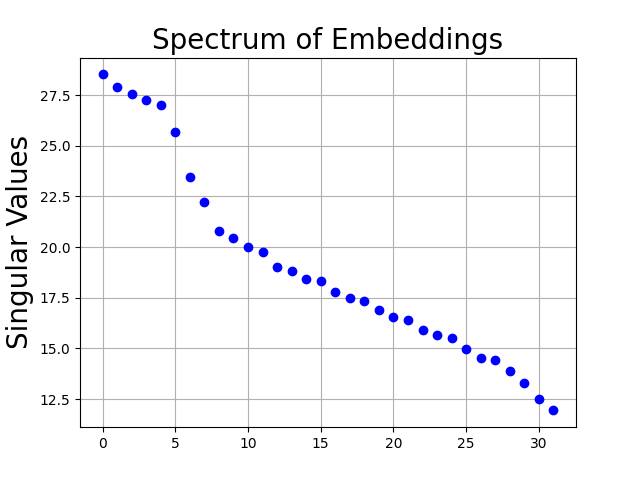}
        \caption{Sample $4$}
        \label{fig:8_32_4}
    \end{subfigure}
    \caption{The spectrum of embedding matrix $W_E$ has eigengaps between the top and bottom eigenvalues, indicating low rank structures. The figure shows results from 4 experimental runs. Number of latent variable $m$ is $8$ and the embedding dimension is $32$.} 
    \label{fig:8_32}
\end{figure}

\begin{figure}
    \centering
    \begin{subfigure}[b]{0.45\textwidth}
        \centering
        \includegraphics[width=\textwidth]{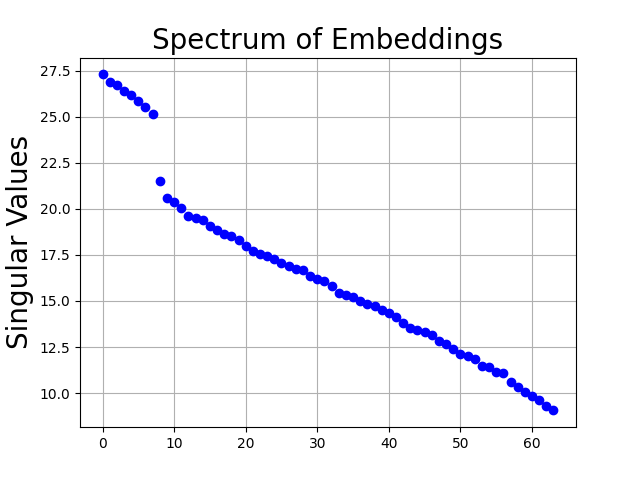}
        \caption{Sample $1$}
        \label{fig:8_64_1}
    \end{subfigure}
    \begin{subfigure}[b]{0.45\textwidth}
        \centering
        \includegraphics[width=\textwidth]{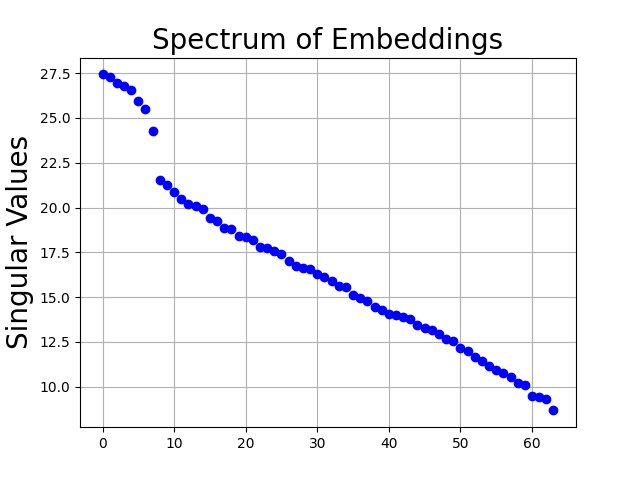}
        \caption{Sample $2$}
        \label{fig:8_64_2}
    \end{subfigure}

    \begin{subfigure}[b]{0.45\textwidth}
        \centering
        \includegraphics[width=\textwidth]{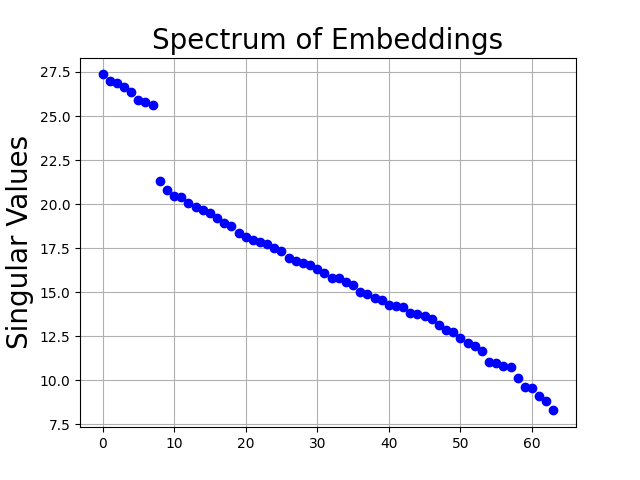}
        \caption{Sample $3$}
        \label{fig:8_64_3}
    \end{subfigure}
    \begin{subfigure}[b]{0.45\textwidth}
        \centering
        \includegraphics[width=\textwidth]{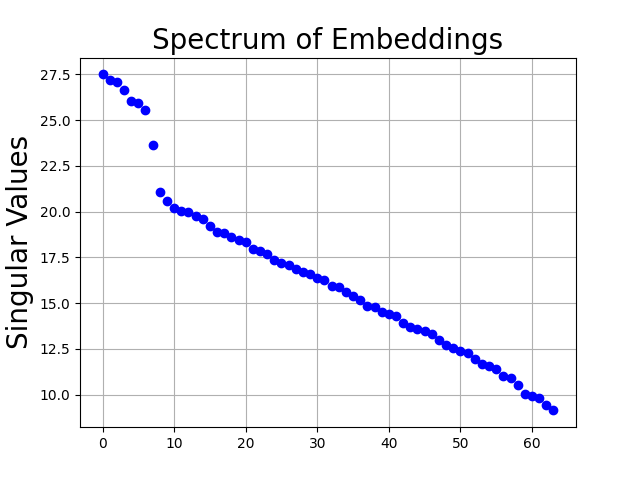}
        \caption{Sample $4$}
        \label{fig:8_64_4}
    \end{subfigure}
    \caption{The spectrum of embedding matrix $W_E$ has eigengaps between the top and bottom eigenvalues, indicating low rank structures. The figure shows results from 4 experimental runs. Number of latent variable $m$ is $8$ and the embedding dimension is $64$.} 
    \label{fig:8_64}
\end{figure}

\subsection{Context hijacking in latent concept association}
\label{appen:hijack-latent}
In this section, we want to simulate context hijacking in the latent concept association model. To achieve that, we first sample two output tokens $y^1$ (true target) and $y^2$ (false target) and then generate contexts $x^1 = (t^{1}_1, ..., t^{1}_L)$ and $x^2 = (t^{2}_1, ..., t^{2}_L)$ from $p(x^1|y^1)$ and $p(x^2|y^2)$. Then we mix the two contexts with rate $p_{m}$. In other words, for the final mixed context $x = (t_1, ..., t_L)$, $t_l$ has probability $1-p_{m}$ to be $t^1_l$ and $p_{m}$ probability to be $t^2_l$. \cref{fig:confusion} shows that, as the mixing rate increases from $0.0$ to $1.0$, the trained transformer tends to favor predicting false targets. This mirrors the phenomenon of context hijacking in LLMs.

\begin{figure}
    \centering
    \begin{subfigure}[b]{0.45\textwidth}
        \centering
        \includegraphics[width=\textwidth]{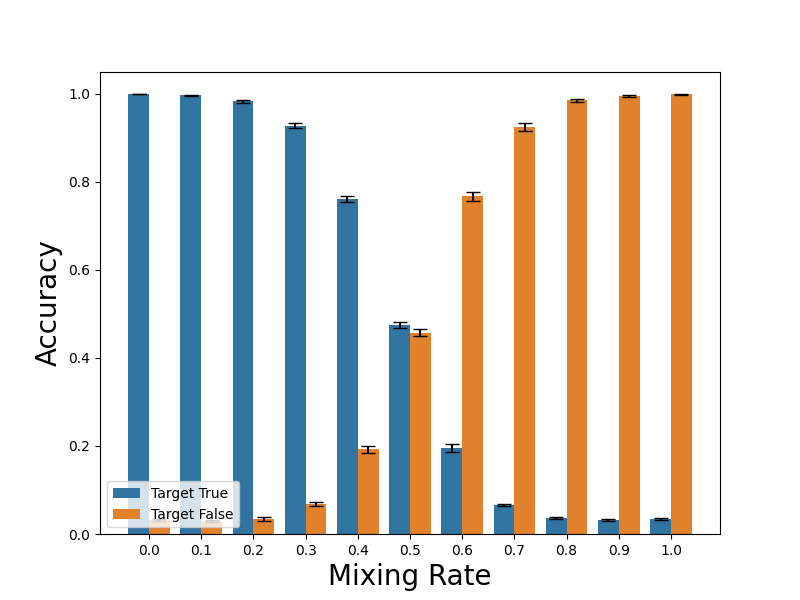}
        \caption{$m=5$}
        \label{fig:confusion_5}
    \end{subfigure}
    \begin{subfigure}[b]{0.45\textwidth}
        \centering
        \includegraphics[width=\textwidth]{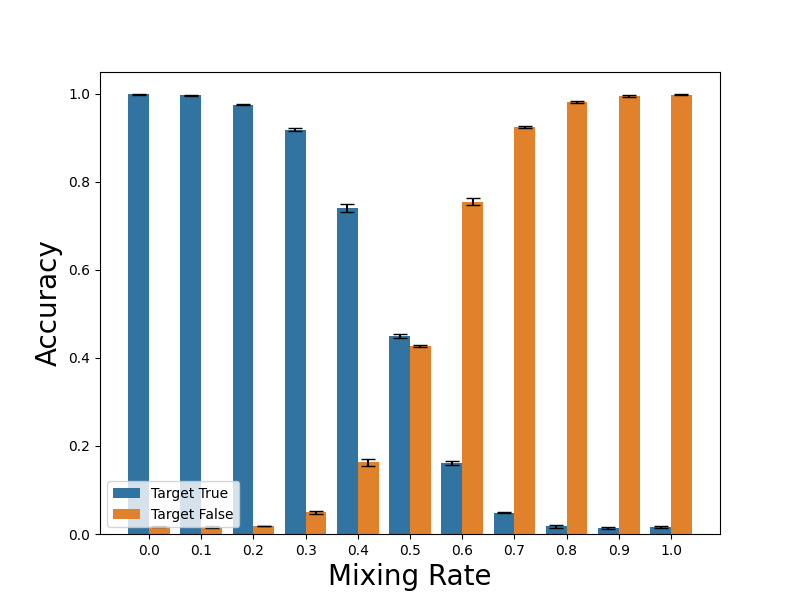}
        \caption{$m=6$}
        \label{fig:confusion_6}
    \end{subfigure}

    \begin{subfigure}[b]{0.45\textwidth}
        \centering
        \includegraphics[width=\textwidth]{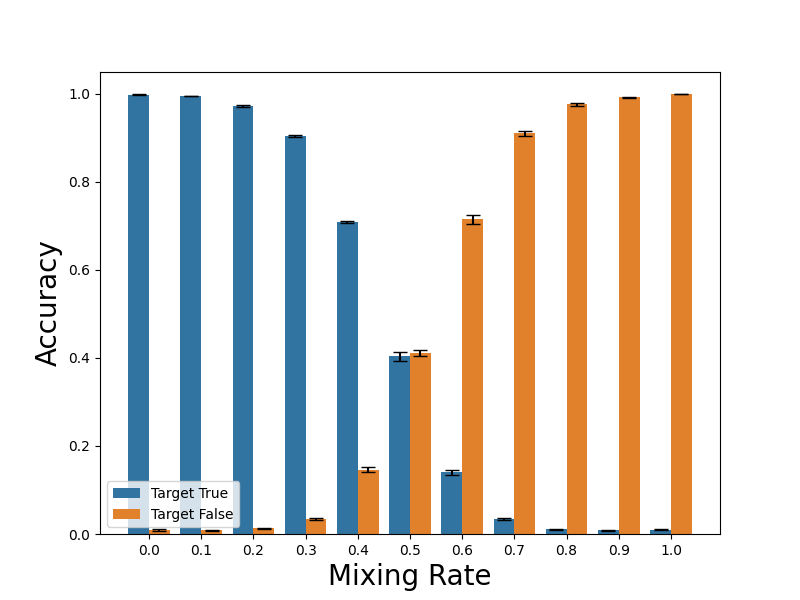}
        \caption{$m=7$}
        \label{fig:confusion_7}
    \end{subfigure}
    \begin{subfigure}[b]{0.45\textwidth}
        \centering
        \includegraphics[width=\textwidth]{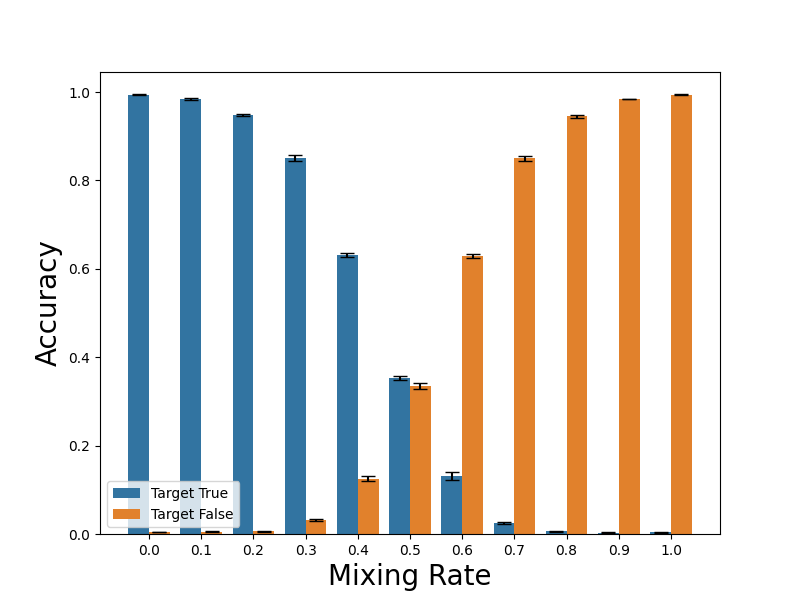}
        \caption{$m=8$}
        \label{fig:confusion_8}
    \end{subfigure}
    \caption{Mixing contexts can cause misclassification. The figure reports accuracy for true target and false target under various context mixing rate. Standard errors are over $5$ runs.} 
    \label{fig:confusion}
\end{figure}

\subsection{On the context lengths}
\label{sec:explength}
As alluded in \cref{sec:misclass}, the memory recall rate is closely related to the KL divergences between context conditional distributions. Because contexts contain mostly i.i.d samples, longer contexts imply larger divergences. This is empirically verified in \cref{fig:length} which demonstrates that longer context lengths can lead to higher accuracy.

\begin{figure}
    \centering
    \begin{subfigure}[b]{0.45\textwidth}
        \centering
        \includegraphics[width=\textwidth]{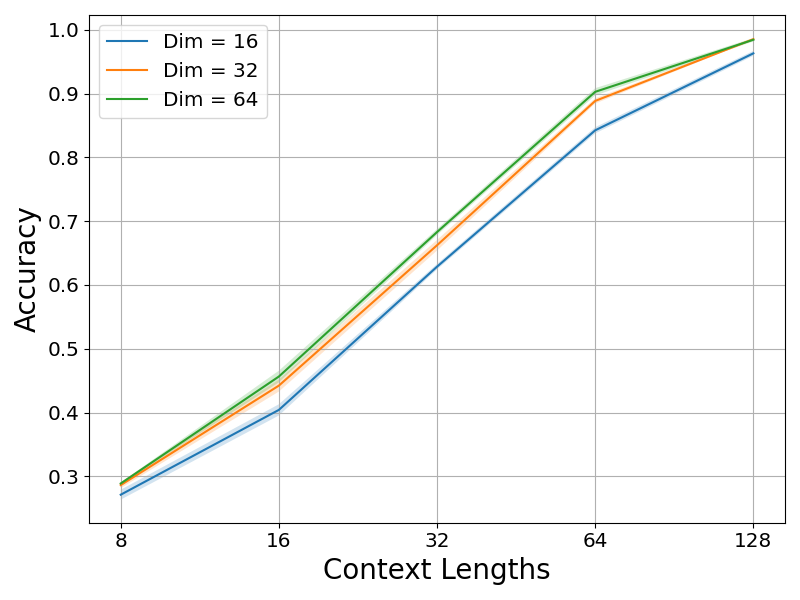}
        \caption{$m=5$}
        \label{fig:length_5}
    \end{subfigure}
    \begin{subfigure}[b]{0.45\textwidth}
        \centering
        \includegraphics[width=\textwidth]{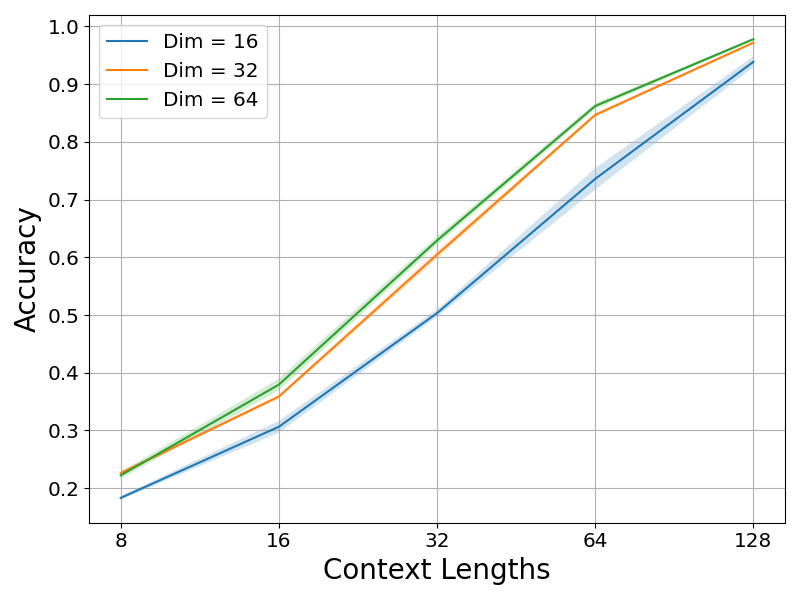}
        \caption{$m=6$}
        \label{fig:length_6}
    \end{subfigure}

    \begin{subfigure}[b]{0.45\textwidth}
        \centering
        \includegraphics[width=\textwidth]{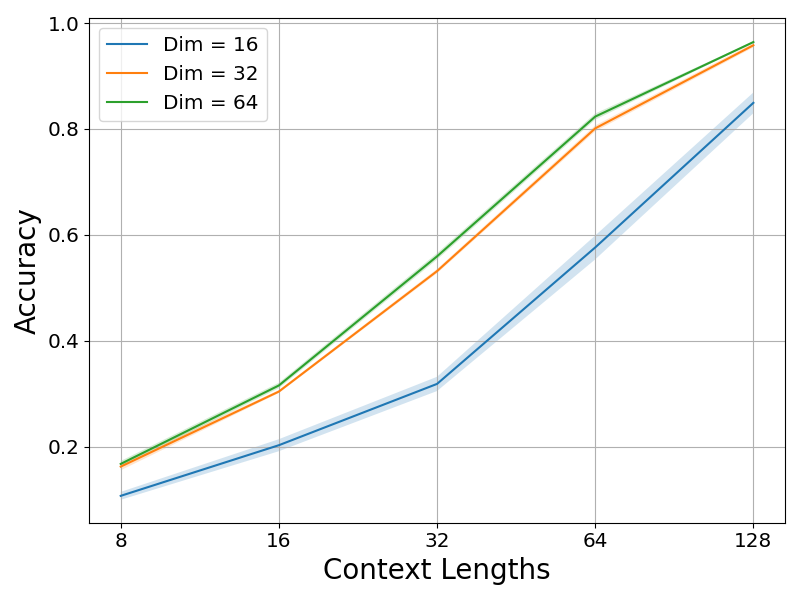}
        \caption{$m=7$}
        \label{fig:length_7}
    \end{subfigure}
    \begin{subfigure}[b]{0.45\textwidth}
        \centering
        \includegraphics[width=\textwidth]{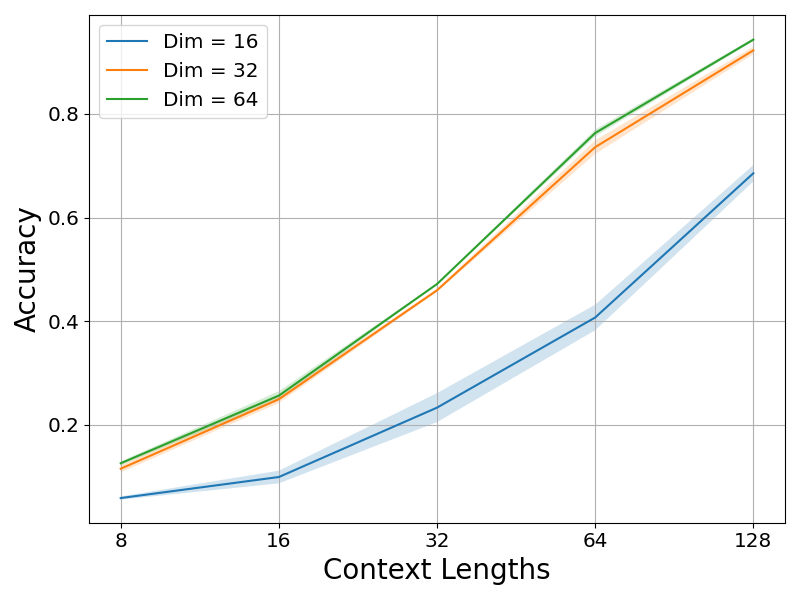}
        \caption{$m=8$}
        \label{fig:length_8}
    \end{subfigure}
    \caption{Increasing context lengths can improve accuracy.
    The figure reports accuracy across various context lengths and dimensions. Standard errors are over $5$ runs.} 
    \label{fig:length}
\end{figure}

\end{document}